\documentclass[11pt]{article}

\usepackage[utf8]{inputenc} 
\usepackage[T1]{fontenc}    

\usepackage{wrapfig,lipsum,booktabs}

\usepackage{paralist}
\usepackage{booktabs}       
\usepackage{nicefrac}       
\usepackage{multirow}
 \usepackage{xargs}

\usepackage{amsmath}
\usepackage{amsthm}
\usepackage{bm}
\allowdisplaybreaks
\usepackage{amssymb,mathrsfs}
\usepackage{amsfonts}
\usepackage{upgreek}
\usepackage{graphicx}
\usepackage{wrapfig}
\usepackage{xcolor}
\usepackage{pifont}
\usepackage{algpseudocode,algorithm,algorithmicx}
\usepackage{stmaryrd}
\usepackage{array}
\usepackage{enumitem}

\usepackage{graphicx} 
\usepackage{tikz}
\usepackage{pgfplots}

\usepackage{aliascnt}
\usepackage[colorlinks = true, citecolor = blue]{hyperref}
\usepackage{cleveref}

\newcommand{\cmark}{\textcolor{green!80!black}{\ding{51}}}
\newcommand{\xmark}{\textcolor{red}{\ding{55}}}

\makeatletter
\newtheorem{theorem}{Theorem}
\crefname{theorem}{theorem}{Theorems}
\Crefname{Theorem}{Theorem}{Theorems}

\newtheorem*{lemma_nonumber*}{Lemma}

\newtheorem{lemma}{Lemma}

\newaliascnt{corollary}{theorem}
\newtheorem{corollary}[corollary]{Corollary}
\aliascntresetthe{corollary}
\crefname{corollary}{corollary}{corollaries}
\Crefname{Corollary}{Corollary}{Corollaries}

\newaliascnt{proposition}{theorem}
\newtheorem{proposition}[proposition]{Proposition}
\aliascntresetthe{proposition}
\crefname{proposition}{proposition}{propositions}
\Crefname{Proposition}{Proposition}{Propositions}

\newaliascnt{definition}{theorem}

\aliascntresetthe{definition}
\crefname{definition}{definition}{definitions}
\Crefname{Definition}{Definition}{Definitions}

\newaliascnt{remark}{theorem}

\aliascntresetthe{remark}
\crefname{remark}{remark}{remarks}
\Crefname{Remark}{Remark}{Remarks}

\crefname{example}{example}{examples}
\Crefname{Example}{Example}{Examples}

\crefname{figure}{figure}{figures}
\Crefname{Figure}{Figure}{Figures}

\newtheorem{assumption}{\textbf{H}\hspace{-3pt}}
\crefformat{assumption}{{\textbf{H}}#2#1#3}

\definecolor{darkgreen}{RGB}{0,128,0}
\definecolor{darkorange}{RGB}{255,140,0}
\definecolor{darkblue}{RGB}{0,0,139}

\newcommand{\PP}{\mathbb{P}}

\def\mcb{\mathcal{B}}

\def\mcf{\mathcal{F}}
\def\mcg{\mathcal{G}}

\def\rset{\mathbb{R}}

\def\nset{\mathbb{N}}
\def\nsets{\mathbb{N}^*}

\def\rmd{\mathrm{d}}

\newcommand{\argmax}{\operatorname*{arg\,max}}

\newcommand{\LeftEqNo}{\let\veqno\@@leqno}

\newcommand{\floor}[1]{\left\lfloor #1 \right\rfloor}

\newcommand{\PE}{\mathbb{E}}

\newcommand{\absLigne}[1]{\vert #1 \vert}

\newcommand{\normLigne}[1]{\Vert #1 \Vert}

\newcommand{\parenthese}[1]{\left(#1 \right)}
\newcommand{\parentheseLignebig}[1]{\big(#1 \big)}
\newcommand{\parentheseLigneBig}[1]{\Big(#1 \Big)}
\newcommand{\parentheseLigne}[1]{(#1 )}

\newcommand{\psLigne}[2]{\langle#1,#2 \rangle}

\newcommand{\expe}[1]{\PE \left[ #1 \right]}

\def\ie{\textit{i.e.}}

\newcommand{\ocint}[1]{\left(#1\right]}
\newcommand{\ooint}[1]{\left(#1\right)}

\newcommandx\sequence[3][2=,3=]
{\ifthenelse{\equal{#3}{}}{\ensuremath{( #1_{#2})}}{\ensuremath{( #1_{#2})_{ #2 \in #3}}}}

\newcommandx\sequencet[3][2=,3=]
{\ifthenelse{\equal{#3}{}}{\ensuremath{( #1_{#2})}}{\ensuremath{( #1_{#2})_{ #2 \geq #3}}}}

\def\iid{\text{i.i.d.}}

\newcommand{\opnorm}[1]{{\left\vert\kern-0.25ex\left\vert\kern-0.25ex\left\vert #1
    \right\vert\kern-0.25ex\right\vert\kern-0.25ex\right\vert}}

\def\Id{\operatorname{Id}}

\def\Id{\operatorname{Id}}

\newcommand\coupling[2]{\Gamma(\mu,\nu)}

\def\bgamma{\bar{\gamma}}

\def\Ltt{\mathtt{L}}
\def\mtt{\mathtt{m}}
\def\Btt{\mathtt{B}}
\def\Mtt{\mathtt{M}}
\def\tMtt{\tilde{\mathtt{M}}}
\def\bMtt{\bar{\mathtt{M}}}

\newcommand{\N}{\mathbb{N}}

\newcommand{\R}{\mathbb{R}}

\newcommand{\E}{\mathbb{E}}
\newcommand{\msa}{\mathsf{A}}
\newcommand{\msx}{\mathsf{X}}

\newcommand{\Rd}{\mathbb{R}^{d}}

\newcommand{\dd}{\mathrm{d}}

\newcommand{\argmin}{\operatornamewithlimits{\arg\min}}

\newcommand{\Oh}{\operatorname{\mathrm{O}}}

\newcommand{\half}{1/2}  
\newcommand{\nofrac}[2]{{#1}/{#2}}  
\newcommand{\txts}{\textstyle}  
\newcommand{\eqsp}{\,}  

\newcommand{\pr}[1]{\left({#1}\right)}
\newcommand{\prn}[1]{({\textstyle{#1}})}

\newcommand{\br}[1]{\left[{#1}\right]}
\newcommand{\bbr}[1]{\left\{{#1}\right\}}
\newcommand{\brn}[1]{[{\textstyle{#1}}]}

\newcommand{\ac}[1]{\left\{{#1}\right\}}
\newcommand{\acn}[1]{\{{\textstyle{#1}}\}}

\newcommand{\norm}[1]{\left\|{#1}\right\|}
\newcommand{\normn}[1]{\|{\textstyle{#1}}\|}

\newcommand{\abs}[1]{\left\lvert{#1}\right\rvert}

\newcommand{\ps}[2]{\left\langle{#1},{#2}\right\rangle}  
\newcommand{\psn}[2]{\langle\textstyle{#1},{#2}\rangle}  

\newcommand{\gauss}{\mathrm{N}}
\newcommand{\wass}{W_{2}}
\newcommand{\up}{\mathscr{C}}

\newcommand{\btheta}{\tilde{\theta}}

\newcommand{\U}{U}

\newcommand{\lip}{\mathtt{L}}
\newcommand{\barM}{\bar{\mathtt{M}}}

\newcommand{\mU}{\mathtt{m}}
\newcommand{\grad}{H}  
\newcommand{\tgrad}{\tilde{H}}
\newcommand{\tgradD}{F}

\newcommand{\qslds}{$\texttt{QLSD}^\star$}

\newcommand{\qlsds}{$\texttt{QLSD}^\star$}
\newcommand{\qlsdsharp}{$\texttt{QLSD}^{\#}$}
\newcommand{\qlsd}{{\texttt{QLSD}}}
\newcommand{\qlsdpp}{$\texttt{QLSD}^{++}$}
\def\Bs{\Btt^{\star}}
\def\tBs{\tilde{\Btt}^{\star}}
\def\sigmas{\sigma_{\star}}
\def\MH{\Mtt}
\def\tMH{\tMtt}
\def\thetas{\theta^{\star}}
\def\bMH{\bar{\MH}}
\def\Pens{\mathcal{P}}
\def\tkappa{\tilde{\kappa}}
\def\Dg{\mathrm{D}_{\gamma}}
\def\DgZ{\mathrm{D}_{0,\gamma}}
\def\wpN{\wp_N}
\def\card{\mathrm{card}}
\def\wpNn{\wp_{N,n}}

\def\balpha{\bar{\alpha}}

\makeatletter
\newcommand{\ostar}{\mathbin{\mathpalette\make@circled\star}}
\newcommand{\pstar}{\mathbin{\mathpalette\make@circled+}}
\newcommand{\make@circled}[2]{%
  \ooalign{$\m@th#1\smallbigcirc{#1}$\cr\hidewidth$\m@th#1#2$\hidewidth\cr}%
}
\newcommand{\smallbigcirc}[1]{%
  \vcenter{\hbox{\scalebox{0.77778}{$\m@th#1\bigcirc$}}}%
}
\makeatother
\def\ttheta{\tilde{\theta}}

\usepackage[accepted]{aistats2022}

\usepackage[round]{natbib}

\author{%
    Maxime Vono$^*$ \\
    Criteo AI Lab \\
    Paris, France
    \And
    Vincent Plassier$^*$ \\
    CMAP, \'Ecole Polytechnique \\
    Lagrange Mathematics and \\ Computing Research Center
    \And
    Alain Durmus$^*$ \\
    Université Paris-Saclay, CNRS \\
    ENS Paris-Saclay Centre Borelli \\
    91190 Gif-sur-Yvette, France 
    \And
    Aymeric Dieuleveut \\
    CMAP, École Polytechnique
    \And
    \'Eric Moulines \\
    CMAP, École Polytechnique 
 }

\begin{document}

\runningauthor{Maxime Vono$^*$, Vincent Plassier$^*$, Alain Durmus$^*$, Aymeric Dieuleveut, Eric Moulines}

\maketitle

\begin{abstract}
The objective of Federated Learning (FL) is to perform statistical
inference for data which are decentralised and stored locally on
networked clients. FL raises many constraints which include privacy
and data ownership, communication overhead, statistical heterogeneity,
and partial client participation. In this paper, we address these
problems in the framework of the Bayesian paradigm. To this end, we
propose a novel federated Markov Chain Monte Carlo algorithm, referred
to as Quantised Langevin Stochastic Dynamics which may be seen as an
extension to the FL setting of Stochastic Gradient Langevin Dynamics,
which handles the communication bottleneck using gradient
compression. To improve performance, we then introduce variance
reduction techniques, which lead to two improved versions coined \texttt{QLSD}$^\star$ and \texttt{QLSD}$^{++}$. We give both
non-asymptotic and asymptotic convergence guarantees for the proposed
algorithms.  We illustrate their performances using various Bayesian Federated Learning benchmarks.
\end{abstract}

\section{INTRODUCTION}
\label{sec:introduction}

A paradigm shift has occurred with \emph{Federated Learning} (FL) \citep{mcmahan2017communication,kairouzFL}. In FL, multiple entities (called clients) which own locally stored data collaborate in learning a ``global'' model which can then be ``adapted'' to each client. In the canonical FL, this task is  coordinated by a central server. The initial focus of FL was on mobile and edge device applications, but recently there has been a surge of interest in applying the FL framework to other scenarios; in particular, those involving a small number of trusted clients (\emph{e.g.} multiple organisations, enterprises, or other stakeholders).

FL has become one of the most active areas of artificial intelligence research over the past 5 years. FL differs significantly from the classical (distributed) ML setup \citep{mcmahan2017communication}:  the storage, computational, and communication capacities of each client vary amongst each other. 
This poses considerable challenges to successfully deal with many constraints raised by (i) partial client participation (\emph{e.g.} in mobile applications, a client is not always active);
(ii) communication bottleneck (clients are communication-constrained with limited bandwidth usage); (iii) model update synchronisation and merging. 

\begin{table*}[h]  
  \caption{Overview of the main existing distributed/federated approximate Bayesian approaches. Column \emph{Comm. overhead} gives the scheme employed to address the communication bottleneck. Column \emph{Heterogeneity} means that the proposed approach tackles the impact of data heterogeneity on convergence while column \emph{Bounds} highlights available non-asymptotic convergence guarantees.}
  \label{table:overview}
  \vskip 0.15in
  \begin{center}
  \begin{small}
  \begin{sc}
  \begin{tabular}{lccccc}
  \toprule
  Method & Comm. overhead & Heterogeneity & Partial participation & Bounds \\
  \midrule
  \citet{JMLR:v18:16-478} & local steps & \xmark & \xmark & \xmark \\
  \citet{nemeth2018} & one-shot & \xmark & \xmark & \xmark \\
  \citet{ThangBui18} & local steps & \xmark & \cmark & \xmark \\
  \citet{Jordan2019} & one-shot & \xmark & \xmark & \cmark \\
  \citet{Corinzia19} & local steps & \xmark & \cmark & \xmark \\
  \citet{simeone20} & local steps & \xmark & \cmark & \xmark \\
  \citet{2020_conducive_gradients} & local steps & \xmark & \xmark & \cmark \\
  \citet{2021_ICML_DGLMC} & local steps & \xmark & \xmark & \cmark \\
  \citet{FedBE21} & local steps & \cmark & \cmark & \xmark \\
  \citet{SimeoneLiu21} & one-shot & \xmark & \xmark & \xmark \\
  This work & compression & \cmark & \cmark & \cmark \\
  \bottomrule
  \end{tabular}
  \end{sc}
  \end{small}
  \end{center}
  \vskip -0.1in
\end{table*}

Many methods derived from stochastic gradient descent techniques have been proposed in the literature to meet the specific FL constraints \citep{mcmahan2017communication,alistarh2017qsgd,DIANA19,scaffold20,fedprox20,Artemis20}, see \citet{FLreview2021} for a recent comprehensive overview.
Whilst these approaches have successfully solved important issues associated to FL, they are unfortunately unable to capture and quantify epistemic predictive uncertainty which is essential in many applications such as autonomous driving or precision medicine \citep{Hunter16,Franchi20}.
Indeed, these methods only provide a point estimate being a minimiser of a target empirical risk function.
In contrast, the Bayesian paradigm \citep{Robert94} stands for a natural candidate to quantify uncertainty by providing a full description of the posterior distribution of the parameter of interest, and as such has become ubiquitous in the machine learning community \citep{Andrieu2003,JMLR:v14:hoffman13a,izmailov2020subspace,izmailov2021bayesian}.

In the last decade, many research efforts have been made to adapt
serial workhorses of Bayesian computational methods such as variational
inference, expectation-propagation, and Markov chain Monte Carlo
(MCMC) algorithms to massively distributed architectures
\citep{Wang2013,Ahn14,Wang2015,JMLR:v18:16-478,ThangBui18,Jordan2019,Rendell2020,Vono_Paulin_Doucet_2019}.
Since the main bottleneck in distributed computing is the
communication overhead, these approaches mainly focus on deriving
efficient algorithms specifically designed to meet such a constraint,
requiring only periodic or few rounds of communication between a
central server and clients; see \citet[Section 4]{2021_ICML_DGLMC} for
a recent overview.  As highlighted in
\Cref{table:overview}, most current Bayesian FL methods adapt these approaches and focus almost
exclusively on Federated Averaging type updates  \citep{mcmahan2017communication}, performing multiple local steps on each client. This is in contrast with predictive FL algorithms (which
are \textbf{not} estimating predictive uncertainty), for which a variety of schemes have been explored, \emph{e.g.} via gradient
compression or client subsampling \citep[Section
3.1.2]{FLreview2021}.
Moreover, very few Bayesian FL works have
attempted to address the challenges raised by partial device
participation or the impact of statistical heterogeneity; see
\citet{SimeoneLiu21_bis,FedBE21}. Convergence results in Bayesian
FL lag far behind ``canonical'' FL.

In this paper, we attempt to fill this gap, by proposing novel MCMC
methods that extend Stochastic Langevin Dynamics to the FL context. It
is assumed that the clients' data are independent and that the global
posterior density is therefore the product of the \emph{non-identical} local posterior
densities of each client. To meet the specificity of Bayesian FL, each iteration of the proposed approaches only requires that a subset of active clients compute a stochastic gradient oracle for their associated
negative log posterior density and send a lossy compression of these stochastic gradient oracles to the central server.
The first scheme we derive, referred to as \emph{Quantised Langevin Stochastic Dynamics} (\texttt{QLSD}), can interestingly be seen as the MCMC counterpart of the \texttt{QSGD} approach in FL \citep{alistarh2017qsgd}, just as the Stochastic Gradient Langevin Dynamics (\texttt{SGLD}) \citep{Welling11} extends the Stochastic Gradient Descent (\texttt{SGD}).
However, \texttt{QLSD} has the same drawbacks as \texttt{SGLD}:
in particular, the invariant distribution of \texttt{QLSD} may deviate from the target distribution and become similar to the invariant measure of \texttt{SGD} when the number of observations is large \citep{brosse2018promises}.
We overcome this problem by deriving two variance-reduced versions \texttt{QLSD}$^{\star}$ and \texttt{QLSD}$^{++}$ that both include control variates.

\vspace{-0.3cm}
\paragraph{Contributions}
(1) We propose a general MCMC algorithm called \texttt{QLSD} specifically designed for Bayesian inference under the FL paradigm and two variance-reduced alternatives, especially tackling \emph{heterogeneity}, \emph{communication overhead} and \emph{partial participation}.
(2) We provide a non-asymptotic convergence analysis of the proposed algorithms. The theoretical analysis highlights the impact of statistical heterogeneity measured by the discrepancy between local posterior distributions.
(3) We propose efficient mechanisms to mitigate the impact of statistical heterogeneity on convergence, either by using biased stochastic gradients or by introducing a \emph{memory} mechanism that extends \cite{DIANA19} to the Bayesian setting. In particular, we find that variance reduction indeed allows the proposed MCMC algorithm to converge towards the desired target posterior distribution when the number of observations becomes large.
(4) We illustrate the advantages of the proposed methods using several FL benchmarks. We show that the proposed methodology performs well compared to state-of-the-art Bayesian FL methods.

\vspace{-0.3cm}
\paragraph{Notations and Conventions}
The Euclidean norm on $\mathbb{R}^d$ is denoted by $\|\cdot\|$ and we set $\nsets = \nset\setminus\{0\}$.
For $n \in \N^*$, we refer to $\{1,\ldots,n\}$ with the notation $[n]$.
For $N \in \N^*$, we use $\wp_N$ to denote the power set of $[N]$ and define $\wpNn = \{ x \in \wpN \, :\,  \card(x) = n\}$ for any $n \in [N]$.
We denote by $\mathrm{N}(m,\Sigma)$ the Gaussian distribution with mean vector $m$ and covariance matrix $\Sigma$.
We define the sign function, for any $x \in \R$, as $\mathrm{sign}(x) = \mathbf{1}\{x\geq0\} - \mathbf{1}\{x<0\}$.
We define the Wasserstein distance of order $2$ for any probability measures $\mu,\nu$ on $\Rd$ with finite $2$-moment by $W_2 (\mu, \nu) = (\inf_{\zeta \in \mathcal{T}(\mu,\nu)} \int_{\mathbb{R}^d \times \mathbb{R}^d}\|\theta-\theta'\|^2\mathrm{d}\zeta(\theta,\theta'))^{\half}$, where $\mathcal{T}(\mu, \nu)$ is the set of transference plans of $\mu$ and $\nu$.


\section{QUANTISED LANGEVIN STOCHASTIC DYNAMICS}
\label{sec:QLSD}

In this section, we present the Bayesian FL framework and introduce the proposed methodology called \texttt{QLSD} along with two variance-reduced instances.

\paragraph{Problem Statement} We are interested in performing Bayesian inference on a parameter $\theta \in \Rd$ based on a training dataset $\mathrm{D}$.
We assume that the posterior distribution admits a product-form density with respect to the $d$-dimensional Lebesgue measure, \emph{i.e.}
\begin{equation}
  \label{eq:target_density}
  \pi\pr{\theta \mid \mathrm{D}} = \mathrm{Z}_{\pi}^{-1}\,\textstyle\prod_{i=1}^b\mathrm{e}^{-U_{i}(\theta)}  \eqsp,
\end{equation}
where $b \in \N^*$ and $\mathrm{Z}_{\pi} = \int_{\Rd} \prod_{i=1}^b\mathrm{e}^{-U_i(\theta)}\,\dd \theta$ is a normalisation constant.
This framework naturally encompasses the considered Bayesian FL problem. In this context, $\{\mathrm{e}^{-U_i}\}_{i \in [b]}$ stand for the unnormalised local posterior density functions associated to $b$ clients, where each client $i \in [b]$ is assumed to own a local dataset $\mathrm{D}_{i}$ such that $\mathrm{D} = \sqcup_{i=1}^b \mathrm{D}_{i}$.
The dependency of $U_{i}$ on the local dataset $\mathrm{D}_{i}$ is omitted for brevity.
A real-world illustration of the considered Bayesian problem is “multi-site fMRI classification” where each site (or client) owns a dataset coming from a local distribution because the methods of data generation and collection differ between sites. This results in different local likelihood functions, which combined with a local prior distribution, lead to heterogeneous local posteriors.

As in embarrassingly parallel MCMC approaches \citep{Neiswanger2014}, \eqref{eq:target_density} implicitly assumes that the prior can be factorized across clients, which can always be done although the choice of this factorization is an open question. 
This product-form formulation can be alleviated by considering a global prior on $\theta$ and only calculating its gradient contribution on the central server during computations, see \Cref{algo:QLSD}. 

A popular approach to sample from a target distribution with density $\pi$ defined in \eqref{eq:target_density} is based on Langevin dynamics with stochastic gradient  which, starting from an initial point $\theta_0$, defines a Markov chain $(\theta_k)_{k \in \N}$ by recursion:
\begin{equation}
  \label{eq:SGLD}
  \theta_{k+1} = \theta_k - \gamma H_{k+1}\prn{\theta_k} + \sqrt{2\gamma} Z_{k+1} \eqsp, \quad k \in \N\eqsp,
\end{equation}
where $\gamma \in \ocint{0,\bgamma}$, for some $\bar{\gamma} > 0$, is a discretisation time step, $(Z_k)_{k \in\nsets}$ is a sequence of i.i.d. standard Gaussian random variables and $(H_k)_{k \in \mathbb{N}^*}$ stand for unbiased estimators of $\nabla U$ with $U = \sum_{i=1}^b U_i$ \citep{PARISI1981,GrenanderMiller1994,Roberts1996}.
In a serial setting involving a single client which owns a dataset of size $N \in \N^*$, the potential $U$ writes $U = U_1 = \sum_{j=1}^N U_{1,j}$ for some functions $U_{1,j} : \rset^d\to \rset$, and a popular instance of this framework is \texttt{SGLD} \citep{Welling11}.  This algorithm consists in the recursion \eqref{eq:SGLD} with the specific choice  $H_{k+1}(\theta) = (N/n)\sum_{j \in \mathcal{S}_{k+1}}\nabla U_{1,j}(\theta)$, where $(\mathcal{S}_k)_{k \in\nset^*}$ is a sequence of \iid~uniform random subsets of $[N]$ of cardinal $n$.

In the FL framework, we assume that at each iteration $k$, the $i$-th client has access to an oracle $H_{k+1}^{(i)}$ based on its local negative log posterior density $U_i$, depending only on $\mathrm{D}_{i}$, 
so that 
$H_{k+1} = \sum_{i=1}^b H_{k+1}^{(i)}$ is a stochastic gradient oracle of $U$.
Note that we do not assume that $H_{k+1}^{(i)}$ is an unbiased estimator of $\nabla U_i$, but only assume that $H_{k+1}$ is unbiased.
This allows us to consider biased local stochastic gradient oracles with better convergence guarantees, see \Cref{sec:theory} for more details.
A simple adaptation of \texttt{SGLD}  to the FL framework under consideration
is given by recursion:
\begin{equation} \label{eq:LSD}
  \theta_{k+1} = \theta_k - \gamma \textstyle\sum_{i=1}^b H_{k+1}^{(i)}\prn{\theta_k} + \sqrt{2\gamma} Z_{k+1} \eqsp, k \in \N\eqsp.
\end{equation}
If for any $i \in [b]$, every potential function $U_i$ also admits a finite-sum expression \emph{i.e.} $U_i = \sum_{j=1}^{N_i} U_{i,j}$, similar to \texttt{SGLD}, we can for example use the local stochastic gradient oracles $H_{k+1}^{(i)}(\theta) = (N_i/n_i)\sum_{j \in \mathcal{S}^{(i)}_{k+1}}\nabla U_{i,j}(\theta)$, where $(\mathcal{S}^{(i)}_{k+1})_{k \in\nsets, \, i\in[b]}$ stand for \iid~uniform random subsets of $[N_i]$ of cardinal $n_i$.
However, considering the MCMC algorithm associated with the recursion \eqref{eq:LSD} is not adapted to the FL context.
Indeed, this algorithm would assume that each client is reliable and suffers from the same issues as \texttt{SGD} in a risk-based minimisation context, especially a  prohibitive communication overhead \citep{girgis2020shuffled}.

\paragraph{Proposed Methodology} To address this problem, we propose to both account for the \emph{partial participation of clients} and \emph{reduce the number of bits transmitted} during the upload period by performing a lossy compression of a subset of $\{H_{k+1}^{(i)}\}_{i \in [b],k\in \N^*}$. This method has been used extensively in the ``canonical'' FL literature \citep{alistarh2017qsgd,lin2018deep,Haddadpour20,Sattler20}, but interestingly has never been considered in Bayesian FL; see \Cref{table:overview}.

To this end, we introduce a compression operator $\mathscr{C}: \Rd \to \Rd$ that is unbiased, \emph{i.e.} for any $v \in \Rd$, $\mathbb{E}[\mathscr{C}(v)] = v$.
In recent years, numerous compression operators have been proposed \citep{seide2014-bit,aji-heafield-2017-sparse,Stich2018}.
For example, the \texttt{QSGD} approach proposed in \citet{alistarh2017qsgd} is based on stochastic quantisation.

\texttt{QSGD} considers for $\up$ a component-wise quantisation operator parameterised by a number of quantisation levels $s \ge 1$, which for each $j \in [d]$ and $v = (v_1,\ldots,v_d) \in \Rd$ are given by

{\small
\begin{align}
  \label{eq:def_quantisation_operator}
  \mathscr{C}^{(s,j)}(v) = \frac{\norm{v} \mathrm{sign}(v_j) }{s}\pr{l_j + \mathbf{1}\bbr{\xi_j \le \frac{s|v_j|}{\norm{v}} - l_j}},
\end{align}
}

\noindent where $l_j = \floor{s|v_j|/\norm{v}}$ and $\{\xi_j\}_{j \in [d]}$ is a sequence of i.i.d. uniform random variables on $[0,1]$.
In this particular case, we will denote the quantisation of $v$ via \eqref{eq:def_quantisation_operator} by $\mathscr{C}^{(s)}(v) = \{\mathscr{C}^{(s,j)}(v)\}_{j \in [d]}$.

The proposed general methodology, called \emph{Quantised Langevin Stochastic Dynamics} (\texttt{QLSD}) stands for a compressed and FL version of the specific instance of \texttt{SGLD} defined in \eqref{eq:LSD}.
More precisely, \texttt{QLSD} is an MCMC algorithm associated with the Markov chain $(\theta_k)_{k \in \N}$ starting from $\theta_0$ and defined for $k \in \N$ as
\begin{align*}
\label{eq:QLSD}
  \theta_{k+1} = \ &\theta_k - \gamma \frac{b}{|\mathcal{A}_{k+1}|}\sum_{i \in \mathcal{A}_{k+1}} \mathscr{C}_{k+1}\br{H_{k+1}^{(i)}\prn{\theta_k}} \\
  &+ \sqrt{2\gamma} Z_{k+1} \eqsp,
\end{align*}
where $\prn{\mathcal{A}_{k}}_{k \in \mathbb{N}^*}$ denotes the subset of active (\emph{i.e.} available) clients at iteration $k$, possibly random.
Note that we indexed $\mathscr{C}$ by $k+1$ to emphasize that this compression operator is a stochastic operator and hence varies across iterations, see \emph{e.g.} \eqref{eq:def_quantisation_operator}.
The derivation of \texttt{QLSD} in the considered Bayesian FL context is described in details in \Cref{algo:QLSD}.
A generalisation of \texttt{QLSD} taking into account \emph{heterogeneous communication constraints} between clients by considering different compression operators $\{\mathscr{C}^{(i)}\}_{i \in [b]}$ is available in the Supplementary Material, see \emph{e.g.} Section S1.
In the particular case of the finite-sum setting where each client owns a dataset of size $N_i$, \ie~for the choice $H^{(i)}_{k+1}(\theta) = (N_i/n_i)\sum_{j \in \mathcal{S}_{k+1}^{(i)}}\nabla U_{i,j}(\theta)$ for $\theta \in\rset^d$, $\mathcal{S}_{k+1}^{(i)} \in \wp_{N_i,n_i}$, we denote the corresponding instance of \qlsd~as \qlsdsharp.

In this paper, we have decided to focus only on a non-adjusted sampling algorithm (\texttt{QLSD}) since the derivations of non-asymptotic results are already consequent, see the Supplementary Material.
In addition, up to authors’ knowledge, a general consensus on the choice between Metropolis-adjusted algorithms and their unadjusted counterparts has not been achieved yet.

\begin{algorithm}
   \caption{Quantised Langevin Stochastic Dynamics (\texttt{QLSD})}
   \label{algo:QLSD}
  \begin{algorithmic}
     \State {\bfseries Input:} nb. iterations $K$, compression operators $\{\mathscr{C}_{k+1}\}_{k\in\mathbb{N}}$, stochastic gradients $\{H_{k+1}^{(i)}\}_{i \in [b],k\in\mathbb{N}}$, step-size $\gamma \in (0,\bar{\gamma}]$ and initial point $\theta_0$.
     \For{$k=0$ {\bfseries to} $K-1$}
     \For{$i \in \mathcal{A}_{k+1}$ \Comment{On active clients $\mathcal{A}_{k+1}$}}
        \State Compute $\textsl{g}_{i,k+1} = \mathscr{C}_{k+1}\br{H_{k+1}^{(i)}\prn{\theta_k}}$.
        \State Send $\textsl{g}_{i,k+1}$ to the central server.
       \EndFor
       \State \Comment{On the central server}
       \State Compute $\textsl{g}_{k+1} = \frac{b}{|\mathcal{A}_{k+1}|}\sum_{i\in \mathcal{A}_{k+1}}\textsl{g}_{i,k+1}$.
       \State Draw $Z_{k+1} \sim \mathrm{N}(0_d,\mathrm{I}_d)$
       \State Compute $\theta_{k+1} = \theta_k - \gamma \textsl{g}_{k+1} + \sqrt{2\gamma}Z_{k+1}$.
       \State Send $\theta_{k+1}$ to the $b$ clients.
     \EndFor
     \State {\bfseries Output:} samples $\{\theta_k\}_{k=0}^{K}$.
  \end{algorithmic}
\end{algorithm}

\paragraph{Variance-Reduced Alternatives} Consider the finite-sum setting \emph{i.e.} for any $i \in [b]$, $U_i = \sum_{j=1}^{N_i} U_{i,j}$ where $N_i$ is the size of the local dataset $\mathrm{D}_i$.
As highlighted in \Cref{sec:introduction}, \texttt{SGLD}-based approaches, including \Cref{algo:QLSD}, involve an invariant distribution that may deviate from the target posterior distribution when $\min_{i \in b} N_i$ goes to infinity, as stochastic gradients with large variance are used \citep{brosse2018promises,BakerFFN19}.
We deal with this problem by proposing two variance-reduced alternatives of \qlsdsharp~ that use control variates.
The simplest variance-reduced approach, referred to as \texttt{QLSD}$^\star$ (see \Cref{algo:QLSD-star}) and discussed in more details in the Supplementary Material (see Section S2), considers a fixed-point approach that uses a minimiser $\thetas$ of the potential $U$ \citep{brosse2018promises,BakerFFN19} defined as
\begin{equation}
  \label{eq:def_theta_star}
  \thetas \in \argmin_{\theta \in \Rd} \sum_{i=1}^bU_i(\theta)\eqsp.
\end{equation}
In this scenario, the stochastic gradient oracles write for each $i \in [b]$, $k \in \N^*$, $\theta \in\rset^d$ and $\mathcal{S}_{k+1}^{(i)} \in \wp_{N_i,n_i}$, $H_{k+1}^{(i)}(\theta) = (N_i/n_i)\sum_{j \in \mathcal{S}_{k+1}^{(i)}} [\nabla U_{i,j}(\theta) - \nabla U_{i,j}(\theta^\star)]$.
Although $\mathbb{E}[H_{k+1}] = \nabla U$, note that for each $i \in [b] $, $\mathbb{E}[H_{k+1}^{(i)}] \neq \nabla U_i$ so $H_{k+1}^{(i)}$ is not an unbiased estimate of $U_i$.
We show in \Cref{sec:theory} that introducing this bias improves the convergence properties of $\texttt{QLSD}^{\#}$ with respect to the discrepancy between local posterior distributions.
Since estimating $\theta^\star$ in a FL context might impose an additional computational burden on the sampling procedure, we propose another variance-reduced alternative referred to as \texttt{QLSD}$^{++}$ (see \Cref{algo:QLSD-VR}). This method builds on the Stochastic Variance Reduced Gradient (\texttt{SVRG}): it uses control variates $(\zeta_k)_{k \in \N}$ that are updated every $l \in \N^*$ iterations \citep{Johnson13SVRG} and at each iteration $k \in\nset$ and for any client $i\in[b]$,
the stochastic gradient oracle $H_{k+1}^{(i)}$   defined by $H_{k+1}^{(i)}(\theta) = (N_i/n_i)\sum_{j \in \mathcal{S}_{k+1}^{(i)}} [\nabla U_{i,j}(\theta) - \nabla U_{i,j}(\zeta_k)] + \nabla U_i(\zeta_k)$.
To reduce the impact of local posterior discrepancy on convergence, we take inspiration from the ``canonical'' FL literature and consider a \emph{memory term} $(\eta^{(i)}_{k})_{k \in \N}$ on each client $i \in [b]$ \citep{DIANA19,dieuleveut2020}.
At each iteration $k$, instead of directly compressing $H_{k+1}^{(i)}$, we compress the difference $H_{k+1}^{(i)} - \eta^{(i)}_k$, store it in $\textsl{g}_{i,k+1}$, and then compute the global stochastic gradient $\textsl{g}_{k+1} = \frac{b}{|\mathcal{A}_{k+1}|}\sum_{i \in \mathcal{A}_{k+1}}\textsl{g}_{i,k+1} + \sum_{i=1}^b\eta^{(i)}_{k}$.
The memory term $(\eta^{(i)}_{k})_{k \in \N}$ is then updated on each client $i \in [b]$, by the recursion $\eta^{(i)}_{k+1} = \eta^{(i)}_{k} + \alpha \mathbf{1}_{\mathcal{A}_{k+1}}(i)\textsl{g}_{i,k+1}$. The benefits of using this memory mechanism will be assessed theoretically in \Cref{sec:theory} and illustrated numerically in Section S5.2 in the Supplementary Material.
\begin{algorithm}[h]
   \caption{Variance-reduced Quantised Langevin Stochastic Dynamics (\texttt{QLSD}$^{++}$)}
   \label{algo:QLSD-VR}
  \begin{algorithmic}
     \State {\bfseries Input:} minibatch sizes $\{n_i\}_{i \in [b]}$, number of iterations $K$, compression operators $\{\mathscr{C}_{k+1}\}_{k\in\mathbb{N}^*}$, step-size $\gamma \in (0,\bar{\gamma}]$ with $\bar{\gamma} > 0$, initial point $\theta_0$ and $\alpha \in (0,\bar{\alpha}]$ with $\bar{\alpha} > 0$.
     \State \Comment{Memory mechanism initialisation}
     \State Initialise $\{\eta^{(1)}_0,\ldots,\eta^{(b)}_0\}$ and $\eta_0 = \sum_{i=1}^b \eta^{(i)}_0$.
     \For{$k=0$ {\bfseries to} $K-1$}
      \State \Comment{Update of the control variates}
      \If{$k \equiv 0$ ($\mathrm{mod} \ l$)}
      \State Set $\zeta_k = \theta_k$.
      \Else{}
      \State Set $\zeta_k = \zeta_{k-1}$
      \EndIf
     \For{$i \in \mathcal{A}_{k+1}$ \Comment{On active clients}}
        \State Draw $\mathcal{S}_{k+1}^{(i)} \sim \mathrm{Uniform}\pr{\wp_{N_i,n_i}}$.
        \State {\small Set $H_{k+1}^{(i)}(\theta_k) = (N_i/n_i)\sum_{j \in \mathcal{S}_{k+1}^{(i)}} [\nabla U_{i,j}(\theta_k) - \nabla U_{i,j}(\zeta_k)] + \nabla U_i(\zeta_k)$.}
        \State Compute $\textsl{g}_{i,k+1} = \mathscr{C}_{k+1}\pr{H_{k+1}^{(i)}(\theta_k) - \eta^{(i)}_{k}}$.
        \State Send $\textsl{g}_{i,k+1}$ to the central server.
        \State Set $\eta^{(i)}_{k+1} = \eta^{(i)}_{k} + \alpha \textsl{g}_{i,k+1}$.
     \EndFor
     \State \Comment{On the central server}
     \State Compute $\textsl{g}_{k+1} = \eta_k + \frac{b}{|\mathcal{A}_{k+1}|}\sum_{i \in \mathcal{A}_{k+1}}\textsl{g}_{i,k+1}$.
     \State Set $\eta_{k+1} = \eta_k + \alpha\sum_{i \in \mathcal{A}_{k+1}}^b\textsl{g}_{i,k+1}$.
     \State Draw $Z_{k+1} \sim \mathrm{N}(0_d,\mathrm{I}_d)$.
     \State Compute $\theta_{k+1} = \theta_k - \gamma \textsl{g}_{k+1} + \sqrt{2\gamma}Z_{k+1}$.
     \State Send $\theta_{k+1}$ to the $b$ clients.
     \EndFor
     \State {\bfseries Output:} samples $\{\theta_k\}_{k=0}^{K}$.
  \end{algorithmic}
\end{algorithm}
%
\section{THEORETICAL ANALYSIS}
\label{sec:theory}

This section provides a detailed theoretical analysis of the proposed methodology.
In particular, we will show the \emph{impact of using stochastic gradients}, \emph{partial participation} and \emph{compression} by deriving quantitative convergence bounds for \texttt{QLSD}, which is detailed in \Cref{algo:QLSD}.
We then derive non-asymptotic convergence bounds for \texttt{QLSD}$^\star$ and \texttt{QLSD}$^{++}$, and explicitly show that these variance-reduced algorithms indeed succeed in reducing both the variance caused by stochastic gradients and the effects of \emph{local posterior discrepancy} in the bounds we obtain for \qlsdsharp.
 We consider the following assumptions on the potential $U$.
\begin{assumption}\label{main:ass:potential_U}
  For any $i\in[b]$, $U_{i}$ is continuously differentiable. In addition, suppose that the following  hold.
  \begin{enumerate}[wide, labelwidth=!, labelindent=0pt,label=(\roman*),noitemsep,nolistsep]
  \item \label{main:ass:potential_U:1} $U$ is $\mU$-strongly convex, \emph{i.e. } for any $\theta_{1},\theta_{2} \in \Rd$, $\psLigne{\nabla U(\theta_1) - \nabla U(\theta_2)}{\theta_1-\theta_2} \geq \mtt \norm{\theta_1-\theta_2}^2$. 
  \item \label{main:ass:potential_U:2} $\U$ is $\lip$-Lipschitz, \emph{i.e. } for any $\theta_{1},\theta_{2} \in \Rd$, $\norm{\nabla\U\pr{\theta_{1}}-\nabla\U\pr{\theta_{2}}}
      \le \lip\norm{\theta_{1}-\theta_{2}}$.
    \end{enumerate}
\end{assumption}
Note that \Cref{main:ass:potential_U}-\ref{main:ass:potential_U:1} implies that $\U$ admits a unique minimiser denoted by $\theta^{\star}\in\Rd$.
%

The compression operators $\{\mathscr{C}_{k+1}\}_{k\in \N}$ are assumed to satisfy the following assumption.
\begin{assumption}\label{main:ass:compression}
  The compression operators $\{\mathscr{C}_{k+1}\}_{k\in \N}$ are independent and satisfy the following conditions.
  \begin{enumerate}[wide, labelwidth=!, labelindent=0pt,label=(\roman*),noitemsep,nolistsep]
    \item \label{main:ass:compression:unbiased} For any $k \in \N^*$, $v\in\R^d$, $\mathbb{E}[\up_{k}\prn{v}] = v$.
    \item \label{main:ass:compression:variance} There exists $\omega\ge 1$, such that for any $k \in \N^*$, $v\in\R^d$, $\mathbb{E}[\norm{\up_k\prn{v}-v}^{2}] \le \omega\norm{v}^{2}$.
  \end{enumerate}
\end{assumption}
 As an example, the assumption on the variance of the compression operator detailed in \Cref{main:ass:compression}-\ref{main:ass:compression:variance} is verified for the quantisation operator $\up^{(s)}$ defined in \eqref{eq:def_quantisation_operator} with $\omega = \min(d/s^2,\sqrt{d}/s)$ \citep[Lemma 3.1]{alistarh2017qsgd}.

\paragraph{Non-Asymptotic Analysis for \Cref{algo:QLSD}} We consider the following assumptions on the stochastic gradient oracles used in \texttt{QLSD}.
  \begin{assumption}\label{main:ass:stochastic_gradient}
    The random fields $\acn{\grad_{k+1}^{(i)}:\Rd\to\Rd}_{i\in[b],k\in \mathbb{N}}$ are independent and satisfy the following conditions.
    \begin{enumerate}[wide, labelwidth=!, labelindent=0pt,label=(\roman*),noitemsep,nolistsep]
    \item \label{main:ass:stochastic_gradient:1} For any $\theta\in\Rd$ and $k \in \N$, $\sum_{i=1}^{b}\mathbb{E}[\grad_{k+1}^{(i)}\prn{\theta}]=\nabla\U\pr{\theta}$.
    \item \label{main:ass:stochastic_gradient:2}
      There exist $\{\Mtt_i > 0\}_{i\in [b]}$, such that for any $i \in [b]$, $k \in \N$, $\theta_{1},\theta_{2} \in \Rd$, $\mathbb{E}\br{\norm{\grad_{k+1}^{(i)}\prn{\theta_{1}}-\grad_{k+1}^{(i)}\prn{\theta_{2}}}^{2}}\le \Mtt_i \ps{\theta_{1}-\theta_{2}}{\nabla\U_{i}\pr{\theta_{1}}-\nabla\U_{i}\pr{\theta_{2}}}.$
    \item \label{main:ass:stochastic_gradient:3} There exist $\sigma_{\star},\Btt^{\star}\in\R_+$ such that for any $\theta\in\Rd$, $k \in \N$, we have $\mathbb{E}\br{\norm{\grad_{k+1}^{(i)}\prn{\theta^{\star}}}}^2 \le \Btt^{\star}/b$, and $\mathbb{E}\br{\norm{\textstyle\sum_{i=1}^{b}\grad_{k+1}^{(i)}\parentheseLigne{\theta^{\star}}}^2} \le \sigmas^2$,
    where $\thetas$ is defined in \eqref{eq:def_theta_star}.
    \end{enumerate}
  \end{assumption}
  We can notice that \Cref{main:ass:stochastic_gradient}-\ref{main:ass:stochastic_gradient:2} implies that $\nabla\U_{i}$ is $\Mtt_i$-Lipschitz continuous since by the Cauchy-Schwarz inequality, for any $i \in [b]$ and any $\theta_{1},\theta_{2}\in\Rd$, $\norm{\nabla\U_{i}\pr{\theta_{1}}-\nabla\U_{i}\pr{\theta_{2}}}^2
    \le \Mtt_i\ps{\theta_{1}-\theta_{2}}{\nabla\U_{i}\pr{\theta_{1}}-\nabla\U_{i}\pr{\theta_{2}}}$.  Conversely,  in  the finite-sum setting, \Cref{main:ass:stochastic_gradient}-\ref{main:ass:stochastic_gradient:2} is satisfied by \qlsdsharp~with $\Mtt_i  = N_i \bMtt$ if for any $i \in [b]$ and $j \in [N_i]$, $U_{i,j}$ is convex and $\nabla\U_{i,j}$ is $\bMtt$-Lipschitz continuous, for $\bMtt \geq 0$ by \citet[Theorem 2.1.5]{nesterov2003introductory}.

    \begin{table*}
\caption{Order of the asymptotic biases $\{B_{\bgamma}, B_{\protect\ostar,\bgamma}, B_{\protect\oplus,\bgamma}\}$, associated to the three proposed MCMC algorithms, in squared 2-Wasserstein distance for two types of asymptotic. \textcolor{red}{Red} dependencies prevent from (quick) convergence while \textcolor{green!40!black}{green} dependencies ensure convergence of associated MCMC algorithms. $\thetas$ is defined in \eqref{eq:def_theta_star}.}
\label{table:bias}
\begin{center}
\begin{tabular}{llcccccc}
\hline
\addlinespace[0.05cm]
  Algo. & Bias & \phantom{aaa} & \multicolumn{3}{c}{\shortstack{Dependencies of the \\ asymptotic bias when $\bgamma \downarrow 0$}} & & \shortstack{Dependencies of the \\ asymptotic bias as $N_i \rightarrow \infty$ } \\
\cmidrule{3-6} &&$d$ & $\grad_{k+1}^{(i)}$ & $\Btt^\star$ & partial particip. & $\omega$ & \\
\addlinespace[0.05cm]
\hline
\addlinespace[0.05cm]
  \texttt{QLSD} & $B_{\bgamma}$ & $d$ & $\sigma^2_\star$ & $\textcolor{red}{\Btt^\star}$ & $(1-p)/p$ & $\omega$ & $\textcolor{red}{\Oh(N_i)}$ \\
\qlsdsharp & $B_{\bgamma}$ & $d$ & $N_i^2$ & $\textcolor{red}{\sum_{i=1}^b \norm{\nabla U_i(\thetas)}^2}$ & $(1-p)/p$ & $\omega$ & $\textcolor{red}{\Oh(N_i)}$ \\
\texttt{QLSD}$^\star$ & $B_{\ostar,\bgamma}$ &$d$ & $N_i$ & - & $(1-p)/p$ & $\omega$ & $ \textcolor{green!40!black}{d\Oh\pr{1}}$ \\
\texttt{QLSD}$^{++}$ & $B_{\oplus,\bgamma}$ & $d$ & $N_i$ & - & $(1-p)/p$ & $\omega$ & $ \textcolor{green!40!black}{d\Oh\pr{1}}$  \\
\addlinespace[0.05cm]
\hline
\end{tabular}
\end{center}
\vspace{-0.4cm}
\end{table*}

  In addition, it is worth mentioning that the first inequality in \Cref{main:ass:stochastic_gradient}-\ref{main:ass:stochastic_gradient:3} is also required for our derivation in the deterministic case where $\grad_{k+1}^{(i)} = \nabla U_i$ due to the compression operator.
  In this particular case, $\Btt^\star$ stands for an upper-bound on $\sum_{i=1}^b\norm{\nabla U_i(\theta^\star)}^2$ and corresponds to some discrepancy between local posterior density functions meaning that $\nabla U_i \neq \nabla U$ for $i \in [b]$. This phenomenon, referred to as \emph{data heterogeneity} in the risk-based literature \citep{DIANA19,scaffold20}, is ubiquitous in the FL context.

Finally, we assume for simplicity that \emph{clients' partial participation} is realised by each client having probability $p \in (0,1]$ of being active in each communication round.
\begin{assumption}
  \label{ass:A_k}
For any $k \in \N^*$, $\mathcal{A}_{k} = \{ i \in [b] \,: \, B_{i,k}= 1 \}$ where $\{B_{i,k}\,: \, i \in [b]\, , \, k\in\nsets\}$ is a family of \iid~Bernouilli random variables with success probability $p \in\ocint{0,1}$.
\end{assumption}
 A generalisation of this scheme considering different probabilities $p_i$ per client can be found in the Supplementary Material, see \emph{e.g.} Section S1.1.
  Under the above assumptions and by denoting $Q_{\gamma}$ the Markov kernel associated to \Cref{algo:QLSD}, the following convergence result holds.

\begin{theorem}\label{thm:QLSD}
  Assume \Cref{main:ass:potential_U}, \Cref{main:ass:compression}, \Cref{main:ass:stochastic_gradient} and \Cref{ass:A_k}.
  Then, there exists $\bgamma_{\infty}$ such that for $\bgamma < \bgamma_{\infty}$, there exist $A_{\bgamma},B_{\bgamma} >0$ (explicitly given in Section S1 in the Supplementary Material) satisfying for any probability measure $\mu\in\mathcal{P}_{2}\pr{\Rd}$, any step size
  $\gamma\in\ocint{0,\bgamma}$ and $k \in \N$,\vspace{-0.2cm}
  \begin{align*}
    &\wass^{2}\pr{\mu Q_{\gamma}^{k},\pi}
    \le (1 - \gamma \mtt/2)^k \cdot \wass^{2}\pr{\mu,\pi} + \gamma B_{\bgamma} \\
    &+ \gamma^2 A_{\bgamma} (1-\mtt \gamma/2)^{k-1} k \cdot \int_{\Rd}\|\theta-\theta^\star\|^2\mu(\rmd\theta)\eqsp,
  \end{align*}
  where $\thetas$ is defined in \eqref{eq:def_theta_star}.
\end{theorem}
\vspace{-0.2cm}Similar to \texttt{ULA} \citep{Dalalyan2017,durmus2018high} and \texttt{SGLD} \citep{Dalalyan2019,durmus2019analysis}, the upper bound given in \Cref{thm:QLSD} includes a contracting term that depends on the initialisation and a bias term $\gamma B_{\bgamma}$ that does not vanish with $k \rightarrow \infty$ due to the use of a fixed step size $\gamma$.
In the asymptotic scenario, \ie~$\bgamma \downarrow 0$, \Cref{table:overview} gives the dependencies of $B_{\bgamma}$ for \qlsd~ and its particular instance \qlsdsharp, in terms of key quantities associated with the setting we consider.
Similar to \texttt{SGLD}, we can observe that the use of stochastic gradients entails a bias term of order $\sigma^2_\star\Oh(\gamma)$.
On the other hand, the use of partial participation and compression compared to \texttt{SGLD} introduces an \emph{additional bias} of order $(\omega/p)(\mtt\Btt^\star + \lip \Mtt d)\Oh(\gamma)$, which grows with in particular $\Btt^\star$, corresponding to the impact of the local posterior discrepancy on convergence.

\paragraph{Non-Asymptotic Analysis for Variance-Reduced Alternatives}
We assume in the sequel that the potential functions $\{U_i\}_{i \in [b]}$ admit the finite-sum decomposition $U_{i}=\sum_{j=1}^{N_i}U_{i,j}$ for each $i\in [b]$ and consider the following assumptions.
\begin{assumption}
  \label{ass:potential_Ui}
  For any $i\in[b],j\in [N_i]$, $U_{i,j}$ is continuously differentiable and the following holds.
  \begin{enumerate}[wide, labelwidth=!, labelindent=0pt,label=(\roman*),noitemsep,nolistsep]
    \item There exists $\MH_i > 0$ such that, for any $\theta_{1},\theta_{2} \in \Rd$, $\norm{\nabla U_i(\theta_2) - \nabla U_i(\theta_1)}^2 \le  \MH \ps{\theta_{2}-\theta_{1}}{\nabla\U_{i}\parentheseLigne{\theta_{2}}-\nabla\U_{i}\parentheseLigne{\theta_{1}}}$.
    \item  \label{ass:potential_Ui_coco} There exists $\bMH\ge 0$ such that, for any $\theta_{1},\theta_{2}\in\Rd$, $\norm{\nabla\U_{i,j}\prn{\theta_{2}}-\nabla\U_{i,j}\prn{\theta_{1}}}^2
      \le \bMH \ps{\nabla\U_{i,j}\prn{\theta_{2}}-\nabla\U_{i,j}\prn{\theta_{1}}}{\theta_{2}-\theta_{1}}$.
  \end{enumerate}
\end{assumption}
As mentioned earlier, \Cref{ass:potential_Ui} is satisfied if for every $i \in[b]$ and $j \in [N_i]$, $U_{i,j}$ is convex and $\nabla U_{i,j}$ is $\bMH$-Lipschitz continuous.
Under these additional conditions, the following non-asymptotic convergence results hold for the two reduced-variance MCMC algorithms described in \Cref{sec:QLSD}.
Denote by $Q_{\ostar,\gamma}$ the Markov kernel associated to \qlsds~ with a step size $\gamma \in\ocint{0,\bgamma}$.
\begin{theorem}\label{thm:QLSD_VR}
  Assume \Cref{main:ass:potential_U}, \Cref{main:ass:compression}, \Cref{ass:A_k} and \Cref{ass:potential_Ui}.
    Then, there exists $\bgamma_{\ostar,\infty}$ such that for $\bgamma < \bgamma_{\ostar,\infty}$, there exist $A_{\ostar,\bgamma},B_{\ostar,\bgamma} >0$ (explicitly given in Section S2 in the Supplementary Material) satisfying for any probability measure $\mu\in\mathcal{P}_{2}\pr{\Rd}$, any step size
  $\gamma\in\ocint{0,\bgamma}$ and $k \in \N$,\vspace{-0.2cm}
  \begin{align*}
    &\wass^{2}\pr{\mu Q_{\ostar,\gamma}^{k},\pi}
    \le (1 - \gamma \mtt/2)^k \cdot \wass^{2}\pr{\mu,\pi} + \gamma B_{\ostar,\bgamma}\\
    &+ \gamma^2 A_{\ostar,\bgamma} (1-\mtt \gamma/2)^{k-1} k \cdot \int_{\Rd}\|\theta-\theta^\star\|^2\mu(\rmd\theta)\eqsp,
  \end{align*}
  where $\thetas$ is defined in \eqref{eq:def_theta_star}.
\end{theorem}
\vspace{-0.25cm}Compared to \qlsd~and \qlsds, \texttt{QLSD}$^{++}$ only defines an inhomogeneous Markov chain, see Section S3.3 in the Supplementary Material for more details.
For a step-size $\gamma \in \ocint{0,\bgamma}$ and an iteration $k\in\nset$, we denote by $\mu Q_{\oplus,\gamma}^{(k)}$ the distribution of $\theta_k$ defined by \qlsdpp~starting from $\theta_0$ with distribution $\mu$.
\begin{theorem}\label{theorem_QLSDpp}
  Assume \Cref{main:ass:potential_U}, \Cref{main:ass:compression}, \Cref{ass:A_k} and \Cref{ass:potential_Ui}, and let $l \in \N^*$ and $\alpha \in (0,1/(\omega+1)]$.
    Then, there exists $\bgamma_{\oplus,\infty}$ such that for $\bgamma < \bgamma_{\oplus,\infty}$, there exist $A_{\oplus,\bgamma},B_{\oplus,\bgamma},C_{\oplus,\bgamma} >0$ (explicitly given in Section S3 in the Supplementary Material and independent of $\alpha$) satisfying for any probability measure $\mu\in\mathcal{P}_{2}\pr{\Rd}$, any step size
  $\gamma\in\ocint{0,\bgamma}$ and $k \in \N$,\vspace{-0.2cm}
  \begin{align*}
    &\wass^{2}\prn{\mu Q_{\pstar,\gamma}^{(k)},\pi}
    \le (1 - \gamma \mtt/2)^k \cdot \wass^{2}\pr{\mu,\pi}  + \gamma B_{\oplus,\bgamma} \\
    &+ \gamma^2 A_{\oplus,\bgamma} (1-\gamma \mtt/2)^{\floor{k/l}} \cdot \int_{\Rd}\|\theta-\theta^\star\|^2\mu(\rmd\theta) \\
    &+ \gamma C_{\oplus,\bgamma} [(1-\alpha)^k\wedge (1-\gamma \mtt/2)^{\floor{k/l}}]\sum_{i=1}^b\norm{\nabla U_i(\theta^\star)}^2,
  \end{align*}
  where $\thetas$ is defined in \eqref{eq:def_theta_star}.
\end{theorem}
\vspace{-0.2cm}\Cref{table:bias} provides the dependencies of the asymptotic bias terms $ B_{\ostar,\bgamma},  B_{\oplus,\bgamma}$ as $\bgamma \downarrow 0$ with respect to key quantities associated to the problem we consider.
For comparison, we do the same regarding the specific instance of \Cref{algo:QLSD}, \texttt{QLSD}$^{\#}$.
Remarkably, thanks to biased local stochastic gradients for \texttt{QLSD}$^\star$ and the memory mechanism for \texttt{QLSD}$^{++}$, we can notice that their associated asymptotic biases do not depend on local posterior discrepancy in contrast to \texttt{QLSD}$^{\#}$.
This is in line with non-asymptotic convergence results in risk-based FL which also show that the impact of data heterogeneity can be alleviated using such a memory mechanism \citep{Artemis20}.
The impact of stochastic gradients is discussed in further details in the next paragraph.

\paragraph{Consistency Analysis in the Big Data Regime} In \citet{brosse2018promises}, it was shown that \texttt{ULA} and \texttt{SGLD} define homogeneous Markov chains, each of which admits a unique stationary distribution. However, while the invariant distribution of \texttt{ULA} gets closer to $\pi$ as $N_i$ increases, conversely the invariant measure of \texttt{SGLD} never approaches $\pi$ and is in fact very similar to the invariant measure of \texttt{SGD}.
Moreover, the non-compressed counterpart of \texttt{QLSD}$^{\star}$ has been shown not to suffer from this problem, and it has been theoretically proven to be a viable alternative to \texttt{ULA} in the Big Data environment.
Since \qlsd~ is a generalisation of \texttt{SGLD}, the conclusions of \citet{brosse2018promises} hold.
On the other hand, we show that the reduced-variance alternatives to \texttt{QLSD} that we introduced provide more accurate estimates of $\pi$ as $N_i$ increases, see the last column in \Cref{table:bias}.
Detailed calculations are deferred to Section S4 in the Supplementary Material.

\section{NUMERICAL EXPERIMENTS}
\label{sec:experiments}
This section illustrates our methodology with three numerical experiments that include both synthetic and real datasets.
For all experiments, we consider the finite-sum setting and use the stochastic quantisation operator $\up^{(s)}$ for $s \geq1$ defined in \eqref{eq:def_quantisation_operator} to perform the compression step. In this case \Cref{main:ass:compression}-\ref{main:ass:compression:variance} is verified with $\omega = \min(d/s^2,\sqrt{d}/s)$.
Further experimental results are given in Section S5 in the Supplementary Material.
\begin{figure}
  \begin{center}
    \includegraphics[scale=0.55]{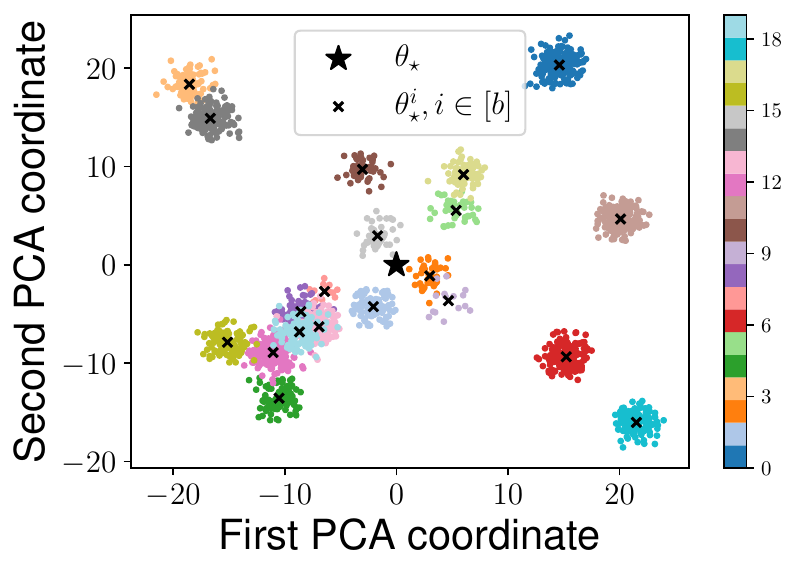}
    \includegraphics[scale=0.55]{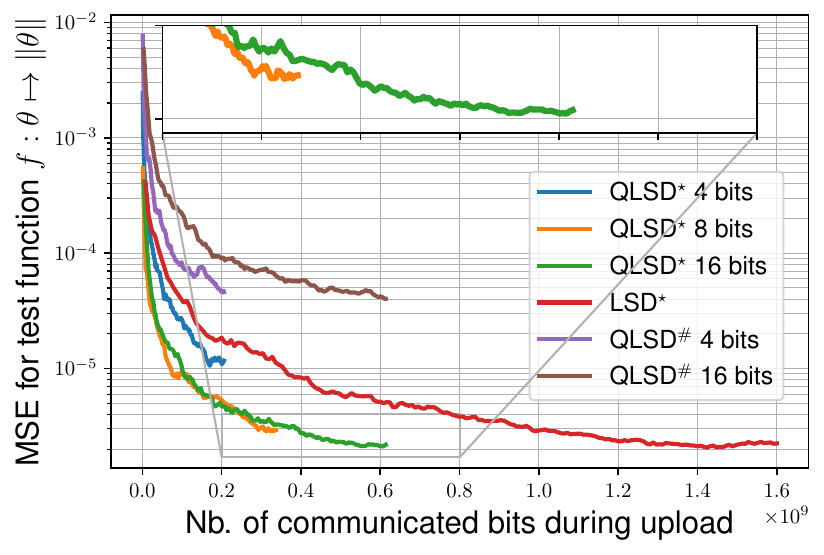}
  \end{center}
  \vspace{-0.4cm}
  \caption{Toy Gaussian example. (top) 2D projection of the heterogeneous synthetic dataset where each color refers to a client and each dot is an observation $y_{i,j}$. (bottom) Estimation performances of the considered Bayesian FL algorithms.\label{fig:toy_example}}
  \vspace{-0.4cm}
\end{figure}

\vspace{-0.1cm}
\paragraph{Toy Gaussian Example} This first experiment aims at illustrating the general behavior of \Cref{algo:QLSD} with respect to the use of stochastic gradients and compression scheme.
To this purpose, we set $b=20$ and $d=50$ and consider a Gaussian posterior distribution with density defined in \eqref{eq:target_density} where, for any $i \in [b]$ and $\theta \in \Rd$, $U_i(\theta) = \sum_{j=1}^{N_i} \|\theta - y_{i,j}\|^2/2$, $\{y_{i,j}\}_{i \in [b],j \in [N_i]}$ being a set of synthetic independent but not identically distributed observations across clients and $N_i \in [10,200]$, see \Cref{fig:toy_example} (top row).
Note that in this specific case, $\theta^\star$ admits a closed form expression.
For all the algorithms, we choose the (optimised) step-size $\gamma = 4.9 \times 10^{-4}$ and choose a minibatch size $n_i=\floor{N_i/10}$.
Instances of \qlsdsharp~and \qslds~using $s = 2^p$ are referred to as p-bits instances of these MCMC algorithms.
We compare these algorithms with the non-compressed counterpart of \qslds~referred to as \texttt{LSD}$^\star$, see \Cref{algo:LSD-star}.
\Cref{fig:toy_example} shows the behavior of the mean squarred error (MSE) associated to the test function $f:\theta \mapsto \norm{\theta}$, computed using 30 independent runs of each algorithm, with respect to the number of bits transmitted.
We can notice that \qlsds~always outperforms \qlsdsharp and that decreasing the value of $\omega$ does not significantly reduce the bias associated to \qlsds.
This illustrates the impact of the variance of the stochastic gradients and supports our theoretical analysis summarised in \Cref{table:bias}.
On the other hand, \texttt{QLSD}$^\star$ with $s=2^{16}$ achieves a similar MSE as \texttt{LSD}$^\star$ while requiring roughly 2.5 times less number of bits.

\vspace{-0.2cm}
\paragraph{Bayesian Logistic Regression} In this experiment, we compare the proposed methodology based on gradient compression with two existing \texttt{FedAvg}-type MCMC algorithms.
Since $\theta^\star$ defined in \eqref{eq:def_theta_star} is not easily available, we implement \texttt{QLSD}$^{++}$ detailed in \Cref{algo:QLSD-VR}.
We adopt a zero-mean Gaussian prior with covariance matrix $2\cdot10^{-2} \mathrm{I}_d$ and use the \textsc{FEMNIST} dataset \citep{FEMNIST}.
We set $b=50$, $l=100$, $\alpha = 1/(\omega+1)$ and $\gamma=10^{-5}$.
We launch \texttt{QLSD}$^{++}$ for $s \in \{2^4,2^8,2^{16}\}$ and compare its performances with \texttt{DG-SGLD} \citep{2021_ICML_DGLMC} and \texttt{FSGLD} \citep{2020_conducive_gradients} which use multiple local steps to address the communication bottleneck.
We are interested in performing uncertainty quantification by estimating highest posterior density (HPD) regions.
For any $\alpha \in (0,1)$, we define $\mathcal{C}_{\alpha} = \{\theta \in \Rd ; -\log \pi(\theta|\mathrm{D}) \le \eta_{\alpha}\}$ where $\eta_{\alpha} \in \mathbb{R}$ is chosen such that $\int_{\mathcal{C}_{\alpha}}\pi(\theta|\mathrm{D})\dd\theta = 1-\alpha$.
We compute the relative HPD error based on the scalar summary $\eta_{\alpha}$, \emph{i.e.} $|\eta_{\alpha}-\eta_{\alpha}^{\texttt{LSD}}|/\eta_{\alpha}^{\texttt{LSD}}$ where $\eta_{\alpha}^{\texttt{LSD}}$ has been estimated using the non-compressed counterpart of \texttt{QLSD}$^{++}$, referred to as \texttt{LSD}$^{++}$ and standing for a serial variance-reduced \texttt{SGLD}, see \Cref{algo:LSD-plus}.
\Cref{table:log_reg} gives this relative HPD error for $\alpha=0.01$ and provides the relative efficiency of \texttt{QLSD}$^{++}$ and competitors corresponding to the savings in terms of transmitted bits per iteration.
One can notice that the proposed approach provides similar results as its non-compressed counterpart while being 3 to 7 times more efficient.
In addition, we show that \texttt{QLSD}$^{++}$ provides similar performances as \texttt{DG-SGLD} and \texttt{FSGLD} which highlight that gradient compression and periodic communication are competing approaches.
\begin{table}
\caption{Bayesian Logistic Regression.}\label{table:log_reg}
\centering
\begin{tabular}{ccc}\\\toprule
Algorithm & 99\% HPD error & Rel. efficiency \\\midrule
\texttt{FSGLD} &5.4e-3 & 6.2\\
\texttt{DG-SGLD} &5.2e-3 & 6.4\\
\texttt{QLSD}$^{++}$ 4 bits &6.1e-3 & 7.6\\
\texttt{QLSD}$^{++}$ 8 bits &4.3e-3 & 6.7\\
\texttt{QLSD}$^{++}$ 16 bits &6.9e-4 & 3.1\\
\bottomrule
\end{tabular}
\vspace{-0.3cm}
\end{table}

\begin{table*}  
\centering 
\caption{Performances of Bayesian FL algorithms on the considered Bayesian neural networks problem.}   
    \begin{tabular}{ccccccccc} \\ \toprule
    Method &\texttt{HMC} & \texttt{SGLD} & \texttt{QLSD}$^{++}$ & \texttt{QLSD}$^{++}$ \texttt{PP} & \texttt{FedBe-Dirichlet} & \texttt{FedBe-Gauss.} &  \texttt{DG-SGLD} &  \texttt{FSGLD} \\
    \midrule
    Accuracy & 89.6 & 88.8 & 88.1 & 86.6 & 90.7 & 90.2 & 92.2 & 87.5 \\
    Agreement & 0.94 & 0.91 & 0.90 & 0.90 & 0.90 & 0.89 & 0.91 & 0.91 \\
    TV & 0.07 & 0.11 & 0.12 & 0.12 & 0.16 & 0.16 & 0.13 & 0.13 \\
    \bottomrule
    \end{tabular}
    \label{table:bbn_comparison}
\end{table*}

\paragraph{Bayesian Neural Networks}
In our third experiment, we go beyond the scope of our theoretical analysis by performing posterior inference in Bayesian neural networks.
We use the {ResNet-20} model \citep{he2016deep}, choose a zero-mean Gaussian prior distribution with variance $1/5$ and consider the classification problem associated with the \textsc{CIFAR-10} dataset \citep{krizhevsky2009learning}.
We run \texttt{QLSD}$^{++}$ with $s=2$, $l=20$, $\alpha=1/(\omega+1)$, and with either $p=1$ (full participation) or $p=0.25$ (partial participation).
We compare the proposed methodology with a long-run Hamiltonian Monte Carlo (\texttt{HMC}) considered as a ``ground truth'' \citep{izmailov2021bayesian} and $\texttt{SGLD}$.
For completeness, we also implement four other distributed/federated approximate sampling approaches, namely two instances of \texttt{FedBe} \citep{FedBE21}, \texttt{DG-SGLD} and \texttt{FSGLD}.
Following \citet{wilson2021evaluating}, we compare the aforementioned algorithms through three metrics: classification \emph{accuracy} on the test dataset using the minimum mean-square estimator, \emph{agreement} between the top-1 prediction given by each algorithm and the one given by \texttt{HMC} and \emph{total variation} between approximate and ``true'' (associated with \texttt{HMC}) predictive distributions.
More details about algorithms' hyperparameters and considered metrics are given in Section S5.3 in the Supplementary Material.
The results we obtain are gathered in \Cref{table:bbn_comparison}. In terms of agreement and total variation, \texttt{QLSD}$^{++}$ (even with partial participation) gives similar results as \texttt{SGLD} and competes favorably with other existing federated approaches.
\Cref{fig:BNN} complements this empirical analysis by showing calibration curves of posterior predictive distributions.

\begin{figure}
  \begin{center}
    \mbox{{\includegraphics[scale=0.65]{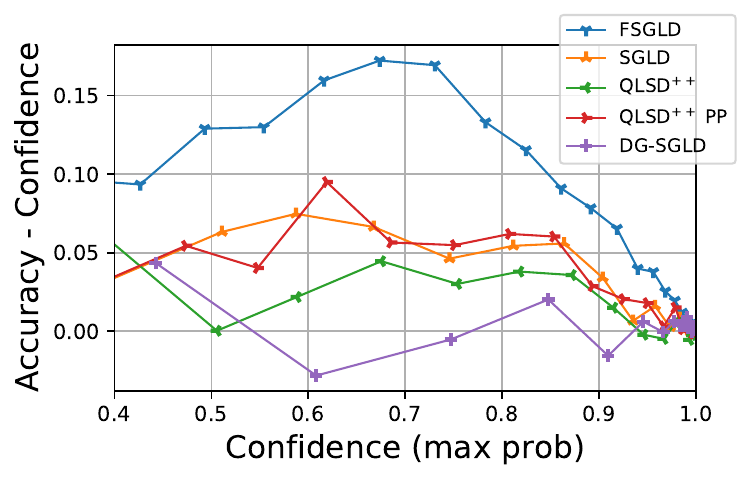}}}
  \end{center}
  \vspace{-0.4cm}
   \caption{Bayesian Neural Networks. \label{fig:BNN}}
   \vspace{-0.4cm}
\end{figure}

\section{CONCLUSION}
\label{sec:conclu}
In this paper, we presented a general methodology based on Langevin stochastic dynamics for Bayesian FL.
In particular, we addressed the challenges associated with this new ML paradigm by assuming that a subset of clients sends compressed versions of its local stochastic gradient oracles to the central server.
Moreover, the proposed method was found to have favorable convergence properties, as evidenced by numerical illustrations.
In particular, it compares favorably to \texttt{FedAvg}-type Bayesian FL algorithms.
A limitation of this work is that the proposed method does not target the initial posterior distribution due to the use of a fixed discretisation time step.
Therefore, this work paves the way for more advanced Bayesian FL approaches based, for example, on Metropolis-Hastings schemes to remove asymptotic biases.
In addition, although the data ownership issue is implicitly tackled by the FL paradigm by not sharing data, stronger privacy guarantees can be ensured, typically by combining differential privacy, secure multi-party computation and homomorphic encryption methods. 
Proposing a differentially private version of our methodology is a possible extension of our work, that is left for further work.
This work has no direct societal impact.

\subsubsection*{Acknowledgements}

The authors acknowledge support of the Lagrange Mathematics and Computing Research Center. 

\bibliography{biblio}
\bibliographystyle{plainnat}


\clearpage
\appendix

\thispagestyle{empty}

{\hsize\textwidth
    \linewidth\hsize \toptitlebar {\centering
        {\Large\bfseries Supplementary Material:\\QLSD: Quantised Langevin Stochastic Dynamics for Bayesian Federated Learning \par}}
    \bottomtitlebar}

\theoremstyle{plain}
\newtheorem{unlemma}{Lemma S}
\newtheorem{unproposition}{Proposition S}
\newtheorem{uncorollary}{Corollary S}
\newtheorem{untheorem}{Theorem S}

\setcounter{equation}{0}
\setcounter{assumption}{0}
\setcounter{figure}{0}
\setcounter{table}{0}
\setcounter{algorithm}{0}
\setcounter{page}{1}
\makeatletter
\renewcommand{\theequation}{S\arabic{equation}}
\renewcommand{\theassumption}{S\arabic{assumption}}
\renewcommand{\thefigure}{S\arabic{figure}}
\renewcommand{\thetable}{S\arabic{table}}
\renewcommand{\thetheorem}{S\arabic{theorem}}
\renewcommand{\thelemma}{S\arabic{lemma}}
\renewcommand{\thesection}{S\arabic{section}}
\renewcommand{\theremark}{S\arabic{remark}}
\renewcommand{\theproposition}{S\arabic{proposition}}
\renewcommand{\thecorollary}{S\arabic{corollary}}
\renewcommand{\thealgorithm}{S\arabic{algorithm}}
 
\paragraph{Notations and conventions.}

We denote by $\mathcal{B}\parentheseLigne{\mathbb{R}^d}$ the Borel $\sigma$-field of $\mathbb{R}^d$, $\mathbb{M}\parentheseLigne{\mathbb{R}^d}$ the set of all Borel measurable functions $f$ on $\mathbb{R}^d$ and $\norm{\cdot}$ the Euclidean norm on $\mathbb{R}^d$.
For $\mu$ a probability measure on $\parentheseLigne{\Rd,\mathcal{B}\parentheseLigne{\Rd}}$ and $f \in \mathbb{M}\parentheseLigne{\mathbb{R}^d}$ a $\mu$-integrable function, denote by $\mu\parentheseLigne{f}$ the integral of $f$ with respect to (w.r.t.) $\mu$.
Let $\mu$ and $\nu$ be two sigma-finite measures on $\parentheseLigne{\Rd,\mathcal{B}\parentheseLigne{\Rd}}$. 
Denote by $\mu \ll \nu$ if $\mu$ is absolutely continuous w.r.t. $\nu$ and $\rmd \mu/\rmd \nu$ the associated density. 
We say that $\zeta$ is a transference plan of $\mu$ and $\nu$ if it is a probability measure on $\parentheseLigne{\Rd \times \Rd,\mathcal{B}\parentheseLigne{\Rd \times \Rd}}$ such that for all measurable set $\mathsf{A}$ of $\Rd$, $\zeta\parentheseLigne{\mathsf{A} \times \Rd} = \mu\parentheseLigne{\mathsf{A}}$
and $\zeta\parentheseLigne{\Rd \times \mathsf{A}} = \nu\parentheseLigne{\mathsf{A}}$.
We denote by $\mathcal{T}\parentheseLigne{\mu,\nu}$ the set of transference plans of $\mu$ and $\nu$.
 In addition, we say that a couple of $\mathbb{R}^d$-random variables $\parentheseLigne{X,Y}$ is a coupling of $\mu$ and $\nu$ if there exists $\zeta \in \mathcal{T}\parentheseLigne{\mu,\nu}$ such that $\parentheseLigne{X,Y}$ are distributed according to $\zeta$.
 We denote by $\mathcal{P}_{2}\parentheseLigne{\Rd}$ the set of probability measures with finite $2$-moment: for all $\mu \in \mathcal{P}_{2}\parentheseLigne{\Rd},\int_{\Rd} \|x\|^{2} \rmd\mu\parentheseLigne{x} < \infty$. 
We define the squared Wasserstein distance of order $2$ associated with $\|\cdot\|$ for any probability measures $\mu,\nu \in \mathcal{P}_{2}\parentheseLigne{\Rd}$ by
\begin{equation*}
W_{2}^{2} \parentheseLigne{\mu,\nu} = \inf_{\zeta \in \mathcal{T}\parentheseLigne{\mu,\nu}} \int_{\mathbb{R}^d \times \mathbb{R}^d}\|x-y\|^{2}\mathrm{d}\zeta\parentheseLigne{x,y} \eqsp.
\end{equation*}
By \citet[Theorem 4.1]{Villani2008}, for all $\mu$, $\nu$ probability measures on $\Rd$, there exists a transference plan $\zeta^{\star} \in \mathcal{T}\parentheseLigne{\mu,\nu}$ such that for any coupling $\parentheseLigne{X,Y}$ distributed according to $\zeta^{\star}$, $W_{2}\parentheseLigne{\mu,\nu} = \E[\|x-y\|^{2}]^{1/2}$. 
This kind of transference plan (respectively coupling) will be called an optimal transference plan (respectively optimal coupling) associated with $W_{2}$. 
By \citet[Theorem 6.16]{Villani2008}, $\mathcal{P}_{2}\parentheseLigne{\Rd}$ equipped with the
Wasserstein distance $W_{2}$ is a complete separable metric space.
For the sake of simplicity, with little abuse, we shall use the same notations for
a probability distribution and its associated probability density function.
For $n \ge 1$, we refer to the set of integers between $1$ and $n$ with the notation $[n]$ and $\wp_n$ the power set of $[n]$.
The $d$-multidimensional Gaussian probability distribution with mean $\mu \in \Rd$ and covariance matrix $\Sigma \in \mathbb{R}^{d \times d}$ is denoted by $\gauss\parentheseLigne{\mu,\Sigma}$.

\section{PROOF OF \Cref{thm:QLSD}}


This section aims at proving \Cref{thm:QLSD} in the main paper.

\subsection{Generalised quantised Langevin stochastic dynamics}
We  show  that \texttt{QLSD} defined in \Cref{algo:QLSD} in the main paper can be cast into a more general framework
that we refer to as generalised quantised Langevin stochastic
dynamics. Then, the guarantees for {\qlsd}~will be a simple consequence of the ones that we will establish for generalised {\qlsd}.    
For ease of reading, we recall first the setting and the
assumptions that we consider all along the paper.  Recall that the dataset $\mathrm{D}$ is assumed to be partitioned into $b$ \emph{shards} $\{\mathrm{D}_{i}\}_{i=1}^{b}$ such that
$\sqcup_{i=1}^{b} \mathrm{D}_{i} = \mathrm{D}$ and the posterior distribution of interest is assumed
to admit a density with respect to the $d$-dimensional Lebesgue
measure which factorises across clients, \emph{\emph{i.e.}} for any
$\theta\in\Rd$,
\begin{align*}\label{eq:target_density_sup}
  &\pi\parentheseLigne{\theta} = \exp\ac{-U\parentheseLigne{\theta}} / \int_{\Rd} \mathrm{e}^{-U\parentheseLigne{\theta}}\,\rmd \theta \eqsp,&
  &\U\parentheseLigne{\theta} = \sum_{i=1}^{b} U_{i}\parentheseLigne{\theta}\eqsp.
\end{align*}
We consider the following assumptions on the potential $\U$.
\begin{assumption}\label{ass:potential_U}
  For any $i\in[b]$, $U_{i}$ is continuously differentiable. In addition, suppose that the following conditions hold.
  \begin{enumerate}[wide, labelwidth=!, labelindent=0pt,label=(\roman*),noitemsep,nolistsep]
  \item \label{ass:potential_U:1} $U$ is $\mU$-strongly convex, \emph{i.e.} for any $\theta_{1},\theta_{2} \in \Rd$,
    \begin{equation*}
    U\parentheseLigne{\theta_{1}} \ge U\parentheseLigne{\theta_{2}} +\ps{\theta_{1}-\theta_{2}}{\nabla\U\parentheseLigne{\theta_{2}}} +\mU\norm{\theta_{1}-\theta_{2}}^{2}/2\eqsp.
    \end{equation*}
  \item \label{ass:potential_U:2} $\U$ is $\lip$-Lipschitz, \emph{i.e.} for any $\theta_{1},\theta_{2} \in \Rd$,
    \begin{equation*}\label{eq:bound:lipschitz}
      \norm{\nabla\U\parentheseLigne{\theta_{1}}-\nabla\U\parentheseLigne{\theta_{2}}}
      \le \lip\norm{\theta_{1}-\theta_{2}}\eqsp.
    \end{equation*}
    \end{enumerate}
\end{assumption}
Note that \Cref{ass:potential_U}-\ref{ass:potential_U:1} implies that $\U$ admits a unique minimiser denoted by $\theta^{\star}\in\Rd$.
Moreover, for any $(\theta_1,\theta_1)\in\Rd$, \Cref{ass:potential_U}-\ref{ass:potential_U:1}-\ref{ass:potential_U:2} combined with \citet[Equation 2.1.24]{nesterov2003introductory} shows that
\begin{equation}\label{eq:intro:Nesterov}
  \ps{\nabla\U \parentheseLigne{\theta_2}-\nabla\U \parentheseLigne{\theta_1}}{\theta_2-\theta_1}
  \ge \frac{\mU \lip}{\mU+\lip}\norm{\theta_2-\theta_1}^2
  +\frac{1}{\mU+\lip}\norm{\nabla\U \parentheseLigne{\theta_2}-\nabla\U \parentheseLigne{\theta_1}}^2\eqsp.
\end{equation}
We consider the following assumptions on the family $\acn{\grad_{i}:\Rd\times \msx_{1}\to\Rd}_{i\in[b]}$ and  $\up$. 

  \begin{assumption}\label{ass:compression}
  There exists a probability measure $\nu_{2}$ on a measurable space $(\msx_{2},\mathcal{X}_{2})$ and a family of measurable functions $\{\up_i:\Rd\times\msx_2\to\Rd\}_{i\in [b]}$ such that the following conditions hold.   
  \begin{enumerate}[wide, labelwidth=!, labelindent=0pt,label=(\roman*),noitemsep,nolistsep]
    \item \label{ass:compression:unbiased} For any $\theta\in\R^d$ and any $i \in [b]$, $\int_{\msx_{2}}\up_i\parentheseLigne{\theta,x^{(2)}}\,\nu_{2}\parentheseLigne{\rmd x^{(2)}} = \theta$.
    \item \label{ass:compression:variance} There exist $\{\omega_i\in\R_+\}_{i\in[b]}$, such that for any $\theta\in\R^d$ and any $i \in [b]$,    
    \begin{equation*}\int_{\msx_{2}}\norm{\up_i\parentheseLigne{\theta,x^{\parentheseLigne{2}}}-\theta}^{2}\,\nu_{2}\parentheseLigne{\rmd x^{\parentheseLigne{2}}} \le \omega_i\norm{\theta}^{2}\eqsp.\end{equation*}
    \end{enumerate}
  \end{assumption}
  \begin{assumption}\label{ass:stochastic_gradient}
    There exist a family of probability measures $\{\nu_{1}^{(i)}\}_{i \in [b]}$ defined on measurable spaces $\{(\msx_{1}^{(i)},\mathcal{X}_{1}^{(i)})\}_{i \in [b]}$ and a family of measurable functions $\acn{\grad_{i}:\Rd\times \msx_{1}^{(i)} \to\Rd}_{i\in[b]}$ such that the following conditions hold.
    \begin{enumerate}[wide, labelwidth=!, labelindent=0pt,label=(\roman*),noitemsep,nolistsep]
    \item \label{ass:stochastic_gradient:1} For any $\theta\in\Rd$,
      \begin{equation*}        
        \sum_{i=1}^{b}\int_{\msx_{1}^{(i)}}\grad_{i}\parentheseLigne{\theta,x^{\parentheseLigne{1,i}}}\nu_{1}^{(i)}(\rmd x^{(1,i)}) = \nabla\U(\theta)\eqsp.
      \end{equation*}
    \item \label{ass:stochastic_gradient:2}
      There exist $\{\MH_i > 0\}_{i \in [b]}$, such that for any $i \in [b]$, $\theta_{1},\theta_{2} \in \Rd$,
      \begin{equation*}
        \int_{\msx_{1}^{(i)}}\norm{\grad_{i}\parentheseLigne{\theta_{2},x^{\parentheseLigne{1,i}}}-\grad_{i}\parentheseLigne{\theta_{1},x^{\parentheseLigne{1,i}}}}^{2}\nu_{1}^{(i)}\parentheseLigne{\rmd x^{\parentheseLigne{1,i}}} \le \MH_i \ps{\theta_{2}-\theta_{1}}{\nabla\U_{i}\parentheseLigne{\theta_{2}}-\nabla\U_{i}\parentheseLigne{\theta_{1}}}\eqsp.
      \end{equation*}
    \item \label{ass:stochastic_gradient:3} There exists $\sigmas,\Bs\in\R_+$ such that for any $i\in [b]$, $\theta\in\Rd$, we have
      \begin{align} \label{eq:stochastic_gradient:3}
        &\int_{\msx_{1}^{(i)}}\norm{\grad_{i}\parentheseLigne{\theta^{\star},x^{\parentheseLigne{1}}}}^2\nu_{1}^{(i)}\parentheseLigne{\rmd x^{\parentheseLigne{1}}}\le \Bs/b \eqsp,&       &\int_{\msx_{1}^{(1)} \times \cdots \times \msx_{1}^{(b)}}\norm{\sum_{i=1}^{b}\grad_{i}\parentheseLigne{\theta^{\star},x^{\parentheseLigne{1,i}}}}^2 \otimes_{i=1}^b \nu_{1}^{(i)}\parentheseLigne{\rmd x^{\parentheseLigne{1,i}}} \le \sigmas^2\eqsp.
      \end{align}
    \end{enumerate}
  \end{assumption}
  We can notice that \Cref{ass:stochastic_gradient}-\ref{ass:stochastic_gradient:2} implies that $\nabla\U_{i}$ is $\MH_i$-Lipschitz continuous since by the Cauchy Schwarz inequality, for any $i \in [b]$ and any $\theta_{1},\theta_{2}\in\Rd$,
  \begin{equation*}\label{eq:bound:jensen}
    \norm{\nabla\U_{i}\parentheseLigne{\theta_{1}}-\nabla\U_{i}\parentheseLigne{\theta_{2}}}^2
    \le \MH_i\ps{\theta_{1}-\theta_{2}}{\nabla\U_{i}\parentheseLigne{\theta_{1}}-\nabla\U_{i}\parentheseLigne{\theta_{2}}}\eqsp.
  \end{equation*}
  In addition, it is worth mentioning that the first inequality in \eqref{eq:stochastic_gradient:3} is also required for our derivation in the deterministic case where $\grad_{i} = \nabla U_{i}$ for any $i \in [b]$ due to the compression step. 
  For $k\ge 1$, consider $\parentheseLigne{X_{k}^{(1,1)},\ldots,X_{k}^{(1,b)}}_{k \in \mathbb{N}}$ and $\parentheseLigne{X_{k}^{(2,1)},\ldots,X_{k}^{(2,b)}}_{k \in \mathbb{N}}$ two independent sequences of random variables distributed according to $\nu_1^{(1:b)} = \nu_{1}^{(1)} \otimes \cdots \otimes \nu_{1}^{(b)}$ and $\nu_{2}^{\otimes b}$, respectively.

  In addition, we consider the partial device participation context where at each communication round $k\ge1$, each client has a probability $p_i \in (0,1]$ of participating, independently from other clients.

  \begin{assumption}
    \label{ass:A_k_supp}
  For any $k \in \N^*$, $\mathcal{A}_{k} = \{ i \in [b] \,: \, B_{i,k}= 1 \}$ where for any $i \in [b]$, $\{B_{i,k}\,: \, , \, k\in\nsets\}$ is a family of \iid~Bernouilli random variables with success probability $p_i \in\ocint{0,1}$.
  \end{assumption}

  In other words, there exists a sequence $(X_k^{(3,1)},\cdots,X_k^{(3,b)})_{k\in \mathbb{N}}$ of i.i.d.  random variables distributed according $\nu_3 = \mathrm{Uniform}((0,1])$, such that for any $k\ge1$ and $i\in[b]$, client $i$ is active at step $k$ if $X_k^{(3,i)} \le p_i$.
  We denote $\mathcal{A}_{k+1} = \{i \in [b]; X_{k+1}^{(3,i)} \le p_i\}$ the set of active clients at round $k$. 
  Given a step-size $\gamma \in (0,\bar{\gamma}]$ for some $\bar{\gamma} > 0$ and starting from $\theta_{0} \in \Rd$, {\qlsd} recursively defines $(\theta_{k})_{k \in \mathbb{N}}$, for any $k \in \mathbb{N}$, as
  \begin{equation}\label{eq:def:recursion:theta}
    \txts \theta_{k+1} = \theta_{k}-\gamma \sum_{i \in \mathcal{A}_{k+1}} (1/p_i)\up_i\parentheseLigne{\grad_{i}\parentheseLigne{\theta_{k},X_{k+1}^{(1,i)}},X_{k+1}^{(2,i)}} + \sqrt{2\gamma}Z_{k+1}\eqsp,
  \end{equation}
where $(Z_{k+1})_{k \in \mathbb{N}}$ is a sequence of standard Gaussian random variables.
Let $\msx_3 = [0,1]$.
For any $i \in [b]$, consider the unbiased partial participation operator $\mathscr{S}_i: \Rd \times \msx_3 \to \Rd$ defined, for any $\theta \in \Rd$ and $x^{(3)} \in \msx_3$ by 
\begin{equation}
  \mathscr{S}_i(\theta,x^{(3)}) = \mathbf{1}\{x^{(3)} \le p_i\} \theta / p_i\eqsp. \label{def:PP}
\end{equation}
Then, \eqref{eq:def:recursion:theta} can be written of the form
\begin{equation}\label{eq:def:recursion:theta_bis}
  \theta_{k+1} = \theta_{k}-\gamma \sum_{i=1}^b \tgrad_{i}\parentheseLigne{\theta_{k},X_{k+1}^{(i)}} + \sqrt{2\gamma}Z_{k+1}\eqsp,\qquad k\in\N\eqsp,
\end{equation}
where for any $i \in [b]$, we denote $X_{k+1}^{(i)}=\parentheseLigne{X_{k+1}^{(1,i)},X_{k+1}^{(2,i)},X_{k+1}^{(3,i)}}$ and for any $\theta \in \Rd, x^{(1,i)}\in\msx_{1}^{(i)}, x^{(2)}\in\msx_{2}$ and $x^{(3)} \in \msx_3$,
\begin{equation}\label{eq:def:tgrad}
  \tgrad_{i}\parentheseLignebig{\theta,(x^{(1,i)},x^{(2)},x^{(3)})} = \mathscr{S}_i \pr{\up_i\parentheseLigneBig{\grad_{i}\parentheseLigne{\theta,x^{(1,i)}},x^{(2)}},x^{(3)}}\eqsp.
\end{equation}
With this notation and setting for any $i \in [b]$
 $\tilde{\msx}^{(i)} = \msx_{1}^{(i)} \times \msx_{2} \times \msx_3$
 and $\tilde{\nu}^{(i)} = \nu_{1}^{(i)} \otimes \nu_{2} \otimes \nu_3$, the Markov kernel associated with
\eqref{eq:def:recursion:theta} is given for any
$(\theta,\msa)\in\Rd\times \mathcal{B}(\Rd)$ by
\begin{equation}\label{eq:def:intro:Q}
 Q_{\gamma}(\theta, \msa )
 = \int_{\msa \times \tilde{\msx}^{(1)} \times \cdots \times \tilde{\msx}^{(b)}} \exp\pr{-\normn{\btheta-\theta+\gamma\sum_{i=1}^b\tgrad_{i}\parentheseLigne{\theta,x^{(i)}}}^2/(4\gamma)}\frac{\,\rmd \tilde{\theta} \, \tilde{\nu}^{(1)}\parentheseLigne{\rmd x^{(1)}} \otimes \cdots \otimes \tilde{\nu}^{(b)}\parentheseLigne{\rmd x^{(b)}}}{\parentheseLigne{4\uppi\gamma}^{d/2}}\eqsp.
\end{equation}

The following result establishes an essential property of $\{\tgrad_{i}\}_{i\in[b]}$ under \Cref{ass:compression} and  \Cref{ass:stochastic_gradient}.
\begin{lemma}\label{lem:intro:1}
  Assume \Cref{ass:compression}, \Cref{ass:stochastic_gradient} and \Cref{ass:A_k_supp}. Then, for any $\theta\in\Rd$, we have
  \begin{align}
    \label{eq:bound:lem1}
    \textstyle\sum_{i=1}^{b}\int_{\tilde{\msx}^{(i)}} \tgrad_{i}\parentheseLigne{\theta,x^{(i)}}\,\rmd \tilde{\nu}^{(i)}(x^{(i)}) &= \nabla\U \parentheseLigne{\theta}\eqsp,\\ 
    \nonumber
    \int_{\tilde{\msx}^{(1:b)}}\norm{\textstyle\sum_{i=1}^{b}\tgrad_{i}\parentheseLigne{\theta,x^{(i)}}-\nabla\U\parentheseLigne{\theta}}^{2}\otimes_{i=1}^b\tilde{\nu}^{(i)}\parentheseLigne{\rmd x^{(i)}}
    &\le 2\max_{i \in [b]}\{\MH_i(\omega_i+1)/p_i\}\ps{\theta-\theta^{\star}}{\nabla\U\parentheseLigne{\theta}} \\
    &+2\br{\sigmas^2 + (\Bs/b)\sum_{i=1}^b(1-p_i+\omega_i)/p_i}\eqsp,\label{eq:bound:lem2}
  \end{align}
  where for any $i \in[b]$, $\tgrad_{i}$ is defined in \eqref{eq:def:tgrad}.
\end{lemma}
\begin{proof}
  The first identity \eqref{eq:bound:lem1} is straightforward using \Cref{ass:stochastic_gradient}-\ref{ass:stochastic_gradient:1} and \Cref{ass:compression}-\ref{ass:compression:unbiased}.
  We now show the inequality~\eqref{eq:bound:lem2}.
  Let $\theta\in\R^d$. Using \Cref{ass:compression}-\ref{ass:compression:unbiased} or \Cref{ass:stochastic_gradient}-\ref{ass:stochastic_gradient:1}, we get
  \begin{multline}\label{eq:eq:lem_compr}
    \int_{\tilde{\msx}^{(1:b)}}\norm{\textstyle\sum_{i=1}^{b}\tgrad_{i}\parentheseLigne{\theta,x^{(i)}}-\nabla\U\parentheseLigne{\theta}}^{2}\otimes_{i=1}^b\tilde{\nu}^{(i)}\parentheseLigne{\rmd x^{(i)}} \\
    =\int_{\tilde{\msx}^{(1:b)}}\norm{\sum_{i=1}^{b}\br{\tgrad_{i}\parentheseLigne{\theta,x^{(i)}}-\up_i\pr{\grad_{i}\parentheseLigne{\theta,x^{\parentheseLigne{1,i}}},x^{(2,i)}}}}^2\otimes_{i=1}^b\tilde{\nu}^{(i)}\parentheseLigne{\rmd x^{(i)}} \\
    +\int_{\msx_{1}^{(1:b)} \times \msx_{2}^{b}}\norm{\sum_{i=1}^{b}\up_i\pr{\grad_{i}\parentheseLigne{\theta,x^{\parentheseLigne{1,i}}},x^{(2,i)}}-\nabla\U\parentheseLigne{\theta}}^2\nu_{2}^{\otimes b}\parentheseLigne{\rmd x^{\parentheseLigne{2,1:b}}}\otimes_{i=1}^b\nu_{1}^{(i)}\parentheseLigne{\rmd x^{\parentheseLigne{1,i}}}\eqsp.
  \end{multline}
  In addition, by \Cref{ass:compression}-\ref{ass:compression:unbiased} and \Cref{ass:compression}-\ref{ass:compression:variance}, we obtain
  \begin{align}
    \nonumber
    &\int_{\tilde{\msx}^{(1:b)}}\norm{\sum_{i=1}^{b}\br{\tgrad_{i}\parentheseLigne{\theta,x^{(i)}}-\up_i\pr{\grad_{i}\parentheseLigne{\theta,x^{\parentheseLigne{1,i}}},x^{(2,i)}}}}^2\otimes_{i=1}^b\tilde{\nu}^{(i)}\parentheseLigne{\rmd x^{(i)}}\\
    \nonumber
    &= \sum_{i=1}^{b} \int_{\tilde{\msx}^{(i)}}\norm{\tgrad_{i}\parentheseLigne{\theta,x^{(i)}}-\up_i\pr{\grad_{i}\parentheseLigne{\theta,x^{\parentheseLigne{1,i}}},x^{(2,i)}}}^2\nu_1^{(i)}\parentheseLigne{\rmd x^{(1,i)}} \nu_2(\rmd x^{(2,i)})\nu_3(\rmd x^{(3,i)})  \\ 
    \nonumber
    &\le \sum_{i=1}^{b} \pr{\frac{1-p_i}{p_i}} \int_{\msx_1^{(i)} \times \msx_2}\norm{\up_i\pr{\grad_{i}\parentheseLigne{\theta,x^{\parentheseLigne{1,i}}},x^{(2,i)}}}^2\nu_1^{(i)}\parentheseLigne{\rmd x^{(1,i)}} \nu_2(\rmd x^{(2,i)}) \\
    \nonumber
    &= \sum_{i=1}^{b} \pr{\frac{1-p_i}{p_i}} \int_{\msx_1^{(i)} \times \msx_2}\norm{\up_i\pr{\grad_{i}\parentheseLigne{\theta,x^{\parentheseLigne{1,i}}},x^{(2,i)}} - \grad_{i}\parentheseLigne{\theta,x^{\parentheseLigne{1,i}}} + \grad_{i}\parentheseLigne{\theta,x^{\parentheseLigne{1,i}}}}^2\nu_1^{(i)}\parentheseLigne{\rmd x^{(1,i)}} \nu_2(\rmd x^{(2,i)}) \\
    \nonumber
    &= \sum_{i=1}^{b} \pr{\frac{1-p_i}{p_i}} \int_{\msx_1^{(i)} \times \msx_2}\norm{\up_i\pr{\grad_{i}\parentheseLigne{\theta,x^{\parentheseLigne{1,i}}},x^{(2,i)}} - \grad_{i}\parentheseLigne{\theta,x^{\parentheseLigne{1,i}}}}^2\nu_1^{(i)}\parentheseLigne{\rmd x^{(1,i)}} \nu_2(\rmd x^{(2,i)}) \\
    \nonumber
    &+ \sum_{i=1}^{b} \pr{\frac{1-p_i}{p_i}} \int_{\msx_1^{(i)}}\norm{\grad_{i}\parentheseLigne{\theta,x^{\parentheseLigne{1,i}}}}^2\nu_1^{(i)}\parentheseLigne{\rmd x^{(1,i)}} \\
    &\le \sum_{i=1}^{b} \br{\pr{\frac{1-p_i}{p_i}}(\omega_i + 1)} \int_{\msx_1^{(i)}}\norm{\grad_{i}\parentheseLigne{\theta,x^{\parentheseLigne{1,i}}}}^2\nu_1^{(i)}\parentheseLigne{\rmd x^{(1,i)}} \eqsp. \label{eq:bound:lem_compr:1} 
  \end{align}
  Using $\normn{a}^2\le2\normn{a-b}^2+2\normn{b}^2$ and \Cref{ass:stochastic_gradient}-\ref{ass:stochastic_gradient:2}-\ref{ass:stochastic_gradient:3}, for any $i\in[b]$, we obtain
  \begin{align*}    
  \int_{\msx_{1}^{(i)}}\norm{\grad_{i}\parentheseLigne{\theta,x^{\parentheseLigne{1,i}}}}^2\nu_{1}^{(i)}\parentheseLigne{\rmd x^{\parentheseLigne{1,i}}}
 &   \le 2\MH_i\ps{\theta-\theta^{\star}}{\nabla\U_{i}\parentheseLigne{\theta}-\nabla\U_{i}\parentheseLigne{\theta^{\star}}}\\ &\qquad +2\int_{\msx_{1}^{(i)}}\norm{\grad_{i}\parentheseLigne{\theta^{\star},x^{\parentheseLigne{1,i}}}}^2\nu_{1}^{(i)}\parentheseLigne{\rmd x^{\parentheseLigne{1,i}}}\\
&      \le 2\MH_i\ps{\theta-\theta^{\star}}{\nabla\U_{i}\parentheseLigne{\theta}-\nabla\U_{i}\parentheseLigne{\theta^{\star}}}+2 \Bs/b \label{eqmax3}\eqsp.
  \end{align*}
  Therefore, combining this result and  \eqref{eq:bound:lem_compr:1} gives
  \begin{align}
    &\int_{\tilde{\msx}^{(1:b)}}\norm{\sum_{i=1}^{b}\br{\tgrad_{i}\parentheseLigne{\theta,x^{(i)}}-\up_i\pr{\grad_{i}\parentheseLigne{\theta,x^{\parentheseLigne{1,i}}},x^{(2,i)}}}}^2\otimes_{i=1}^b\tilde{\nu}^{(i)}\parentheseLigne{\rmd x^{(i)}} \\
    &\le 2\sum_{i=1}^{b} \MH_i\pr{\frac{1-p_i}{p_i}}(\omega_i + 1) \ps{\theta-\theta^{\star}}{\nabla\U_{i}\parentheseLigne{\theta}-\nabla\U_{i}\parentheseLigne{\theta^{\star}}} 
    + \frac{2\Bs}{b} \sum_{i=1}^{b} \pr{\frac{1-p_i}{p_i}}(\omega_i + 1)\eqsp.\label{eq:bound:lem_compr:23}
  \end{align}
Similarly, by \Cref{ass:compression}-\ref{ass:compression:unbiased} and \Cref{ass:compression}-\ref{ass:compression:variance}, we have 
\begin{align*}
  \nonumber
  &\int_{\msx_{1}^{(1:b)} \times \msx_{2}^{b}}\norm{\sum_{i=1}^{b}\up_i\pr{\grad_{i}\parentheseLigne{\theta,x^{\parentheseLigne{1,i}}},x^{(2,i)}}-\nabla\U\parentheseLigne{\theta}}^2\nu_{2}^{\otimes b}\parentheseLigne{\rmd x^{\parentheseLigne{2,1:b}}}\otimes_{i=1}^b\nu_{1}^{(i)}\parentheseLigne{\rmd x^{\parentheseLigne{1,i}}} \\
  \nonumber
  &= \int_{\msx_{1}^{(1:b)} \times \msx_2^b}\Big \|\sum_{i=1}^{b}\br{\up_i\pr{\grad_{i}\parentheseLigne{\theta,x^{\parentheseLigne{1,i}}},x^{(2,i)}} - \grad_{i}\parentheseLigne{\theta,x^{\parentheseLigne{1,i}}}} \\
  \nonumber
  &+ \sum_{i=1}^b \{\grad_{i}\parentheseLigne{\theta,x^{\parentheseLigne{1,i}}}\}-\nabla\U\parentheseLigne{\theta}\Big\|^2\nu_{2}^{\otimes b}\parentheseLigne{\rmd x^{\parentheseLigne{2,1:b}}}\otimes_{i=1}^b\nu_{1}^{(i)}\parentheseLigne{\rmd x^{\parentheseLigne{1,i}}} \\
  \nonumber
  &= \sum_{i=1}^{b}\int_{\msx_{1}^{(i)} \times \msx_{2}}\norm{\up_i\pr{\grad_{i}\parentheseLigne{\theta,x^{\parentheseLigne{1,i}}},x^{(2)}} - \grad_{i}\parentheseLigne{\theta,x^{\parentheseLigne{1,i}}}}^2\nu_{2}(\rmd x^{(2)})\nu_{1}^{(i)}(\rmd x^{(1,i)}) \\
  \nonumber
  &+ \int_{\msx_{1}^{(1:b)}}\norm{ \sum_{i=1}^b\grad_{i}\parentheseLigne{\theta,x^{\parentheseLigne{1,i}}}-\nabla\U\parentheseLigne{\theta}}^2\otimes_{i=1}^b\nu_{1}^{(i)}\parentheseLigne{\rmd x^{\parentheseLigne{1,i}}} \\
  &\le \sum_{i=1}^{b}\omega_i\int_{\msx_{1}^{(i)}}\norm{\grad_{i}\parentheseLigne{\theta,x^{\parentheseLigne{1,i}}}}^2\nu_{1}^{(i)}(\rmd x^{(1,i)}) + \int_{\msx_{1}^{(1:b)}}\norm{ \sum_{i=1}^b\grad_{i}\parentheseLigne{\theta,x^{\parentheseLigne{1,i}}}-\nabla\U\parentheseLigne{\theta}}^2\otimes_{i=1}^b\nu_{1}^{(i)}\parentheseLigne{\rmd x^{\parentheseLigne{1,i}}} \label{eqmax1}\eqsp.
\end{align*}

Since for any $a,b\in\Rd, \norm{a+b}^2\le 2\norm{a}^2+2\norm{b}^2$, we have by \Cref{ass:stochastic_gradient}-\ref{ass:stochastic_gradient:1}
  \begin{align}
  &\int_{\msx_1^{(1:b)}}\norm{\sum_{i=1}^{b}\grad_{i}\parentheseLigne{\theta,x^{\parentheseLigne{1,i}}}-\nabla\U\parentheseLigne{\theta}}^2\otimes_{i=1}^b\nu_{1}^{(i)}\parentheseLigne{\rmd x^{\parentheseLigne{1,i}}}\\
  \nonumber
  &=\int_{\msx_1^{(1:b)}}\norm{\sum_{i=1}^{b}\br{\grad_{i}\parentheseLigne{\theta,x^{\parentheseLigne{1,i}}}-\int_{\msx_1^{(i)}}\grad_{i}\parentheseLigne{\theta,x^{\parentheseLigne{1}}}\nu_1^{(i)}\parentheseLigne{\rmd x^{(1)}}}}^2\otimes_{i=1}^b\nu_{1}^{(i)}\parentheseLigne{\rmd x^{\parentheseLigne{1,i}}}\\
  \nonumber
  &= \sum_{i=1}^{b}\int_{\msx_1^{(i)}}\norm{\grad_{i}\parentheseLigne{\theta,x^{(1,i)}}-\int_{\msx_1^{(i)}}\grad_{i}\parentheseLigne{\theta,x^{\parentheseLigne{1}}}\nu_1^{(i)}\parentheseLigne{\rmd x^{(1)}}}^2\nu_1^{(i)}\parentheseLigne{\rmd x^{(1,i)}}\\
  \nonumber
  &\le 2\sum_{i=1}^{b}\int_{\msx_1^{(i)}}\norm{\grad_{i}\parentheseLigne{\theta,x^{(1,i)}}-\grad_{i}\parentheseLigne{\theta^{\star},x^{(1,i)}}-\br{\int_{\msx_1^{(i)}}(\grad_{i}\parentheseLigne{\theta,x^{(1)}}-\grad_{i}\parentheseLigne{\theta^{\star},x^{(1)})}\nu_1^{(i)}\parentheseLigne{\rmd x^{(1)}}}}^2\nu_1^{(i)}\parentheseLigne{\rmd x^{(1,i)}}\\
  \nonumber
  &+2\sum_{i=1}^{b}\int_{\msx_1^{(i)}}\norm{\grad_{i}\parentheseLigne{\theta_\star,x^{(1,i)}}-\int_{\msx_1^{(i)}}\grad_{i}\parentheseLigne{\theta_\star,x^{\parentheseLigne{1}}}\nu_1^{(i)}\parentheseLigne{\rmd x^{(1)}}}^2\nu_1^{(i)}\parentheseLigne{\rmd x^{(1,i)}}\\
  &\le 2 \sigmas^2 + 2\sum_{i=1}^b\MH_i\ps{\nabla\U_i\parentheseLigne{\theta} - \nabla\U_i\parentheseLigne{\theta^\star}}{\theta-\theta^{\star}}  \label{eqmax2}\eqsp.
  \end{align}
  By combining \eqref{eqmax3}, \eqref{eqmax1} and \eqref{eqmax2}, we obtain
  \begin{align*}
    &\int_{\msx_{1}^{(1:b)}}\norm{\sum_{i=1}^{b}\up_i\pr{\grad_{i}\parentheseLigne{\theta,x^{\parentheseLigne{1,i}}},x^{(2,i)}}-\nabla\U\parentheseLigne{\theta}}^2\nu_{2}^{\otimes b}\parentheseLigne{\rmd x^{\parentheseLigne{2,1:b}}}\otimes_{i=1}^b\nu_{1}^{(i)}\parentheseLigne{\rmd x^{\parentheseLigne{1,i}}} \\
    &\le 2\sum_{i=1}^b \MH_i(\omega_i+1)\ps{\nabla\U_i\parentheseLigne{\theta} - \nabla\U_i(\theta^\star)}{\theta-\theta^{\star}}
    +2 \pr{\sigmas^2 + \frac{2\Bs}{b}\sum_{i=1}^b \omega_i}\eqsp.
  \end{align*}
  Finally, the last inequality combined with \eqref{eq:eq:lem_compr} and \eqref{eq:bound:lem_compr:23} completes the proof.
\end{proof}
In view of \Cref{lem:intro:1}, it suffices to study the recursion specified in \eqref{eq:def:recursion:theta_bis} under the following assumption on $\parentheseLigne{\tgrad_{i}}_{i\in[b]}$ gathered in \Cref{ass:tgrad:sharp}. Indeed, \Cref{lem:intro:1} shows that Condition \Cref{ass:tgrad:sharp} below holds with $\msx^{(i)}=\tilde{\msx}^{(i)} = \msx_{1}^{(i)} \times \msx_{2} \times \msx_{3}$, $\mathcal{X}^{(i)}=\tilde{\mathcal{X}}^{(i)} = \mathcal{X}_{1}^{(i)} \otimes \mathcal{X}_{2}\otimes \mathcal{X}_{3}$, 
$\tilde{\nu}^{(i)} = \nu_{1}^{(i)} \otimes \nu_{2} \otimes \nu_{3}$, $\{\tgrad_i\}_{i=1}^{b} = \{\tgradD_i\}_{i=1}^{b}$,  
\begin{align*}
  \tilde{\MH} &= 2\max_{i \in [b]}\{\MH_i(1+\omega_i)/p_i\}\eqsp,\\
  \tBs &= 2[\sigmas^2 + (\Bs/b)\sum_{i=1}^b(1-p_i+\omega_i)/p_i]\eqsp.
\end{align*}
  
  \begin{assumption}\label{ass:tgrad:sharp}
    There exists a family of probability measure $\{\nu^{(i)}\}_{i \in [b]}$ on a measurable spaces $\{(\tilde{\msx}^{(i)},\tilde{\mathcal{X}}^{(i)})\}_{i \in [b]}$ and a family of measurable functions $\acn{\tgradD_{i}:\Rd\times\msx^{(i)}\to\Rd}_{i\in[b]}$ such that the following conditions hold.
    \begin{enumerate}[wide, labelwidth=!, labelindent=0pt,label=(\roman*),noitemsep,nolistsep]
    \item \label{ass:tgrad:sharp:2} For any $\theta\in\Rd$, we have
      \begin{equation*}
      \sum_{i=1}^{b}\int_{\tilde{\msx}^{(i)}}\tgradD_{i}\parentheseLigne{\theta, x^{(i)}}\nu^{(i)}\parentheseLigne{\rmd x^{(i)}}
      =\nabla\U\parentheseLigne{\theta}\eqsp.
      \end{equation*}
    \item \label{ass:tgrad:sharp:1} There exists $\parentheseLigne{\tilde{\MH},\tBs}\in\R_+^{2}$ such that for any $\theta\in\Rd$, we have
    \begin{equation*}
      \int_{\tilde{\msx}^{(1:b)}}\norm{\sum_{i=1}^{b}\tgradD_{i}\parentheseLigne{\theta,x^{(i)}}-\nabla\U\parentheseLigne{\theta}}^{2}\otimes_{i=1}^b\nu^{(i)}\parentheseLigne{\rmd x^{(i)}}
      \le \tilde{\MH}\ps{\theta-\theta^{\star}}{\nabla\U\parentheseLigne{\theta}-\nabla\U\parentheseLigne{\theta^{\star}}}
      +\tBs\eqsp.
    \end{equation*}
    \end{enumerate}
  \end{assumption}
  %
  Then under \Cref{ass:tgrad:sharp}, consider  $\parentheseLigne{X_{k}^{(1)} , \ldots, X_{k}^{(b)}}_{k\in\N^*}$ an independent sequence distributed according to $\otimes_{i=1}^b\nu^{(i)}$. Define the general recursion
  \begin{equation*}
    \tilde{\theta}_{k+1} = \tilde{\theta}_{k}-\gamma \sum_{i=1}^{b}\tgradD_{i}\parentheseLigne{\tilde{\theta}_{k},X_{k+1}^{(i)}} + \sqrt{2\gamma}Z_{k+1}\eqsp,\qquad k\in\N\eqsp.
  \end{equation*}
  and the corresponding  the Markov kernel given for any $\gamma \in \R_+^*$, $\theta\in\Rd, \msa \in \mcb(\rset^d)$ by
\begin{equation*}\label{eq:def:tQ}
  \tilde{Q}_{\gamma}(\theta, \msa) = \parentheseLigne{4\uppi\gamma}^{-d/2}\int_{\msa \times \tilde{\msx}^{(1:b)}}\exp\parentheseLigne{-\parentheseLigne{4\gamma}^{-1}\normn{\bar{\theta}-\theta+\gamma\sum_{i=1}^{b}\tgradD_{i}\parentheseLigne{\theta,x^{(i)}}}^2} \,\rmd \bar{\theta} \,\rmd \otimes_{i=1}^b\nu^{(i)}(x^{(i)})\eqsp.
\end{equation*}
We refer to this Markov kernel as the generalised {\qlsd} kernel.
In our next section, we establish quantitative bounds between the iterates of this kernel and $\pi$ in $W_2$. We then apply this result to {\qlsd}~and \texttt{QLSD}$^\star$ as particular cases.

\subsection{Quantitative bounds for the generalised {\qlsd}~kernel}

Define
\begin{equation*}
  \label{eq:bar_gamma}
  \bgamma = \bgamma_1 \wedge \bgamma_2\wedge \bgamma_3 \eqsp,\quad  \bgamma_1 = 2/[5\parentheseLigne{\mU+\lip}] \eqsp, \quad \bgamma_2 = (\mtt + \Ltt + \tMtt)^{-1} \eqsp, \quad\bgamma_3 = (10\mtt)^{-1}\eqsp.
\end{equation*}
\begin{theorem}\label{thm:bound:sec_continuous_vs_cv}
  Assume \Cref{ass:potential_U} and \Cref{ass:tgrad:sharp}. Then, for
  any probability measure
  $\mu\in\mathcal{P}_{2}\parentheseLigne{\Rd}$, any step size
  $\gamma\in\ocint{0,\bgamma}$, any $k \in \N$, we have
  \begin{align*}
    \wass^{2}\parentheseLigne{\mu\tilde{Q}_{\gamma}^{k},\pi}
    \le (1-\gamma \mtt/2)^k\wass^{2}\parentheseLigne{\mu,\pi}
    +\gamma \tilde{B}_{\bar{\gamma}} 
        +\gamma^{2}\tilde{A}_{\bgamma}(1-\mtt \gamma/2)^{k-1}k\int_{\Rd}\normn{\theta-\theta^{\star}}^{2}\mu\parentheseLigne{\rmd\theta}
\eqsp,
  \end{align*}
  where $\tilde{Q}_{\gamma}$ is defined in \eqref{eq:def:tQ} and
     \begin{align*}
    \label{eq:def:D_gamma}
      \tilde{B}_{\bar{\gamma}} &= (2d\Ltt^2/\mtt)\parentheseLigne{1/\mU + 5\bar{\gamma}}\br{1+\bar{\gamma}\lip^2/(2\mU) + \bar{\gamma}^{2}\lip^{2}/12} + 2\tBs/\mtt + 2\lip\tMH\parentheseLigne{2d+\bar{\gamma}\tBs}/\mtt^2 \\
      \tilde{A}_{\bgamma} &= \lip\tMH\eqsp.
    \end{align*}
\end{theorem}
  Let $\xi \in \Pens_2(\rset^{2d})$ be a probability measure on $(\rset^{2d},\mathcal{B}(\rset^{2d}))$ with marginals $\xi_1$ and $\xi_2$, \ie~$\xi(\msa \times \rset^d) = \xi_1(\msa)$ and $\xi(\msa \times \rset^d) = \xi_2(\msa)$ for any $\msa \in \mathcal{B}(\rset^{d})$. Note that under \Cref{ass:potential_U}, the Langevin diffusion defines a Markov semigroup $\sequencet{P}[t][0]$ satisfying $\pi P_t =\pi$ for any $t \geq 0$, see \emph{e.g.} \citet[Theorem 2.1]{Roberts1996}. We introduce a synchronous coupling $(\vartheta_{k\gamma},\theta_k)$ between $\xi_1 P_{k\gamma}$ and $\xi_2\tilde{Q}_{\gamma}^{k}$ for any $k\in\N$ based on a $d$-dimensional  standard Brownian motion $\sequencet{B}[t][0]$ and a couple of random variables $(\theta_0,\vartheta_0)$ with distribution $\xi$ independent of $\sequencet{B}[t][0]$. Consider $\parentheseLigne{\vartheta_{t}}_{t \geq 0}$ the strong solution of the Langevin stochastic differential equation (SDE)
\begin{equation}\label{eq:def:sde_prop}
  \rmd \vartheta_{t}=-\nabla\U\parentheseLigne{\vartheta_{t}}\rmd t+\sqrt{2}\,\rmd B_{t} \eqsp,
\end{equation}
starting from $\vartheta_{0}$. Note that under \Cref{ass:potential_U}-\ref{ass:potential_U:1}, this SDE admits a unique strong solution \citep[Theorem (2.1) in Chapter IX]{revuz2013continuous}.
In addition, define $\parentheseLigne{\theta_{k}}_{k\in\N}$ starting from $\theta_{0}$ and satisfying the recursion: for $k\ge 0$,
\begin{equation}\label{eq:def:theta_continuous_pp_cv}
\theta_{k+1}=\theta_{k}-\gamma\sum_{i=1}^{b}\tgradD_{i}\parentheseLigne{\theta_{k},x_{k+1}^{(i)}}+\sqrt{2}\parentheseLigne{B_{\gamma (k+1)}-B_{\gamma k}}\eqsp,
\end{equation}
where $\parentheseLigne{x_{j}^{(1)},\ldots,x_{j}^{(b)}}_{j\in\N^*}$ is an independent sequence of random variables with distribution $\otimes_{i=1}^b\nu^{(i)}$. Then, by definition,  $(\vartheta_{k\gamma},\theta_k)$ is a coupling between  $\xi_1 P_{k\gamma}$ and $\xi_2\tilde{Q}_{\gamma}^{k}$ for any $k \in\nset$ and therefore
\begin{equation}\label{eq:bound:W2_cont_pp_cv}
    \wass\parentheseLigne{\xi_1 P_{k\gamma} , \xi_2\tilde{Q}_{\gamma}^{k}} \le \E\br{\|\vartheta_{\gamma k}-\theta_{k}\|^{2}}^{\half}\eqsp.
\end{equation}
We can now give the proof of \Cref{thm:bound:sec_continuous_vs_cv}.
 
\begin{proof}
  By \citet[Theorem 4.1]{Villani2008}, for any couple of probability
  measures on $\Rd$, there exists an optimal transference plan
  $\xi^{\star}$ between $\nu$ and $\pi$ since $\pi\in\Pens_2(\rset^d)$ by the strong convexity assumption \Cref{ass:potential_U}-\ref{ass:potential_U:1}. Let $(\vartheta_0,\theta_0)$ be a corresponding
  coupling which therefore satisfies
  $\wass\parentheseLigne{\mu,\pi}=\PE^{1/2}[\norm{\vartheta_0-\theta_0}^2]$. Consider then
  $\parentheseLigne{\vartheta_{k}}_{k\in\N},\parentheseLigne{\theta_{k}}_{k\in\N}$ defined in
  \eqref{eq:def:sde_prop}-\eqref{eq:def:theta_continuous_pp_cv}
  starting from $(\vartheta_0,\theta_0)$. Note that since $\pi P_t = \pi$ by \citet[Theorem 2.1]{Roberts1996} for any $t \geq 0$ and $\theta_0$ has distribution $\pi$, we get by \citet[Proposition 1]{durmus2018high} that for any $k \in\nset$, $\PE[\normLigne{\vartheta_{k\gamma} - \thetas}^2] \le d/\mtt$ and then
  \Cref{lem:bound:sec_continuous_vs_cv} below shows that for any $k \in\nset$,
    \begin{equation*}
    \E[\norm{\vartheta_{(k+1)\gamma}-\theta_{k+1}}^2]
    \le \kappa_{\gamma}    \E[\norm{\vartheta_{k\gamma}-\theta_{k}}^2]
    +\gamma^{2}\lip\tMH\, \expe{\normn{\theta_0-\theta^{\star}}^{2}}\tkappa_{\gamma}^{k}
    + \gamma^2 \Dg\eqsp,
  \end{equation*}
  where we have set
  \begin{align*}
    \kappa_\gamma = 1-\gamma \mU\parentheseLigne{1-5\gamma\mU}\eqsp,&
    &\tkappa_{\gamma} = 1-\gamma \mtt\br{2-\gamma\parentheseLigne{\mU+\tilde{\MH}}} \eqsp,&
    &\Dg = \DgZ + \prn{1/\mtt + 5\gamma}\prn{\nofrac{\gamma d \Ltt^4}{2\mtt}} \eqsp.
  \end{align*}
A straightforward induction shows that
  \begin{equation*}
    \E[\norm{\vartheta_{k\gamma}-\theta_{k}}^2]
    \le \kappa_{\gamma}^{k}\wass^2\parentheseLigne{\mu,\pi}
    +\gamma^{2}\lip\tMH\, \expe{\normn{\theta_0-\theta^{\star}}^{2}}\sum_{l=0}^{k-1}\kappa_{\gamma}^l\tkappa_{\gamma}^{k-1-l}
    + \gamma^2 \Dg/(1-\kappa_{\gamma})\eqsp.
  \end{equation*}
  Using $\kappa_{\gamma} \wedge \tkappa_{\gamma} \le 1 - \mtt \gamma/2$ since $\gamma \le \bgamma$,  \eqref{eq:bound:W2_cont_pp_cv}  and $\pi P_t = \pi$ for any $t \geq 0$ completes the proof.
\end{proof}

\subsubsection{Supporting Lemmata}\label{subsection:lemmas:general_qlsd}

In this subsection, we derived two lemmas. Taking $\prn{\theta_k}_{k\in\N}$ defined by the recursion \eqref{eq:def:theta_continuous_pp_cv}, \Cref{lem:bound:thetak_minustheta_star} aims to upper bound the squared deviation between $\theta_{k}$ and the minimiser of $\U$ denoted $\theta^\star$, for any $k\in\N$.

\begin{lemma}\label{lem:bound:thetak_minustheta_star}
    Assume \Cref{ass:potential_U} and \Cref{ass:tgrad:sharp}.
    Let $\txts\gamma\in\ocint{0,2/\parentheseLigne{\mU+\Ltt+\tilde{\MH}}}$. Then, for any $k \in \N,\theta_0\in\Rd$, we have
    \begin{equation*}
      \int_{\Rd}\norm{\theta-\theta^{\star}}^{2} \tilde{Q}_{\gamma}^k\parentheseLigne{\theta_0,\,\rmd\theta}
      \le\parentheseLigne{1-\gamma \mtt\br{2-\gamma\parentheseLigne{\mU+\tilde{\MH}}}}^{k}\normn{\theta_0-\theta^{\star}}^{2}
      +\frac{2d+\gamma\tBs}{\mtt\br{2-\gamma\parentheseLigne{\mU+\tilde{\MH}}}}\eqsp,
    \end{equation*}
    where $\tilde{Q}_{\gamma}$ is defined in \eqref{eq:def:tQ}.
\end{lemma}
\begin{proof}
For any $\theta_0\in\Rd$, by definition \eqref{eq:def:tQ} of $\tilde{Q}_{\gamma}$ and using \Cref{ass:tgrad:sharp}-\ref{ass:tgrad:sharp:2}, we obtain
\begin{multline}\label{eq:eq:contraction_proof_D}
  \int_{\Rd}\norm{\theta-\theta^{\star}}^{2} \tilde{Q}_{\gamma}\parentheseLigne{\theta_0,\,\rmd\theta}
=\normn{\theta_0-\theta^{\star}}^{2}
  -2\gamma\ps{\theta_0-\theta^{\star}}{\nabla\U\parentheseLigne{\theta_0}} \\
  +\gamma^{2}\int_{\tilde{\msx}^{(1:b)}}\norm{\textstyle\sum_{i=1}^{b}\tgradD_{i}\parentheseLigne{\theta_0,x^{(i)}}}^{2} \otimes_{i=1}^b\nu^{(i)}\parentheseLigne{\rmd x^{(i)}}
  +2\gamma d\eqsp.
\end{multline}
Moreover, using \Cref{ass:potential_U}, \Cref{ass:tgrad:sharp} and \eqref{eq:intro:Nesterov}, it follows that
\begin{align}
\int_{\tilde{\msx}^{(1:b)}}\norm{\textstyle\sum_{i=1}^{b}\tgradD_{i}\parentheseLigne{\theta_0,x^{(i)}}}^{2}\otimes_{i=1}^b\nu^{(i)}\parentheseLigne{\rmd x^{(i)}}
\nonumber
&=  \int_{\tilde{\msx}^{(1:b)}}\norm{\textstyle\sum_{i=1}^{b}\tgradD_{i}\parentheseLigne{\theta_0,x^{(i)}}- \nabla U(\theta_0)}^2\otimes_{i=1}^b\nu^{(i)}\parentheseLigne{\rmd x^{(i)}}\\
\nonumber
&+\norm{\nabla U(\theta_0)}^2 \\
\nonumber
&\le  \tilde{\MH} \ps{\theta_0-\theta^{\star}}{\nabla\U(\theta_0)} +  \tBs +   \norm{\nabla U(\theta_0) - \nabla U(\thetas)}^2  \\
&\le  [\mtt + \Ltt + \tilde{\MH}] \ps{\theta_0-\theta^{\star}}{\nabla\U(\theta_0)} +  \tBs  - \Ltt \mtt \norm{\theta_0-\thetas}^2 \eqsp.
  \label{eq:bound:contraction_proof:tgrad_D}
\end{align}
Plugging \eqref{eq:bound:contraction_proof:tgrad_D} in \eqref{eq:eq:contraction_proof_D} implies
\begin{multline*}
\int_{\Rd}\norm{\theta-\theta^{\star}}^{2} \tilde{Q}_{\gamma}(\theta_0,\rmd \theta)
\le (1-\gamma^2 \mtt \Ltt)\normn{\theta_0-\theta^{\star}}^{2}
  -\gamma\{2-\gamma [\mtt + \Ltt + \tilde{\MH}]\} \ps{\theta_0-\theta^{\star}}{\nabla\U\parentheseLigne{\theta_0}} 
  +\gamma^{2}  \tBs  + 2 \gamma d \eqsp.
\end{multline*}
Using \Cref{ass:potential_U}-\ref{ass:potential_U:1}, we have $\psn{\theta_0-\theta^{\star}}{\nabla\U\parentheseLigne{\theta_0}}\ge \mtt \normn{\theta_0-\theta^{\star}}^{2}$ which, combined with the condition $\gamma \le 1/(\mtt +\Ltt + \tMH)$, gives
\begin{equation*}
\int_{\Rd}\norm{\theta-\theta^{\star}}^{2}\tilde{Q}_{\gamma}(\theta_0,\rmd\theta)
\le\parentheseLigne{1-\gamma \mtt\brn{2-\gamma\parentheseLigne{\mU+\tilde{\MH}}}}\normn{\theta_0-\theta^{\star}}^{2}
+\gamma\parentheseLigne{2d+\gamma\tBs}\eqsp.
\end{equation*}
Using $0<\gamma<2/\parentheseLigne{\mU+\tilde{\MH}}$ and the Markov property combined with a straightforward induction completes the proof.
\end{proof}
For any $k\in\N$, the following lemma gives an explicit upper bound on the expected squared norm between $\vartheta_{k+1}$ and $\theta_{k+1}$ in function of $\vartheta_{k}, \theta_{k}$. 
The purpose of this lemma is to derive a contraction property involving a contracting term and a bias term which is easy to control. 
\begin{lemma}\label{lem:bound:sec_continuous_vs_cv}
  Assume \Cref{ass:potential_U} and \Cref{ass:tgrad:sharp}. Consider $\sequencet{\vartheta}[t][0]$ and $(\theta_k)_{k\in\nset}$ defined in \eqref{eq:def:sde_prop} and \eqref{eq:def:theta_continuous_pp_cv}, respectively, for some initial distribution $\xi \in \Pens_2(\rset^{2d})$. For any $k \in \N$ and  $\gamma \in \ooint{0, 2/\brn{(5\parentheseLigne{\mU+\lip}) \vee \parentheseLigne{\mU+\tilde{\MH}+\Ltt}}}$, we have 
  \begin{align*}
    \E\br{\norm{\vartheta_{\gamma (k+1)}-\theta_{k+1}}^2}
    &\le \{1-\gamma \mU\parentheseLigne{1-5\gamma\mU}\}\E\br{\normn{\vartheta_{k}-\theta_k}^2}   +\gamma^2 \DgZ \\ & \qquad +\gamma^{2}\lip\tilde{\MH}( 1-\gamma \mtt\br{2-\gamma\parentheseLigne{\mU+\tilde{\MH}}})^{k}\PE[\normn{\theta_0-\theta^{\star}}^{2}]
\\ & \qquad+ \gamma^3\parentheseLigne{1/\mU+5\gamma}\lip^{4}\expe{\normLigne{\vartheta_{k\gamma}-\thetas}^2}/2 \eqsp,
  \end{align*}
  where
  \begin{equation*}
    \DgZ = d\Ltt^2\parentheseLigne{1/\mU+5\gamma}\br{1+\gamma^{2}\lip^{2}/12}
    +\tBs+\frac{\lip\tMH\parentheseLigne{2d+\gamma\tBs}}{\mU\br{2-\gamma\parentheseLigne{\mU+\tMH}}}\eqsp.
  \end{equation*}
\end{lemma}
\begin{proof}
Let $k \in\nset$.
By \eqref{eq:def:sde_prop} and \eqref{eq:def:theta_continuous_pp_cv}, we have
\begin{multline*}
\vartheta_{\gamma(k+1)}-\theta_{k+1}=\vartheta_{\gamma k}-\theta_{k}-\gamma\br{\nabla\U\parentheseLigne{\vartheta_{\gamma k}}-\nabla\U\parentheseLigne{\theta_{k}}}\\
-\int_{0}^{\gamma}\br{\nabla\U\parentheseLigne{\vartheta_{\gamma k+s}}-\nabla\U\parentheseLigne{\vartheta_{\gamma k}}}\,\rmd s
+\gamma\sum_{i=1}^{b}\br{\tgradD_{i}\parentheseLigne{\theta_{k},X_{k+1}^{(i)}}-\nabla\U_{i}\parentheseLigne{\theta_{k}}}\eqsp.
\end{multline*}
Define the filtration $\parentheseLigne{\mathcal{F}_{\tilde{k}}}_{\tilde{k} \in \mathbb{N}}$ as $\mathcal{F}_{0}=\sigma\parentheseLigne{\vartheta_{0},\theta_{0}}$ and for $\tilde{k} \in \mathbb{N}^{*}$,
\begin{align*}
  \mathcal{F}_{\tilde{k}} &= \sigma\parentheseLigne{\vartheta_{0},\theta_{0},\parentheseLigne{X_{l}^{(1)} , \ldots, X_{l}^{(b)}}_{1\le l\le \tilde{k}},\parentheseLigne{B_{t}}_{0\le t\le \gamma \tilde{k}}}\eqsp.
\end{align*}

Note that since $\sequencet{\vartheta}[t][0]$ is a strong solution of \eqref{eq:def:sde_prop}, then is easy to see that $(\vartheta_{\gamma \tilde{k}}, \theta_{\tilde{k}})_{\tilde{k}\in\nset}$ is $(\mcf_{\tilde{k}})_{\tilde{k}\in\nset}$-adapted. 
Taking the squared norm and the conditional expectation with respect to $\mathcal{F}_{k}$, we obtain using \Cref{ass:tgrad:sharp}-\ref{ass:tgrad:sharp:2} that
\begin{align}
  &\nonumber\E^{\mathcal{F}_{k}}\br{\norm{\vartheta_{\gamma(k+1)}-\theta_{k+1}}^{2}}
    =\norm{\vartheta_{\gamma k}-\theta_{k}}^{2}
    -2\gamma\ps{\vartheta_{\gamma k}-\theta_{k}}{\nabla\U\parentheseLigne{\vartheta_{\gamma k}}-\nabla\U\parentheseLigne{\theta_{k}}}\\
    &\qquad\qquad\qquad\nonumber+2\gamma\int_{0}^{\gamma}\ps{\nabla\U\parentheseLigne{\vartheta_{\gamma k}}-\nabla\U\parentheseLigne{\theta_{k}}}{\E^{\mathcal{F}_{k}}\br{\nabla\U\parentheseLigne{\vartheta_{\gamma k+s}}-\nabla\U\parentheseLigne{\vartheta_{\gamma k}}}}\,\rmd s\\
    &\qquad\qquad\qquad\nonumber-2\int_{0}^{\gamma}\ps{\vartheta_{\gamma k}-\theta_{k}}{\E^{\mathcal{F}_{k}}\br{\nabla\U\parentheseLigne{\vartheta_{\gamma k+s}}-\nabla\U\parentheseLigne{\vartheta_{\gamma k}}}}\,\rmd s\\
    &\qquad\qquad\qquad\nonumber+\gamma^{2}\norm{\nabla\U\parentheseLigne{\vartheta_{\gamma k}}-\nabla\U\parentheseLigne{\theta_{k}}}^{2}\\
    &\qquad\qquad\qquad\nonumber+\E^{\mathcal{F}_{k}}\br{\norm{\int_{0}^{\gamma}\br{\nabla\U\parentheseLigne{\vartheta_{\gamma k+s}}-\nabla\U\parentheseLigne{\vartheta_{\gamma k}}}\,\rmd s}^{2}}\\
    &\qquad\qquad\qquad+\gamma^{2}\E^{\mathcal{F}_{k}}\br{\norm{\sum_{i=1}^{b}\tgradD_{i}\parentheseLigne{\theta_{k},X_{k+1}^{(i)}}-\nabla\U\parentheseLigne{\theta_{k}}}^{2}}\eqsp.\label{eq:eq:diff_cont_discret_pp_cv}
\end{align}
First, using Jensen inequality and the fact that for any $a,b\in\Rd$, $\abs{\ps{a}{b}}\le 2\norm{a}^{2}+2\norm{b}^{2}$, we get
\begin{align}
  \nonumber
  &\int_{0}^{\gamma}\ps{\nabla\U\parentheseLigne{\vartheta_{\gamma k}}-\nabla\U\parentheseLigne{\theta_{k}}}{\E^{\mathcal{F}_{k}}\br{\nabla\U\parentheseLigne{\vartheta_{\gamma k+s}}-\nabla\U\parentheseLigne{\vartheta_{\gamma k}}}}\,\rmd s \\
  &\qquad \le 2\gamma\norm{\nabla\U\parentheseLigne{\vartheta_{\gamma k}}-\nabla\U\parentheseLigne{\theta_{k}}}^{2}
  +2\int_{0}^{\gamma} \E^{\mathcal{F}_{k}}\br{\norm{\nabla\U\parentheseLigne{\vartheta_{\gamma k+s}}-\nabla\U\parentheseLigne{\vartheta_{\gamma k}}}^{2}}\,\rmd s\eqsp, \label{eq:bound:ps_un_pp_cv} \\
  \nonumber
  &\E^{\mathcal{F}_{k}}\br{\norm{\int_{0}^{\gamma}\br{\nabla\U\parentheseLigne{\vartheta_{\gamma k+s}}-\nabla\U\parentheseLigne{\vartheta_{\gamma k}}}\,\rmd s}^{2}}
  \le\gamma\int_{0}^{\gamma}\E^{\mathcal{F}_{k}}\br{\norm{\nabla\U\parentheseLigne{\vartheta_{\gamma k+s}}-\nabla\U\parentheseLigne{\vartheta_{\gamma k}}}^{2}}\,\rmd s\eqsp.
\end{align}
In addition, given that for any $\varepsilon>0,a,b\in\Rd$, $\lvert\ps{a}{b}\rvert\le \varepsilon\norm{a}^{2}+\parentheseLigne{4\varepsilon}^{-1}\norm{b}^{2}$, we get
\begin{multline}\label{eq:bound:ps_deux_pp_cv}
\abs{\int_{0}^{\gamma}\ps{\theta_{k}-\vartheta_{\gamma k}}{\E^{\mathcal{F}_{k}}\br{\nabla\U\parentheseLigne{\vartheta_{\gamma k+s}}-\nabla\U\parentheseLigne{\vartheta_{\gamma k}}}}\,\rmd s}
\le \gamma\varepsilon\norm{\vartheta_{\gamma k}-\theta_{k}}^{2}\\
+\parentheseLigne{4\varepsilon}^{-1}\int_{0}^{\gamma}\E^{\mathcal{F}_{k}}\br{\norm{\nabla\U\parentheseLigne{\vartheta_{\gamma k+s}}-\nabla\U\parentheseLigne{\vartheta_{\gamma k}}}^{2}}\,\rmd s\eqsp.
\end{multline}
By \Cref{ass:potential_U}, for $k\in\N$ we get by \eqref{eq:intro:Nesterov}
\begin{equation}\label{eq:bound:ps_{t}rois_pp_cv}
\norm{\nabla\U\parentheseLigne{\vartheta_{\gamma k}}-\nabla\U\parentheseLigne{\theta_{k}}}^{2}
\le \parentheseLigne{\mU+\lip}\ps{\vartheta_{\gamma k}-\theta_{k}}{\nabla\U\parentheseLigne{\vartheta_{\gamma k}}-\nabla\U\parentheseLigne{\theta_{k}}}-\mU\lip\norm{\vartheta_{\gamma k}-\theta_{k}}^2\eqsp.
\end{equation}
Lastly, \Cref{ass:tgrad:sharp}-\ref{ass:tgrad:sharp:1} yields
\begin{equation}\label{eq:bound:tgrad}
    \E^{\mcf_k}\br{\norm{\sum_{i=1}^{b}\tgradD_{i}\parentheseLigne{\theta_{k},X_{k+1}^{(i)}}-\nabla\U\parentheseLigne{\theta_{k}}}^{2}}
    \le \tilde{\MH}\ps{\theta_{k}-\theta^{\star}}{\nabla\U(\theta_{k})-\nabla\U(\theta^{\star})}
    +\tBs\eqsp.
\end{equation}
Combining \eqref{eq:bound:ps_un_pp_cv}, \eqref{eq:bound:ps_deux_pp_cv}, \eqref{eq:bound:ps_{t}rois_pp_cv} and \eqref{eq:bound:tgrad} into \eqref{eq:eq:diff_cont_discret_pp_cv}, for $k\in\N$ we get for any $\varepsilon >0$,
\begin{align*}
    \E^{\mathcal{F}_{k}}\br{\norm{\vartheta_{\gamma(k+1)}-\theta_{k+1}}^{2}}
    \nonumber
    &\le\parentheseLigne{1+2\gamma\varepsilon-5\gamma^2\mU\lip}\norm{\vartheta_{\gamma k}-\theta_{k}}^{2}\\
    \nonumber
    &-\gamma\br{2-5\gamma\parentheseLigne{\mU+\lip}}\ps{\vartheta_{\gamma k}-\theta_{k}}{\nabla\U\parentheseLigne{\vartheta_{\gamma k}}-\nabla\U\parentheseLigne{\theta_{k}}}\\
    \nonumber
    &+\parentheseLigne{5\gamma+(2\varepsilon)^{-1}}\int_{0}^{\gamma}\E^{\mathcal{F}_{k}}\br{\norm{\nabla\U\parentheseLigne{\vartheta_{\gamma k+s}}-\nabla\U\parentheseLigne{\vartheta_{\gamma k}}}^{2}}\,\rmd s\\
    &+\gamma^2\tilde{\MH}\ps{\theta_{k}-\theta^{\star}}{\nabla\U(\theta_{k})-\nabla\U(\theta^{\star})}
    +\gamma^{2}\tBs\eqsp.
\end{align*}
Next, we use that under \Cref{ass:potential_U},  $\psLigne{\vartheta_{\gamma k}-\theta_{k}}{\nabla\U\parentheseLigne{\vartheta_{\gamma k}}-\nabla\U\parentheseLigne{\theta_{k}}}\ge \mU \normLigne{\vartheta_{\gamma k}-\theta_{k}}^{2}$ and $\absLigne{\psLigne{\theta_{k}-\theta^{\star}}{\nabla\U(\theta_{k})-\nabla\U(\theta^{\star})}}\le \lip\normLigne{\theta_{k}-\theta^{\star}}^2$,
which implies taking $\varepsilon = \mtt /2$ and since $2 - 5\gamma(\mtt + \Ltt) \geq 0$,
\begin{align}
    \E^{\mathcal{F}_{k}}\br{\norm{\vartheta_{\gamma(k+1)}-\theta_{k+1}}^{2}}
    \nonumber
    &\le\parentheseLigne{1-\gamma\mU(1-5\gamma\mU)}\norm{\vartheta_{\gamma k}-\theta_{k}}^{2}\\
    \nonumber
    &+\parentheseLigne{5\gamma+\mtt^{-1}}\int_{0}^{\gamma}\E^{\mathcal{F}_{k}}\br{\norm{\nabla\U\parentheseLigne{\vartheta_{\gamma k+s}}-\nabla\U\parentheseLigne{\vartheta_{\gamma k}}}^{2}}\,\rmd s\\
    &+\gamma^2\tilde{\MH}\Ltt\norm{\theta_{k}-\theta^{\star}}^2
    +\gamma^{2}\tBs\eqsp.
    \label{main:eq:bound:diff_rec_cont_discrete_pp_cv}
\end{align}

Further, for any $s\in\R_+$, using \citet[Lemma 21]{durmus2018high} we have 
\begin{equation*}
\lip^{-2}\,\E^{\mathcal{F}_{k}}\br{\norm{\nabla\U\parentheseLigne{\vartheta_{\gamma k+s}}-\nabla\U\parentheseLigne{\vartheta_{\gamma k}}}^{2}}
\le d s\pr{2+s^{2}\lip^{2}/3}+3s^{2}\lip^{2}/2\norm{\vartheta_{\gamma k}-\theta^{\star}}^{2}\eqsp.
\end{equation*}
Integrating the previous inequality on $\br{0,\gamma}$, for $k\ge 0$ we obtain
\begin{equation*}
\lip^{-2}\int_{0}^{\gamma}\E^{\mathcal{F}_{k}}\br{\norm{\nabla\U\parentheseLigne{\vartheta_{\gamma k+s}}-\nabla\U\parentheseLigne{\vartheta_{\gamma k}}}^{2}}\,\rmd s
\le d\gamma^2+d\gamma^4\lip^{2}/12+\gamma^{3}\lip^{2}/2\norm{\vartheta_{\gamma k}-\theta^{\star}}^{2}\eqsp.
\end{equation*}
Plugging this bounds in \eqref{main:eq:bound:diff_rec_cont_discrete_pp_cv} and taking the expectation combined with \Cref{lem:bound:thetak_minustheta_star} conclude the proof.
\end{proof}


\subsection{Proof of \Cref{thm:QLSD}}

Based on \Cref{thm:bound:sec_continuous_vs_cv}, the next corollary explicits an upper bound in Wasserstein distance between $\pi$ and $\mu Q_{\gamma}^k$, where we consider $\prn{\theta_{k}}_{k\in\N}$ defined in \eqref{eq:def:recursion:theta} and starting from $\btheta$ following $\mu\in\mathcal{P}_2\prn{\Rd}$.

\begin{theorem}\label{cor:bound:sec_continuous_vs_cv}
  Assume \Cref{ass:potential_U}, \Cref{ass:compression}, \Cref{ass:stochastic_gradient} and \Cref{ass:A_k_supp}. 
  Then, for any probability measure
  $\mu\in\mathcal{P}_{2}\parentheseLigne{\Rd}$, any step size $\gamma\in\ocint{0,\bgamma}$ where $\bar{\gamma}$ is defined in \eqref{eq:bar_gamma}, any $k \in \N$, we have
  \begin{align*}
    \wass^{2}\parentheseLigne{\mu Q_{\gamma}^{k},\pi}
    \le (1-\gamma \mtt/2)^k\wass^{2}\parentheseLigne{\mu,\pi}
    +\gamma B_{\bar{\gamma}}
        +\gamma^{2}A_{\bgamma}(1-\mtt \gamma/2)^{k-1}k\int_{\Rd}\normn{\theta-\theta^{\star}}^{2}\mu\parentheseLigne{\rmd\theta}\eqsp,
  \end{align*}
  where $Q_{\gamma}$ is defined in \eqref{eq:def:intro:Q} and
  \begin{align}
      \nonumber
      B_{\bar{\gamma}} &= (2d\lip^2/\mtt)\pr{1/\mU+5\bar{\gamma}}\br{1+\bar{\gamma} \lip^2/(2\mU)+\bar{\gamma}^{2} \lip^2/12}+4[\sigmas^2 + (\Bs/b)\sum_{i=1}^b(1-p_i+\omega_i)/p_i]/\mtt \\
      &+ 8\lip\max_{i \in [b]}\{\MH_i(1+\omega_i)/p_i\})\br{d+\bar{\gamma}[\sigmas^2 + (\Bs/b)\sum_{i=1}^b(1-p_i+\omega_i)/p_i]}/\mU^2\label{eq:def:E_gamma} \\
      \nonumber
      A_{\bgamma} &= 2\lip\max_{i \in [b]}\{\MH_i(1+\omega_i)/p_i\})\eqsp.
  \end{align}
\end{theorem}
\begin{proof}
  By \Cref{lem:intro:1}, the assumption \Cref{ass:tgrad:sharp} is satified for a choice of $\tilde{\MH}=2\max_{i \in [b]}\{\MH_i(1+\omega_i)/p_i\})$ and $\tBs=2[\sigmas^2 + (\Bs/b)\sum_{i=1}^b(1-p_i+\omega_i)/p_i]$. Therefore, applying \Cref{thm:bound:sec_continuous_vs_cv} completes the proof.
\end{proof}

\section{PROOF OF \Cref{thm:QLSD_VR}}\label{sec:cv_qlsdstar}


We assume here that $\{U_i\}_{i \in [b]}$ are defined, for any $i\in[b]$ and $\theta\in\Rd$, by
\begin{equation*}\label{eq:def:Uij}
	U_{i}\prn{\theta}=\sum_{j=1}^{N_i}U_{i,j}\prn{\theta}\eqsp, \quad N_i \in \N^*\eqsp.
\end{equation*}
      We consider the following set of assumptions on $\{U_i\}_{i \in [b]}$ and $\{U_{i,j} : j \in [N_i]\}_{i \in [b]}$.
      \begin{assumption}
        \label{ass:potential_fij}
         For any $i\in[b],j\in [N_i]$, $U_{i,j}$ is continuously differentiable and the following conditions hold. 
    \begin{enumerate}[wide, labelwidth=!, labelindent=0pt,label=(\roman*),noitemsep,nolistsep]
    \item 
    There exist $\{\MH_i > 0\}_{i \in [b]}$, such that for any $i \in [b]$, $\theta_{1},\theta_{2} \in \Rd$,
      \begin{equation*}      
        \norm{\nabla U_i(\theta_2) - \nabla U_i(\theta_2)}^2 \le  \MH_i \ps{\theta_{2}-\theta_{1}}{\nabla\U_{i}\prn{\theta_{2}}-\nabla\U_{i}\prn{\theta_{1}}}\eqsp.
      \end{equation*}
    \item   
There exists $\bMH\ge 0$ such that, for any $\theta_{1},\theta_{2}\in\Rd$,
  	\begin{equation*}
  		\norm{\nabla\U_{i,j}\prn{\theta_{2}}-\nabla\U_{i,j}\prn{\theta_{1}}}^2
  		\le \bMH \ps{\nabla\U_{i,j}\prn{\theta_{2}}-\nabla\U_{i,j}\prn{\theta_{1}}}{\theta_{2}-\theta_{1}}\eqsp.
              \end{equation*}
            \end{enumerate}
          \end{assumption}

In all this section, we assume for any $i \in [b]$ that $n_i \in \nsets$, $n_i \le N_i$ is fixed. 
For any $i \in [b]$, recall that  $\wp_{N_i}$ denotes the power set of $[N_i]$ and
\begin{equation*}
  \label{eq:2}
  \wp_{N_i,n_i} = \{ x \in \wp_{N_i} \, :\,  \card(x) = n_i\} \eqsp. 
\end{equation*}
We set in this section $\nu_1^{(i)}$ as the uniform distribution on $\wp_{N_i,n_i}$.
We consider the family of measurable functions $\acn{H^\star_i:\Rd \times \Rd\times \wp_{N_i}\to\Rd}_{i\in[b]}$, defined for any $i \in [b]$, $\theta\in\Rd$,  $x \in \wp_{N_i,n_i}$ by
\begin{equation}
	\label{eq:definition_Fi_star}	
    H^\star_i\prn{\theta,x}=\frac{N_i}{n_i}\sum_{j=1}^{N_i}\mathbf{1}_{x}(j)\br{\nabla\U_{i,j}(\theta)-\nabla\U_{i,j}(\thetas)}\eqsp.
    \end{equation}

Using this specific family of gradient estimators boils down to the \texttt{QLSD}$^\star$ algorithm detailed in \Cref{algo:QLSD-star}.

\begin{algorithm}[h]
   \caption{Variance-reduced Quantised Langevin Stochastic Dynamics (\texttt{QLSD}$^{\star}$)}
   \label{algo:QLSD-star}
  \begin{algorithmic}
     \State {\bfseries Input:} minibatch sizes $\{n_i\}_{i \in [b]}$, number of iterations $K$, compression operators $\{\mathscr{C}_{k+1}\}_{k\in\mathbb{N}^*}$, step-size $\gamma \in (0,\bar{\gamma}]$ with $\bar{\gamma} > 0$ and initial point $\theta_0$.
     \For{$k=0$ {\bfseries to} $K-1$}
     \For{$i \in \mathcal{A}_{k+1}$ \Comment{On active clients}}
        \State Draw $\mathcal{S}_{k+1}^{(i)} \sim \mathrm{Uniform}\pr{\wp_{N_i,n_i}}$.
        \State {\small Set $H_{k+1}^{(i)}(\theta_k) = (N_i/n_i)\sum_{j \in \mathcal{S}_{k+1}^{(i)}} [\nabla U_{i,j}(\theta_k) - \nabla U_{i,j}(\thetas)]$.}
        \State Compute $\textsl{g}_{i,k+1} = \mathscr{C}_{k+1}\pr{H_{k+1}^{(i)}(\theta_k)}$.
        \State Send $\textsl{g}_{i,k+1}$ to the central server.
     \EndFor
     \State \Comment{On the central server}
     \State Compute $\textsl{g}_{k+1} = \frac{b}{|\mathcal{A}_{k+1}|}\sum_{i \in \mathcal{A}_{k+1}}\textsl{g}_{i,k+1}$.
     \State Draw $Z_{k+1} \sim \mathrm{N}(0_d,\mathrm{I}_d)$.
     \State Compute $\theta_{k+1} = \theta_k - \gamma \textsl{g}_{k+1} + \sqrt{2\gamma}Z_{k+1}$.
     \State Send $\theta_{k+1}$ to the $b$ clients.
     \EndFor
     \State {\bfseries Output:} samples $\{\theta_k\}_{k=0}^{K}$.
  \end{algorithmic}
\end{algorithm}

    Let $(X_k^{(1,1)},\ldots,X_k^{(1,b)})_{k \in\nsets}$ and $(X_k^{(2,1)},\ldots,X_k^{(2,b)})_{k \in\nsets}$  be  two independent \iid~sequences with distribution $\otimes_{i=1}^b\nu_1^{(i)}$ and $\nu_2^{\otimes b}$. Let $(Z_k)_{k\in\nsets}$ be an \iid~sequence of $d$-dimensional standard Gaussian random variables independent of $(X_k^{(1,1)},\ldots,X_k^{(1,b)})_{k \in\nsets}$ and $(X_k^{(2,1)},\ldots,X_k^{(2,b)})_{k \in\nsets}$. 
    Similarly as before, we consider the partial device participation context where at each communication round $k\ge1$, each client has a probability $p_i \in (0,1]$ of participating, independently from other clients.
    In other words, there exists a sequence $(X_k^{(3,1)},\cdots,X_k^{(3,b)})_{k\in \mathbb{N}^*}$ of i.i.d.  random variables distributed according $\nu_3 = \mathrm{Uniform}((0,1])$, such that for any $k\ge1$ and $i\in[b]$, client $i$ is active at step $k$ if $X_k^{(3,i)} \le p_i$.
    We denote $\mathcal{A}_{k+1} = \{i \in [b]; X_{k+1}^{(3,i)} \le p_i\}$ the set of active clients at round $k$. 
    For ease of notation, denote for any $k \in\nsets$, $X_{k}^{(1)} = (X_k^{(1,1)},\ldots,X_k^{(1,b)})$, $X_k^{(2)}= (X_k^{(2,1)},\ldots,X_k^{(2,b)})$, $X_k^{(3)}= (X_k^{(3,1)},\ldots,X_k^{(3,b)})$ and $X_k = (X_k^{(1)},X_k^{(2)},X_k^{(3)})$.
    
Note that with this notation and under \Cref{ass:compression}, \texttt{QLSD}$^\star$ can be cast into the framework of the generalised {\qlsd}~scheme defined in \eqref{eq:def:recursion:theta} since the recursion associated to \texttt{QLSD}$^\star$ can be written as
  \begin{equation}\label{eq:def:recursion:tildetheta}
    \tilde{\theta}_{k+1} = \tilde{\theta}_{k}-\gamma \sum_{i=1}^{b}\mathscr{S}_i\br{\mathscr{C}_i \parenthese{H^{\star}_i\prn{\tilde{\theta}_{k},X_{k+1}^{(1,i)}}, X_{k+1}^{(2,i)}},X_{k+1}^{(3,i)}} + \sqrt{2\gamma}Z_{k+1}\eqsp,\qquad k\in\N\eqsp,
  \end{equation}
  where, for any $i \in [b]$,  $\mathscr{S}_i$ is defined in \eqref{def:PP}.
Therefore, we only need to verify that \Cref{ass:tgrad:sharp} is satisfied with $\msx^{(i)} = \tilde{\msx}^{(i)} = \msx_{1}^{(i)} \times \msx_{2}\times \msx_{3}$, $\mathcal{X}^{(i)}=\tilde{\mathcal{X}}^{(i)} = \mathcal{X}_{1}^{(i)} \otimes \mathcal{X}_{2}\otimes \mathcal{X}_{3}$, $\tilde{\nu}^{(i)} = \nu_{1}^{(i)} \otimes \nu_{2} \otimes \nu_{3}$ for $i \in [b]$ and $\{F_i\}_{i=1}^{b} = \{F_i^\star\}_{i=1}^{b} = \{\mathscr{S}_i \circ \mathscr{C}_i \circ H_i^\star\}_{i=1}^{b}$.
This is done in \Cref{subsubsec:QLSDs_lemmata}.

\subsection{Proof of \Cref{thm:QLSD_VR}}

The Markov kernel associated with
\eqref{eq:def:recursion:tildetheta} is given for any
$(\theta,\msa)\in\Rd\times \mathcal{B}(\Rd)$ by
\begin{equation}\label{eq:def:kernel_QLSDs}
 Q_{\ostar,\gamma}(\theta, \msa )
 = \parentheseLigne{4\uppi\gamma}^{-d/2}\int_{\msa \times\tilde{\msx}^{b}} \exp\pr{-\normn{\btheta-\theta+\gamma\sum_{i=1}^{b}F_{i}^\star\parentheseLigne{\theta,x^{(i)}}}^2/(4\gamma)}\rmd \tilde{\theta} \, \otimes_{i=1}^b\tilde{\nu}^{(i)}\parentheseLigne{\rmd x^{\parentheseLigne{i}}}\eqsp.
\end{equation}
Then, the following non-asymptotic convergence result holds for \texttt{QLSD}$^\star$.
\begin{theorem}\label{cor:bound:sec_continuous_vs_fp:1}
  Assume \Cref{ass:potential_U}, \Cref{ass:compression}, \Cref{ass:A_k_supp} and \Cref{ass:potential_fij}.
  Then, for any probability measure
  $\mu\in\mathcal{P}_{2}\parentheseLigne{\Rd}$, any step size $\gamma\in\ocint{0,\bgamma}$ where $\bar{\gamma}$ is defined in \eqref{eq:bar_gamma}, any $k \in \N$, we have
  \begin{align*}
    \wass^{2}\parentheseLigne{\mu Q_{\ostar,\gamma}^{k},\pi}
    \le (1-\gamma \mtt/2)^k\wass^{2}\parentheseLigne{\mu,\pi}
    + \gamma B_{\ostar,\bar{\gamma}} 
        +\gamma^{2}A_{\ostar,\bgamma}(1-\mtt \gamma/2)^{k-1}k\int_{\Rd}\normn{\theta-\theta^{\star}}^{2}\mu\parentheseLigne{\rmd\theta}\eqsp,
  \end{align*}
  where $Q_{\ostar,\gamma}$ is defined in \eqref{eq:def:kernel_QLSDs} and
  \begin{align}
      \label{B_gamma_QLSDs}
      B_{\ostar,\bar{\gamma}} &= (2d\lip^2/\mtt)\pr{1/\mU+5\bar{\gamma}}\br{1+\bar{\gamma} \lip^2/(2\mU)+\bar{\gamma}^{2} \lip^2/12} \\
      \nonumber
      &+ 4 \lip d\bMH\max_{i\in[b]}\bbr{\omega_iN_i + (\omega_i+1)(N_i[1-p_i]/p_i + A_{n_i,N_i})}/\mtt^2 \\
      \nonumber
      A_{\ostar,\bgamma} &= \lip\bMH\max_{i\in[b]}\bbr{\omega_iN_i + (\omega_i+1)(N_i[1-p_i]/p_i + A_{n_i,N_i})}\eqsp,
  \end{align}
  $A_{n_i,N_i}$ being defined in \eqref{eq:defAnN} for any $i \in [b]$.

\end{theorem}
\begin{proof}
Using \Cref{lem:bound:sec_continuous_vs_fp:end}, \Cref{ass:tgrad:sharp} is satisfied and applying \Cref{thm:bound:sec_continuous_vs_cv} completes the proof.
\end{proof}

\subsection{Supporting Lemmata}
\label{subsubsec:QLSDs_lemmata}

In this subsection, we derive two key lemmata in order to prove \Cref{cor:bound:sec_continuous_vs_fp:1}. 

\begin{lemma}\label{lem:var_grad_sto}
  For any $i \in [b]$ and any sequence $\{a_j\}_{j=1}^{N_i}\in\prn{\Rd}^{\otimes N_i}$ where $N_i \ge 2$, we have 
  \begin{equation*}
     \int_{\msx_1^{(i)}}\norm{\sum_{j=1}^{N_i}\br{\mathbf{1}_{x^{(1)}}(j)-\frac{n_i}{N_i}} a_{j}}^2\nu_1^{(i)}\prn{\rmd x^{(1)}}
     \le\frac{n_i(N_i-n_i)}{N_i(N_i-1)}\sum_{j=1}^{N_i} \norm{a_j}^2\eqsp.
  \end{equation*}
\end{lemma}
\begin{proof}
  Let $i \in [b]$ and $X^{(1,i)}$ distributed according to $\nu_1^{(i)}$.
  Since $\sum_{j=1}^{N_i}\mathbf{1}_{X^{(1,i)}}(j) = n_i$, we have 
  \begin{equation*}
    \sum_{l=1}^{N_i}\mathbf{1}_{X^{(1,i)}}(l)+\sum_{j\ne j^\prime}\mathbf{1}_{X^{(1,i)}}(j)\mathbf{1}_{X^{(1,i)}}(j^\prime)=n_i^2\eqsp.
  \end{equation*}
  Integrating this equality over $\msx_1^{(i)}$ gives
  \begin{equation*}
    N_i \times\frac{n_i}{N_i}+N_i(N_i-1)\times\int_{\msx_1^{(i)}}\br{\mathbf{1}_{x^{(1,i)}}(1)\mathbf{1}_{x^{(1,i)}}(2)}\nu_1^{(i)}\prn{\rmd x^{(1,i)}}=n_i^2\eqsp.
  \end{equation*}
  Thus, we deduce that $\int_{\msx_1^{(i)}}\brn{\mathbf{1}_{x^{(1,i)}}(1)\mathbf{1}_{x^{(1,i)}}(2)}\nu_1^{(i)}\prn{\rmd x^{(1,i)}}=n_i(n_i-1)\br{N_i(N_i-1)}^{-1}$.
  In addition, using that
  \begin{equation*}
    \int_{\msx_1^{(i)}}\pr{\mathbf{1}_{x^{(1,i)}}(j)-\frac{n_i}{N_i}}\pr{\mathbf{1}_{x^{(1,i)}}(j^\prime)-\frac{n_i}{N_i}}\nu_1^{(i)}\prn{\rmd x^{(1,i)}}
    =\int_{\msx_1^{(i)}}\brn{\mathbf{1}_{x^{(1,i)}}(1)\mathbf{1}_{x^{(1,i)}}(2)}\nu_1^{(i)}\prn{\rmd x^{(1,i)}}
    -\frac{n_i^2}{N_i^2}\eqsp,
  \end{equation*}
  we obtain
  \begin{align*}
    \int_{\msx_1^{(i)}}{\norm{\sum_{j=1}^{N_i}\br{\mathbf{1}_{x^{(1,i)}}(j)-\frac{n_i}{N_i}} a_{j}}^2}\nu_1^{(i)}\prn{\rmd x^{(1,i)}}
    &=\frac{n_i(N_i-n_i)}{N_i^2}\br{\sum_{l=1}^{N_i} \norm{a_l}^2
    -\sum_{j\neq j^\prime}\frac{\ps{a_j}{a_{j^\prime}}}{N_i-1}}\\
    &=\frac{n_i(N_i-n_i)}{N_i^2(N_i-1)}\br{N_i\sum_{l=1}^{N_i} \norm{a_l}^2
    -\norm{\sum_{l=1}^{N_i} a_l}^2}\eqsp.
  \end{align*}
\end{proof}
For any $i\in [b]$, denote
\begin{equation}
  \label{eq:defAnN}
  A_{n_i,N_i} = \frac{N_i(N_i-n_i)}{n_i(N_i-1)}\eqsp.
\end{equation}
The next lemma aims at controlling the variance of the global stochastic gradient considered in \texttt{QLSD}$^{\star}$, required to apply \Cref{thm:bound:sec_continuous_vs_cv}.
\begin{lemma}\label{lem:bound:sec_continuous_vs_fp:end}
	Assume \Cref{ass:compression}, \Cref{ass:A_k_supp} and \Cref{ass:potential_fij}. 
  Then, for any $\theta\in\Rd$, we have
    \begin{multline*}
      \int_{\msx^{(1:b)}}\norm{\textstyle\sum_{i=1}^{b}\mathscr{S}_i\br{\up_i\pr{H_{i}^\star\prn{\theta,x^{(1,i)}},x^{(2,i)}},x^{(3,i)}}-\nabla\U\prn{\theta}}^{2}\otimes_{i=1}^b\nu^{(i)}\prn{\rmd x^{(i)}}\\
      \le \bMH\max_{i\in[b]}\bbr{\omega_iN_i + (\omega_i+1)(N_i[1-p_i]/p_i + A_{n_i,N_i})}\ps{\theta-\theta^{\star}}{\nabla\U\prn{\theta}-\nabla\U\prn{\theta^{\star}}}\eqsp,
    \end{multline*}
    where $\{H_i^\star\}_{i \in [b]}$ and $\{A_{n_i,N_i}\}_{i \in [b]}$ are defined in \eqref{eq:definition_Fi_star} and \eqref{eq:defAnN}, respectively.
    Hence \Cref{ass:tgrad:sharp} is satified with $\tBs=0$ and $\tMH = \bMH\max_{i\in[b]}\bbr{\omega_iN_i + (\omega_i+1)(N_i[1-p_i]/p_i + A_{n_i,N_i})}$.
\end{lemma}
\begin{proof}
Let $\theta\in\Rd$, using \Cref{ass:compression} gives
\begin{align}
  \nonumber
	&\int_{\msx^{(1:b)}}\norm{\textstyle\sum_{i=1}^{b}\mathscr{S}_i\br{\up_i\pr{H_{i}^\star\prn{\theta,x^{(1,i)}},x^{(2,i)}},x^{(3,i)}}-\nabla\U\prn{\theta}}^{2}\otimes_{i=1}^b\nu^{(i)}\prn{\rmd x^{(i)}}\\
  \nonumber
  &=\int_{\msx^{(1:b)}}\norm{\textstyle\sum_{i=1}^{b}\mathscr{S}_i\br{\up_i\pr{H_{i}^\star\prn{\theta,x^{(1,i)}},x^{(2,i)}},x^{(3,i)}}-\up_i\pr{H_{i}^\star\prn{\theta,x^{(1,i)}},x^{(2,i)}}}^{2}\otimes_{i=1}^b\nu^{(i)}\prn{\rmd x^{(i)}} \\
  \nonumber
  &+\int_{\msx_1^{(1:b)} \times \msx_2^b}\norm{\textstyle\sum_{i=1}^{b}\up_i\pr{H_{i}^\star\prn{\theta,x^{(1,i)}},x^{(2,i)}}-\nabla\U\prn{\theta}}^{2}\nu_2^{\otimes b}\prn{\rmd x^{(2,1:b)}}\otimes_{i=1}^b\nu_1^{(i)}\prn{\rmd x^{(1,i)}} \\
  \nonumber
  &\le\sum_{i=1}^{b} \pr{\frac{1-p_i}{p_i}}(\omega_i + 1) \int_{\msx_1^{(i)}}\norm{\grad_{i}^\star\parentheseLigne{\theta,x^{\parentheseLigne{1,i}}}}^2\nu_1^{(i)}\parentheseLigne{\rmd x^{(1,i)}} \\
  \nonumber
  &+\int_{\msx_1^{(1:b)} \times \msx_2^b}\norm{\textstyle\sum_{i=1}^{b}\up_i\pr{H_{i}^\star\prn{\theta,x^{(1,i)}},x^{(2,i)}}-\nabla\U\prn{\theta}}^{2}\nu_2^{\otimes b}\prn{\rmd x^{(2,1:b)}}\otimes_{i=1}^b\nu_1^{(i)}\prn{\rmd x^{(1,i)}} \\
  \nonumber
  &\le\bar{\MH} \sum_{i=1}^{b} \pr{\frac{1-p_i}{p_i}}(\omega_i + 1)N_i\ps{\theta-\theta^{\star}}{\nabla\U_i\prn{\theta}-\nabla\U_i\prn{\theta^{\star}}} \\
  &+\int_{\msx_1^{(1:b)} \times \msx_2^b}\norm{\textstyle\sum_{i=1}^{b}\up_i\pr{H_{i}^\star\prn{\theta,x^{(1,i)}},x^{(2,i)}}-\nabla\U\prn{\theta}}^{2}\nu_2^{\otimes b}\prn{\rmd x^{(2,1:b)}}\otimes_{i=1}^b\nu_1^{(i)}\prn{\rmd x^{(1,i)}} \eqsp. \label{eq:max5}
\end{align}

Again using \Cref{ass:compression}, it follows that
\begin{align}
  \nonumber
  &\int_{\msx_1^{(1:b)} \times \msx_2^b}\norm{\textstyle\sum_{i=1}^{b}\up_i\pr{H_{i}^\star\prn{\theta,x^{(1,i)}},x^{(2,i)}}-\nabla\U\prn{\theta}}^{2}\nu_2^{\otimes b}\prn{\rmd x^{(2,1:b)}}\otimes_{i=1}^b\nu_1^{(i)}\prn{\rmd x^{(1,i)}} \\
  \nonumber
  &=\int_{\msx_1^{(1:b)} \times \msx_2^b}\Bigg\|\sum_{i=1}^{b}\up_i\pr{\frac{N_i}{n_i}\sum_{j=1}^{N_i}\mathbf{1}_{x^{(1,i)}}(j)\br{\nabla\U_{i,j}\prn{\theta}-\nabla\U_{i,j}(\theta^{\star})}, x^{(2,i)}}\\
  \nonumber
  &-\sum_{i=1}^{b}\frac{N_i}{n_i}\sum_{j=1}^{N_i}\mathbf{1}_{x^{(1,i)}}(j)\br{\nabla\U_{i,j}\prn{\theta}-\nabla\U_{i,j}(\theta^{\star})}\Bigg\|^2  \\
  \nonumber
  &+ \int_{\msx_1^{(1:b)}}\norm{\sum_{i=1}^{b}\frac{N_i}{n_i}\sum_{j=1}^{N_i}\pr{\mathbf{1}_{x^{(1,i)}}(j)-\frac{n_i}{N_i}}\br{\nabla\U_{i,j}\prn{\theta}-\nabla\U_{i,j}(\theta^{\star})}}^2\otimes_{i=1}^b\nu_1^{(i)}\prn{\rmd x^{(1,i)}}\\
  \nonumber
  &\le \sum_{i=1}^{b}\omega_i\pr{\frac{N_i}{n_i}}^2\int_{\msx_{1}^{(i)}}\norm{\sum_{j=1}^{N_i}\mathbf{1}_{x^{(1,i)}}(j)\br{\nabla\U_{i,j}\prn{\theta}-\nabla\U_{i,j}(\theta^{\star})}}^2\nu_{1}^{(i)}\prn{\rmd x^{(1,i)}}\\
  \nonumber
  &+\sum_{i=1}^{b}\pr{\frac{N_i}{n_i}}^2\int_{\msx_{1}^{(i)}}\norm{\sum_{j=1}^{N}\pr{\mathbf{1}_{x^{(1,i)}}(j)-\frac{n_i}{N_i}}\br{\nabla\U_{i,j}\prn{\theta}-\nabla\U_{i,j}(\theta^{\star})}}^2\nu_{1}^{(i)}\prn{\rmd x^{(1,i)}}\\
  \nonumber
  &=\sum_{i=1}^{b}\omega_i\norm{\nabla\U_{i}\prn{\theta}-\nabla\U_{i}(\theta^{\star})}^2 \\
  &+ \sum_{i=1}^{b}\prn{\omega_i+1}\pr{\frac{N_i}{n_i}}^2\int_{\msx_{1}^{(i)}}\norm{\sum_{j=1}^{N_i}\pr{\mathbf{1}_{x^{(1,i)}}(j)-\frac{n_i}{N_i}}\br{\nabla\U_{i,j}\prn{\theta}-\nabla\U_{i,j}(\theta^{\star})}}^2\nu_{1}^{(i)}\prn{\rmd x^{(1,i)}}\eqsp. \label{eq:bound:fp:1}
\end{align}

Using \Cref{lem:var_grad_sto} combined with \Cref{ass:potential_fij} yields, for any $i \in [b]$,
\begin{multline}\label{eq:bound:fp:3}
	\int_{\msx_{1}^{(i)}}\norm{\textstyle\sum_{j=1}^{N_i}\pr{\mathbf{1}_{x^{(1,i)}}(j)-n_i/N_i}\br{\nabla\U_{i,j}\prn{\theta}-\nabla\U_{i,j}(\theta^{\star})}}^2\nu_{1}^{(i)}\prn{\rmd x^{(1,i)}}\\
	\le \frac{n_i(N_i-n_i)}{N_i(N_i-1)}\bMH\ps{\theta-\theta^{\star}}{\nabla\U_{i}\prn{\theta}-\nabla\U_{i}(\theta^{\star})}\eqsp.
\end{multline}
In addition, Jensen inequality implies, for any $i\in[b]$, that 
$$
\normn{\nabla\U_{i}\prn{\theta}-\nabla\U_{i}\prn{\theta^{\star}}}^2\le N_i\sum_{j=1}^{N_i}\\\norm{\nabla\U_{i,j}\prn{\theta}-\nabla\U_{i,j}\prn{\theta^{\star}}}^2\eqsp,
$$
and therefore, using \Cref{ass:potential_fij}, we have for any $i \in [b]$,
\begin{equation}\label{eq:bound:potential_coco}
	\normn{\nabla\U_{i}\prn{\theta}-\nabla\U_{i}\prn{\theta^{\star}}}^2
	\le \bMH N_i\ps{\nabla\U_{i}\prn{\theta}-\nabla\U_{i}\prn{\theta^{\star}}}{\theta-\theta^{\star}}\eqsp.
\end{equation}
Injecting \eqref{eq:bound:fp:3} and \eqref{eq:bound:potential_coco} into \eqref{eq:bound:fp:1} and using \eqref{eq:max5} conclude the proof.
\end{proof}


\section{PROOF OF \Cref{theorem_QLSDpp}}


\subsection{Problem formulation.}

We assume here that $U$ is still of the form \eqref{eq:target_density}  and that there exist $\{N_i \in \N^*\}_{i \in [b]}$ such that for any $i\in[b]$, there exist $N_i$ functions $\acn{U_{i,j}:\theta\in\Rd\to\R}_{j\in[N_i]}$ such that for any $\theta\in\Rd$,
\begin{equation*}\label{eq:def:Uij}
	U_{i}\prn{\theta}=\sum_{j=1}^{N_i}U_{i,j}\prn{\theta}\eqsp.
\end{equation*}
In all this section, we assume for any $i \in [b]$ that $n_i \in \nsets$, $n_i \le N_i$ is fixed. 
Recall that  $\wpN$ denotes the power set of $[N]$ and
\begin{equation*}
  \label{eq:2}
  \wpNn = \{ x \in \wpN \, :\,  \card(x) = n\} \eqsp. 
\end{equation*}
In addition, we set in this section $\nu_1^{(i)}$ as the uniform distribution on $\wp_{N_i,n_i}$.
We consider the family of measurable functions $\acn{G_{i}:\Rd \times \Rd\times \wp_{N_i}\to\Rd}_{i\in[b]}$, defined for any $i \in [b]$, $\theta\in\Rd$, $\zeta\in\Rd$, $x \in \wp_{N_i,n_i}$ by
\begin{equation}
	\label{eq:definition_Gi_QLSDpp}	G_{i}\prn{\theta,\zeta;x}=\frac{N_i}{n_i}\sum_{j=1}^{N_i}\mathbf{1}_{x}(j)\br{\nabla\U_{i,j}(\theta)-\nabla\U_{i,j}(\zeta)} + \nabla\U_{i}(\zeta)\eqsp.
\end{equation}
For ease of reading, we formalise more precisely  the recursion associated with \texttt{QLSD}$^{++}$ under \Cref{ass:compression}. 
Let $(X_k^{(1,1)},\ldots,X_k^{(1,b)})_{k \in\nsets}$ and $(X_k^{(2,1)},\ldots,X_k^{(2,b)})_{k \in\nsets}$  be  two independent \iid~sequences with distribution $\otimes_{i=1}^b\nu_1^{(i)}$ and $\nu_2^{\otimes b}$. Let $(Z_k)_{k\in\nsets}$ be an \iid~sequence of $d$-dimensional standard Gaussian random variables independent of $(X_k^{(1,1)},\ldots,X_k^{(1,b)})_{k \in\nsets}$ and $(X_k^{(2,1)},\ldots,X_k^{(2,b)})_{k \in\nsets}$. 
Similarly as before, we consider the partial device participation context where at each communication round $k\ge1$, each client has a probability $p_i \in (0,1]$ of participating, independently from other clients.
In other words, there exists a sequence $(X_k^{(3,1)},\cdots,X_k^{(3,b)})_{k\in \mathbb{N}^*}$ of i.i.d.  random variables distributed according $\nu_3 = \mathrm{Uniform}((0,1])$, such that for any $k\ge1$ and $i\in[b]$, client $i$ is active at step $k$ if $X_k^{(3,i)} \le p_i$.
We denote $\mathcal{A}_{k+1} = \{i \in [b]; X_{k+1}^{(3,i)} \le p_i\}$ the set of active clients at round $k$. 
For ease of notation, denote for any $k \in\nsets$, $X_{k}^{(1)} = (X_k^{(1,1)},\ldots,X_k^{(1,b)})$, $X_k^{(2)}= (X_k^{(2,1)},\ldots,X_k^{(2,b)})$, $X_k^{(3)}= (X_k^{(3,1)},\ldots,X_k^{(3,b)})$ and $X_k = (X_k^{(1)},X_k^{(2)},X_k^{(3)})$. 
Let $l \in \mathbb{N}^*$, $\gamma \in (0,\bar{\gamma}]$ and $\alpha \in (0,\bar{\alpha}]$ for $\bar{\gamma},\bar{\alpha} > 0$.
Given $\Theta_0 = (\theta_0,\zeta_0,\{\eta_0^{(i)}\}_{i \in [b]}) \in \rset^d \times \rset^d \times \rset^{db}$, with $\zeta_0 = \theta_0$, we recursively define the sequence $(\Theta_k)_{k \in \mathbb{N}} = (\theta_k,\zeta_k,\{\eta_k^{(i)}\}_{i \in [b]})_{k \in \mathbb{N}}$, for any $k \in \mathbb{N}$ as
\begin{equation}
	\label{eq:recursion_theta_QLSDpp}
	\theta_{k+1} = \theta_{k} - \gamma \tilde{G}(\Theta_{k};X_{k+1}) + \sqrt{2\gamma}Z_{k+1}\eqsp,
\end{equation}
where 
\begin{equation}
	\label{eq:QLSDpp_gradsto}
	\tilde{G}(\Theta_{k};X_{k+1}) = \sum_{i=1}^b\br{\mathscr{S}_i\pr{\mathscr{C}_i\left\{G_{i}\pr{\theta_{k},\zeta_{k};X_{k+1}^{(1,i)}} - \eta^{(i)}_{k};X_{k+1}^{(2,i)}\right\},X_{k+1}^{(3,i)}} + \eta^{(i)}_{k}},
\end{equation}
\begin{equation}
	\zeta_{k+1} = \begin{cases}
					\theta_{k+1} \eqsp, \text{ if } k+1 \equiv 0 \pmod l \eqsp,\\
					\zeta_{k}\eqsp , \text{ otherwise}\eqsp,\label{eq:QLSDpp_zeta}
				\end{cases}
\end{equation}
and for any $i \in [b]$,
\begin{equation}
	\label{eq:recursion_eta_QLSDpp}
	\eta^{(i)}_{k+1} = \eta^{(i)}_{k} + \alpha \mathscr{S}_i\pr{\mathscr{C}_i\left\{G_{i}\pr{\theta_{k},\zeta_{k};X_{k+1}^{(1,i)}} - \eta^{(i)}_{k};X_{k+1}^{(2,i)}\right\},X_{k+1}^{(3,i)}} \eqsp.
\end{equation}
Since \texttt{QLSD}$^{++}$ involves auxiliary variables gathered with $(\theta_k)_{k\in\nset}$ in $(\Theta_k)_{k \in \N}$, we cannot follow the same proof as for \texttt{QLSD}$^\star$ by verifying \Cref{ass:tgrad:sharp} and then applying  \Cref{thm:bound:sec_continuous_vs_cv}. 
Instead, we will adapt the proof \Cref{thm:bound:sec_continuous_vs_cv} and in particular \Cref{lem:bound:thetak_minustheta_star}  and bound the variance associated to the stochastic gradient defined in \eqref{eq:QLSDpp_gradsto}.
Once this variance term will be tackled, the proof of \Cref{theorem_QLSDpp} will follow the same lines as the proof of \Cref{thm:bound:sec_continuous_vs_cv} upon using specific moment estimates for \qlsdpp. In the next section, we focus on these two goals: we provide uniform bounds in the number of iterations $k$  on the variance of the sequence of stochastic gradients associated with \qlsdpp, $(\PE[\normLigne{\tilde{G}_i(\Theta_k,X_{k+1}) - \nabla U(\theta_k)}^2])_{k \in\nset}$ for any $i \in [b]$, and $(\PE[\normLigne{\theta_k-\thetas}^2])_{k \in\nset}$, see \Cref{proposition:varianceQLSDpp_final} and \Cref{coro:lyapunov}. To this end, a key ingredient is the design of an appropriate Lyapunov function defined in \eqref{eq:lyapunov}. 

\subsection{Uniform bounds on the stochastic gradients and moment estimates for \texttt{QLSD}$^{++}$}
Consider the filtration associated with $(\Theta_k)_{k\in\nset}$ defined by $\mcg_0 = \sigma(\Theta_0)$ and for $k \in\nsets$, 
\begin{equation*}
	\label{filtration:QLSDpp}
	\mcg_{k} = \sigma (\Theta_0, (X_{\tilde{k}})_{\tilde{k} \le k}, (Z_{\tilde{k}})_{\tilde{k} \le k})\eqsp.
\end{equation*}
We denote for any $i \in [b]$, $\theta, \zeta \in \Rd$, 
	\begin{equation}
		\label{eq:defDelta_i}
		\Delta_i(\theta,\zeta) = \nabla U_i(\theta) - \nabla U_i(\zeta)\eqsp.
	\end{equation}
	Similarly, we consider, for any $i \in [b]$, $j \in [N]$, $\theta, \zeta \in \Rd$,
	\begin{equation}
		\label{eq:defDelta_ij}
		\Delta_{i,j}(\theta,\zeta) = \nabla U_{i,j}(\theta) - \nabla U_{i,j}(\zeta)\eqsp.
	\end{equation}
	The following lemma provides a first upper bound on the variance of the stochastic gradients used in \qlsdpp.
\begin{lemma}
	\label{lemma:varianceQLSDpp}
	Assume \Cref{ass:potential_U}, \Cref{ass:compression}, \Cref{ass:A_k_supp} and \Cref{ass:potential_fij} and let $\gamma \in (0,\bar{\gamma}]$, $\alpha \in (0,\bar{\alpha}]$ for some $\bar{\gamma},\bar{\alpha} > 0$.
	Then, for any $s \in \nset$, $r \in \{0,\ldots,l-1\}$, we have
	\begin{align*}
		&\mathbb{E}^{\mathcal{G}_{sl+r}}\br{\norm{\tilde{G}(\Theta_{sl+r};X_{sl+r+1}) - \nabla U(\theta_{sl+r})}^2} \le \br{2\sum_{i=1}^b \frac{\mathtt{M}_i^2}{p_i}(\omega_i+1-p_i) + \pr{\frac{\omega_i+1}{p_i}}A_{n_i,N_i}\barM \mathtt{M}_i}\norm{\theta_{sl+r} - \theta^\star}^2 \\ 
		&+ \br{2\sum_{i=1}^b(\omega_i+1-p_i)/p_i} \norm{\nabla U_{i}(\theta^\star) - \eta^{(i)}_{sl+r}}^2 
		+ 2 \barM \sum_{i=1}^b\br{\pr{\frac{\omega_i+1}{p_i}}A_{n_i,N_i}\mathtt{M}_i}\norm{\theta_{sl} - \theta^\star}^2\eqsp,
	\end{align*}
	where $(\Theta_{\tilde{k}})_{\tilde{k} \in\nset} = (\theta_{\tilde{k}},\zeta_{\tilde{k}},\{\eta^{(i)}_{\tilde{k}}\}_{i \in [b]})_{\tilde{k} \in\nset}$, $\tilde{G}$ and $A_{n,N}$ are defined in \eqref{eq:recursion_theta_QLSDpp}, \eqref{eq:QLSDpp_zeta}, \eqref{eq:recursion_eta_QLSDpp}, \eqref{eq:QLSDpp_gradsto} and \eqref{eq:defAnN}, respectively.

\end{lemma}
\begin{proof}
	Let $s \in \N$ and $r \in \{0,\ldots,l-1\}$. 
	Using \Cref{ass:compression}, \eqref{eq:defDelta_i} and \eqref{eq:defDelta_ij}, we have
	\begin{align*}
		&\mathbb{E}^{\mathcal{G}_{sl+r}}\br{\norm{\tilde{G}(\Theta_{sl+r};X_{sl+r+1}) - \nabla U(\theta_{sl+r})}^2} \\
		&= \sum_{i=1}^b\mathbb{E}^{\mathcal{G}_{sl+r}}\big[\big\|\mathscr{S}_i\pr{\mathscr{C}_i\left\{G_{i}\pr{\theta_{sl+r},\zeta_{sl+r};X_{sl+r+1}^{(1,i)}} - \eta^{(i)}_{sl+r};X_{sl+r+1}^{(2,i)}\right\},X_{sl+r+1}^{(3,i)}} \\
		&- \mathscr{C}_i\left\{G_{i}\pr{\theta_{sl+r},\zeta_{sl+r};X_{sl+r+1}^{(1,i)}} - \eta^{(i)}_{sl+r};X_{sl+r+1}^{(2,i)}\right\}\big\|^2\big] \\
		&+ \sum_{i=1}^b\mathbb{E}^{\mathcal{G}_{sl+r}}\br{\norm{\mathscr{C}_i\left\{G_{i}\pr{\theta_{sl+r},\zeta_{sl+r};X_{sl+r+1}^{(1,i)}} - \eta^{(i)}_{sl+r};X_{sl+r+1}^{(2,i)}\right\} + \eta_{sl+r}^{(i)} - \nabla U_i(\theta_{sl+r})}^2} \\
		&\le \sum_{i=1}^b\pr{\frac{1-p_i}{p_i}}\mathbb{E}^{\mathcal{G}_{sl+r}}\br{\norm{\mathscr{C}_i\left\{G_{i}\pr{\theta_{sl+r},\zeta_{sl+r};X_{sl+r+1}^{(1,i)}} - \eta^{(i)}_{sl+r};X_{sl+r+1}^{(2,i)}\right\}}^2} \\
		&+ \sum_{i=1}^b\omega_i\mathbb{E}^{\mathcal{G}_{sl+r}}\br{\norm{G_{i}\pr{\theta_{sl+r},\zeta_{sl+r};X_{sl+r+1}^{(1,i)}}-\eta_{sl+r}^{(i)}}}
		+ \sum_{i=1}^b\mathbb{E}^{\mathcal{G}_{sl+r}}\br{\norm{G_{i}\pr{\theta_{sl+r},\zeta_{k};X_{sl+r+1}^{(1,i)}} - \nabla U_i(\theta_{sl+r})}^2} \\
		&\le \sum_{i=1}^b\pr{\frac{\omega_i+1-p_i}{p_i}}\mathbb{E}^{\mathcal{G}_{sl+r}}\br{\norm{G_{i}\pr{\theta_{sl+r},\zeta_{sl+r};X_{sl+r+1}^{(1,i)}} - \eta^{(i)}_{sl+r}}^2} \\
		&+ \sum_{i=1}^b\mathbb{E}^{\mathcal{G}_{sl+r}}\br{\norm{G_{i}\pr{\theta_{sl+r},\zeta_{sl+r};X_{sl+r+1}^{(1,i)}} - \nabla U_i(\theta_{sl+r})}^2} \\
		&\le \sum_{i=1}^b\pr{\frac{\omega_i+1}{p_i}}\mathbb{E}^{\mathcal{G}_{sl+r}}\br{\norm{\frac{N_i}{n_i}\sum_{j=1}^{N_i}\bbr{\mathbf{1}_{X_{sl+r+1}^{(1,i)}}(j) \Delta_{i,j}(\theta_{sl+r},\zeta_{sl+r})} - \Delta_{i}(\theta_{sl+r},\zeta_{sl+r})}^2} \\
		&+ \sum_{i=1}^b\pr{\frac{\omega_i+1-p_i}{p_i}}\mathbb{E}^{\mathcal{G}_{sl+r}}\br{\norm{\nabla U_i(\theta_{sl+r}) - \eta_{sl+r}^{(i)}}^2} \\
		&\le \sum_{i=1}^b\pr{\frac{\omega_i+1}{p_i}}\frac{N_i(N_i-n_i)}{n_i(N_i-1)}\barM\langle\theta_{sl+r}-\zeta_{sl+r},\nabla U_i(\theta_{sl+r}) - \nabla U_i(\zeta_{sl+r})\rangle \\
		&+ \sum_{i=1}^b\pr{\frac{\omega_i+1-p_i}{p_i}}\mathbb{E}^{\mathcal{G}_{sl+r}}\br{\norm{\nabla U_i(\theta_{sl+r}) - \eta_{sl+r}^{(i)}}^2} \eqsp,
	\end{align*}
	where the last line follows from \Cref{ass:potential_fij} and \Cref{lem:var_grad_sto}.
	The proof is concluded by using the Cauchy-Schwarz inequality, \Cref{ass:potential_U} and $\zeta_{sl+r} = \theta_{sl}$.
\end{proof}
The two following lemmas aim at controlling the terms that appear in \Cref{lemma:varianceQLSDpp}.


\begin{lemma}
	\label{lemma:QLSDpp_theta_thetastar}
	Assume \Cref{ass:potential_U}, \Cref{ass:compression}, \Cref{ass:A_k_supp} and \Cref{ass:potential_fij}, and let $\gamma \in (0,\bar{\gamma}]$, $\alpha \in (0,\bar{\alpha}]$ for some $\bar{\gamma},\bar{\alpha} > 0$.
	Then, for any $s \in \nset$ and $r \in [l]$, we have
	\begin{align*}
		\mathbb{E}^{\mathcal{G}_{sl+r-1}}&\br{\norm{\theta_{sl+r}-\theta^{\star}}^2}
		\le \pr{1 -2\gamma m +  \gamma^2B_{\mathbf{n},\mathbf{N}}}\norm{\theta_{sl+r-1}-\theta^\star}^2 \\
		&+ \gamma^2\br{2\sum_{i=1}^b(\omega_i+1-p_i)/p_i} \norm{\nabla U_{i}(\theta^\star) - \eta^{(i)}_{sl+r-1}}^2 
		+ 2 \barM \gamma^2 \sum_{i=1}^b\br{\pr{\frac{\omega_i+1}{p_i}}A_{n_i,N_i}\mathtt{M}_i}\norm{\theta_{sl} - \theta^\star}^2 + 2\gamma d\eqsp,
	\end{align*}
	where
	\begin{equation}
		\label{eq:def_B_n_N}
		B_{\mathbf{n},\mathbf{N}}= 2\sum_{i=1}^b \bbr{\frac{\mathtt{M}_i^2}{p_i}(\omega_i+1-p_i) + \pr{\frac{\omega_i+1}{p_i}}A_{n_i,N_i}\barM \mathtt{M}_i} + \lip^2 \eqsp,
	\end{equation}
	$(\Theta_{\tilde{k}})_{\tilde{k} \in\nset} = (\theta_{\tilde{k}},\zeta_{\tilde{k}},\{\eta^i_{\tilde{k}}\}_{i \in [b]})_{\tilde{k} \in\nset}$ and $A_{n,N}$ are defined in \eqref{eq:recursion_theta_QLSDpp}, \eqref{eq:QLSDpp_zeta}, \eqref{eq:recursion_eta_QLSDpp} and \eqref{eq:defAnN} respectively.
\end{lemma}
\begin{proof}
	Let $s \in \N$ and $r \in [l]$.
	Using \eqref{eq:recursion_theta_QLSDpp} and \Cref{ass:compression}, it follows 
	\begin{multline}\label{eq:proofQLSDpp1}
		\mathbb{E}^{\mathcal{G}_{sl+r-1}}\br{\norm{\theta_{sl+r} - \theta^\star}^2}
		= \norm{\theta_{sl+r-1} - \theta^\star}^2 + 2\gamma d - 2\gamma \langle\nabla U(\theta_{sl+r-1}),\theta_{sl+r-1}-\theta^\star \rangle \\
		+ \gamma^2\mathbb{E}^{\mathcal{G}_{sl+r-1}}\br{\norm{\tilde{G}(\Theta_{sl+r-1};X_{sl+r})}^2}\eqsp.
	\end{multline}

	Using \Cref{ass:compression} and \eqref{eq:definition_Gi_QLSDpp}-\eqref{eq:QLSDpp_gradsto}, we have
	\begin{align}
		\nonumber
		&\mathbb{E}^{\mathcal{G}_{sl+r-1}}\br{\norm{\tilde{G}(\Theta_{sl+r-1};X_{sl+r})}^2} \\
		\nonumber
		&= \sum_{i=1}^b\mathbb{E}^{\mathcal{G}_{sl+r-1}}\Bigg[\Bigg\|\mathscr{S}_i\pr{\mathscr{C}_i\left\{G_{i}\pr{\theta_{sl+r-1},\zeta_{sl+r-1};X_{sl+r}^{(1,i)}} - \eta^{(i)}_{sl+r-1};X_{sl+r}^{(2,i)}\right\},X_{sl+r}^{(3,i)}} \\ 
		\nonumber
		&- \mathscr{C}_i\left\{G_{i}\pr{\theta_{sl+r-1},\zeta_{sl+r-1};X_{sl+r}^{(1,i)}} - \eta^{(i)}_{sl+r-1};X_{sl+r}^{(2,i)}\right\}\Bigg\|^2\Bigg] \\
		\nonumber
		&+ \mathbb{E}^{\mathcal{G}_{sl+r-1}}\br{\norm{\sum_{i=1}^b\mathscr{C}_i\left\{G_{i}\pr{\theta_{sl+r-1},\zeta_{sl+r-1};X_{sl+r}^{(1,i)}} - \eta^{(i)}_{sl+r-1};X_{sl+r}^{(2,i)}\right\} + \eta_{sl+r-1}^{(i)}}^2} \\
		\nonumber
		&\le \sum_{i=1}^b\pr{\frac{\omega_i+1-p_i}{p_i}}\mathbb{E}^{\mathcal{G}_{sl+r-1}}\br{\norm{G_{i}\pr{\theta_{sl+r-1},\zeta_{sl+r-1};X_{sl+r}^{(1,i)}} - \eta^{(i)}_{sl+r-1}}^2} \\
		\nonumber
		&+ \sum_{i=1}^b\mathbb{E}^{\mathcal{G}_{sl+r-1}}\br{\norm{\frac{N_i}{n_i}\sum_{j=1}^{N_i}\bbr{\mathbf{1}_{X_{sl+r}^{(1,i)}}(j) \Delta_{i,j}(\theta_{sl+r-1},\zeta_{sl+r-1})} - \Delta_{i}(\theta_{sl+r-1},\zeta_{sl+r-1})}^2}
		+ \norm{\nabla U(\theta_{sl+r-1})}^2 \\
		\nonumber
		&= \sum_{i=1}^b\pr{\frac{\omega_i+1}{p_i}}\mathbb{E}^{\mathcal{G}_{sl+r-1}}\br{\norm{\frac{N_i}{n_i}\sum_{j=1}^{N_i}\bbr{\mathbf{1}_{X_{sl+r}^{(1,i)}}(j) \Delta_{i,j}(\theta_{sl+r-1},\zeta_{sl+r-1})} - \Delta_{i}(\theta_{sl+r-1},\zeta_{sl+r-1})}^2} \\
		\nonumber
		&+ \sum_{i=1}^b\pr{\frac{\omega_i+1-p_i}{p_i}}\mathbb{E}^{\mathcal{G}_{sl+r-1}}\br{\norm{\nabla U_i(\theta_{sl+r-1}) - \eta^{(i)}_{sl+r-1}}^2} 
		+ \norm{\nabla U(\theta_{sl+r-1})}^2 \\
		\nonumber
		&\le \sum_{i=1}^b\pr{\frac{\omega_i+1}{p_i}}\frac{N_i(N_i-n_i)}{n_i(N_i-1)}\barM\langle\theta_{sl+r-1}-\zeta_{sl+r-1},\nabla U_i(\theta_{sl+r-1}) - \nabla U_i(\zeta_{sl+r-1})\rangle \\
		&+ \sum_{i=1}^b\pr{\frac{\omega_i+1-p_i}{p_i}}\norm{\nabla U_i(\theta_{sl+r-1}) - \eta_{sl+r-1}^{(i)}}^2 + \norm{\nabla U(\theta_{sl+r-1})}^2 \eqsp, \label{eq:proofQLSDpp2}
	\end{align}
	where the last line follows from \Cref{ass:potential_fij} and \Cref{lem:var_grad_sto}.
	The proof is concluded by injecting \eqref{eq:proofQLSDpp2} into \eqref{eq:proofQLSDpp1}, using the Cauchy-Schwarz inequality, $\nabla U(\thetas) = 0$, \Cref{ass:potential_U} and $\zeta_{sl+r-1} = \theta_{sl}$.
\end{proof}

\begin{lemma}
	\label{lemma:QLSDpp_thetastar_eta}
	Assume \Cref{ass:potential_U}, \Cref{ass:compression}, \Cref{ass:A_k_supp} and \Cref{ass:potential_fij}. 
	Let $\gamma \in (0,\bar{\gamma}]$ for some $\bar{\gamma}>0$ and $\alpha \in (0,1/(\max_{i\in[b]}\omega_i+1)]$.
	Then, for any $s\in \N$ and $r \in [l]$, we have
	\begin{multline*}
	\sum_{i=1}^b\mathbb{E}^{\mathcal{G}_{sl+r-1}}\br{\norm{\nabla U_{i}(\theta^{\star}) - \eta^{(i)}_{sl+r}}^2} 
	\le (1-\alpha)\sum_{i=1}^b\norm{\nabla U_{i}(\theta^{\star})- \eta^{(i)}_{sl+r-1}}^2 \\
	\qquad \qquad + \alpha C_{\mathbf{n},\mathbf{N}}\norm{\theta_{sl+r-1} - \theta^\star}^2 + 2\alpha \br{\sum_{i=1}^bA_{n_i,N_i}\barM \mathtt{M}_i}\norm{\theta_{sl} - \theta^\star}^2\eqsp,
	\end{multline*}
	where
        \begin{equation}
          \label{eq:def_C_n_N}
						C_{\mathbf{n},\mathbf{N}} =           
						2\sum_{i=1}^b \bbr{A_{n_i,N_i}\barM \mathtt{M}_i + \mathtt{M}_i^2} \eqsp,
        \end{equation}
        $(\Theta_{\tilde{k}})_{\tilde{k} \in\nset} = (\theta_{\tilde{k}},\zeta_{\tilde{k}},\{\eta^i_{\tilde{k}}\}_{i \in [b]})_{\tilde{k} \in\nset}$ and $A_{n,N}$ are defined in \eqref{eq:recursion_theta_QLSDpp}, \eqref{eq:QLSDpp_zeta}, \eqref{eq:recursion_eta_QLSDpp} and \eqref{eq:defAnN}, respectively.
\end{lemma}
\begin{proof}
	Let $s \in \N$ and $r \in [l]$. 
	Then, it follows
	\begin{align}
		\nonumber
		&\sum_{i=1}^b\mathbb{E}^{\mathcal{G}_{sl+r-1}}\br{\norm{\nabla U_{i}(\theta^{\star}) - \eta^{(i)}_{sl+r}}^2} = \sum_{i=1}^b\norm{\nabla U_{i}(\theta^{\star})- \eta^{(i)}_{sl+r-1}}^2 \\
		&+ \sum_{i=1}^b \mathbb{E}^{\mathcal{G}_{sl+r-1}}\br{\norm{\eta^{(i)}_{sl+r} - \eta^{(i)}_{sl+r-1}}^2} 
		+ 2\sum_{i=1}^b\langle\mathbb{E}^{\mathcal{G}_{sl+r-1}}\br{\eta^{(i)}_{sl+r} - \eta^{(i)}_{sl+r-1}},\eta^{(i)}_{sl+r-1} - \nabla U_{i}(\theta^{\star})\rangle\eqsp.\label{eq:QLSDpp1}
	\end{align}
	Using \eqref{eq:recursion_eta_QLSDpp} and \Cref{ass:compression}, we have for any $i \in [b]$,
	\begin{align}
		\nonumber
		\mathbb{E}^{\mathcal{G}_{sl+r-1}}&\br{\norm{\eta^{(i)}_{sl+r} - \eta^{(i)}_{sl+r-1}}^2} \\
		&\le \alpha^2(\omega_i+1)\mathbb{E}^{\mathcal{G}_{sl+r-1}}\br{\norm{G_{i}\pr{\theta_{sl+r-1},\zeta_{sl+r-1};X_{sl+r}^{(1,i)}} - \eta^{(i)}_{sl+r-1}}^2}\eqsp, \label{eq:QLSDpp2}
	\end{align}
	\begin{align}
		\mathbb{E}^{\mathcal{G}_{sl+r-1}}\br{\eta^{(i)}_{sl+r} - \eta^{(i)}_{sl+r-1}} &= \alpha\mathbb{E}^{\mathcal{G}_{sl+r-1}}\br{G_{i}\pr{\theta_{sl+r-1},\zeta_{sl+r-1};X_{sl+r}^{(1,i)}} - \eta^{(i)}_{sl+r-1}}\eqsp.\label{eq:QLSDpp3}
	\end{align}
	Plugging \eqref{eq:QLSDpp2} and \eqref{eq:QLSDpp3} into \eqref{eq:QLSDpp1} yields
	\begin{align*}
		\sum_{i=1}^b&\mathbb{E}^{\mathcal{G}_{sl+r-1}}\br{\norm{\nabla U_{i}(\theta^{\star}) - \eta^{(i)}_{sl+r}}^2}
		\le \sum_{i=1}^b\norm{\nabla U_{i}(\theta^{\star})- \eta^{(i)}_{sl+r-1}}^2 \\
		&\qquad \qquad+ \alpha^2\sum_{i=1}^b(\omega_i+1)\mathbb{E}^{\mathcal{G}_{sl+r-1}}\br{\norm{G_{i}\pr{\theta_{sl+r-1},\zeta_{sl+r-1};X_{sl+r}^{(1,i)}} - \eta^{(i)}_{sl+r-1}}^2} \\
		&\qquad \qquad+ 2\alpha\sum_{i=1}^b\langle\mathbb{E}^{\mathcal{G}_{sl+r-1}}\br{G_{i}\pr{\theta_{sl+r-1},\zeta_{sl+r-1};X_{sl+r}^{(1,i)}} - \eta^{(i)}_{sl+r-1}},\eta^{(i)}_{sl+r-1} - \nabla U_{i}(\theta^{\star})\rangle \eqsp.
	\end{align*}
	Using for any $i\in [b]$ $\alpha(1+\omega_i) \le 1$ and the fact, for any $a,b,c \in \Rd$, that $\norm{a-c}^2 + 2\langle (a-c),(c-b)\rangle = \norm{a-b}^2 - \norm{c-b}^2$, we have
	\begin{align}
		\nonumber
		&\sum_{i=1}^b\mathbb{E}^{\mathcal{G}_{sl+r-1}}\br{\norm{\nabla U_{i}(\theta^{\star}) - \eta^{(i)}_{sl+r}}^2}
		\le (1-\alpha)\sum_{i=1}^b\norm{\nabla U_{i}(\theta^{\star})- \eta^{(i)}_{sl+r-1}}^2 \\
		&\qquad + \alpha\sum_{i=1}^b\mathbb{E}^{\mathcal{G}_{sl+r-1}}\br{\norm{G_{i}\pr{\theta_{sl+r-1},\zeta_{sl+r};X_{sl+r}^{(1,i)}} - \nabla U_{i}(\theta^{\star})}^2}\eqsp.\label{eq:QLSDpp4}
	\end{align}
	Using \eqref{eq:definition_Gi_QLSDpp}, \Cref{ass:potential_fij} and \Cref{lem:var_grad_sto}, it follows
	\begin{align}
		\nonumber
		\sum_{i=1}^b
		&\mathbb{E}^{\mathcal{G}_{sl+r-1}}\br{\norm{G_{i}\pr{\theta_{sl+r-1},\zeta_{sl+r-1};X_{sl+r}^{(1,i)}} - \nabla U_{i}(\theta^{\star})}^2} \\
		\nonumber
		&\le \sum_{i=1}^b\frac{N_i(N_i-n_i)}{n_i(N_i-1)}\barM \langle\theta_{sl+r-1}-\zeta_{sl+r-1},\nabla U_i(\theta_{sl+r-1})-\nabla U_i(\zeta_{sl+r-1})\rangle \\
		&\qquad \qquad+ \sum_{i=1}^b\norm{\nabla U_{i}(\theta_{sl+r-1}) - \nabla U_{i}(\theta^{\star})}^2\eqsp.\label{eq:QLSDpp5}
	\end{align}
	The proof is concluded by plugging \eqref{eq:QLSDpp5} into \eqref{eq:QLSDpp4}, using the Cauchy-Schwarz inequality, \Cref{ass:potential_U} and $\zeta_{sl+r-1} = \theta_{sl}$.
\end{proof}

\Cref{lemma:QLSDpp_theta_thetastar} and \Cref{lemma:QLSDpp_thetastar_eta} involve two dependent terms which prevents us from using a straightforward induction.
To cope with this issue, we consider a Lyapunov function $\psi: \Rd \times \mathbb{R}^{bd} \to \mathbb{R}$ defined, for any  $\theta \in \Rd$ and $\eta = (\eta^{(1)},\ldots,\eta^{(b)})^{\top} \in \R^{bd}$ by
\begin{equation}
	\label{eq:lyapunov}
	\psi(\theta,\eta) = \norm{\theta - \theta^{\star}}^2 + (3/\alpha) \max_{i \in [b]}\{(\omega_i+1-p_i)/p_i\} \gamma^2\sum_{i=1}^b \norm{\nabla U_{i}(\theta^{\star})- \eta^{(i)}}^2\eqsp.
\end{equation}
The following lemma provides an upper bound on this Lyapunov function.
Define for $\alpha >0$,
\begin{equation}
  \label{eq:def_bgamma_1}
  \bgamma_{\alpha,1} = \mathtt{m}^{-1}[\{\mtt^2(B_{\mathbf{n},\mathbf{N}} + 3 \omega C_{\mathbf{n},\mathbf{N}})^{-1}\}\wedge\{ \alpha/3\}] \eqsp,
\end{equation}
where $B_{\mathbf{n},\mathbf{N}}$ and $C_{\mathbf{n},\mathbf{N}}$ are defined in \eqref{eq:def_B_n_N} and \eqref{eq:def_C_n_N} respectively.
\begin{lemma}
	\label{lemma:QLSDpp_{i}nnerrecursionLyapunov}
	Assume \Cref{ass:potential_U}, \Cref{ass:compression}, \Cref{ass:A_k_supp} and \Cref{ass:potential_fij}. 
	Let $\alpha \in (0,1/(1+\max_{i\in[b]}\omega_i)]$,  $\gamma \in (0,\bgamma_{\alpha,1}]$.
	Then, for any $s \in \N$ and $r \in [l]$, we have
	\begin{align*}
		\mathbb{E}^{\mathcal{G}_{sl+r-1}}\br{\psi(\theta_{sl+r},\eta_{sl+r})} &\le \pr{1 - \gamma \mathtt{m}}\psi(\theta_{sl+r-1},\eta_{sl+r-1}) \\
		&\qquad \qquad+8\barM \gamma^2\max_{i \in [b]}\{(\omega_i+1)/p_i\} \sum_{i=1}^b A_{n_i,N_i} \mathtt{M}_i\norm{\theta_{sl}- \theta^{\star}}^2
		+2\gamma d\eqsp,
	\end{align*}
	where $\psi$ is defined in \eqref{eq:lyapunov} and  $(\Theta_{\tilde{k}})_{\tilde{k} \in\nset} = (\theta_{\tilde{k}},\zeta_{\tilde{k}},\{\eta^i_{\tilde{k}}\}_{i \in [b]})_{\tilde{k} \in\nset}$ and $A_{n,N}$ are defined in \eqref{eq:recursion_theta_QLSDpp}, \eqref{eq:QLSDpp_zeta}, \eqref{eq:recursion_eta_QLSDpp} and \eqref{eq:defAnN}, respectively.
\end{lemma}
\begin{proof}
	Let $s \in \N$ and $r \in [l]$.
	Using \Cref{lemma:QLSDpp_theta_thetastar} and \Cref{lemma:QLSDpp_thetastar_eta}, we have
	\begin{align*}
		\mathbb{E}^{\mathcal{G}_{sl+r-1}}&\br{\psi(\theta_{sl+r},\eta_{sl+r})} \\
		&\le\pr{1 - 2\gamma \mathtt{m} + \gamma^2\br{B_{\mathbf{n},\mathbf{N}} + 3 \omega C_{\mathbf{n},\mathbf{N}}}}\norm{\theta_{sl+r-1} - \theta^{\star}}^2 \\
		&\qquad \qquad+\br{		(2/3)\alpha  + (1-\alpha)} (3\gamma^2/\alpha)\max_{i \in [b]}\{(\omega_i+1-p_i)/p_i\} \sum_{i=1}^b \norm{\nabla U_{i}(\theta^{\star})- \eta^{(i)}_{sl+r-1}}^2 \\
		&\qquad \qquad+8\barM \gamma^2\max_{i \in [b]}\{(\omega_i+1)/p_i\} \sum_{i=1}^b A_{n_i,N_i} \mathtt{M}_i\norm{\theta_{sl}- \theta^{\star}}^2
		+2\gamma d\eqsp.
	\end{align*}
	Since $\gamma \le \bgamma_{\alpha,1}$ with $\bgamma_{\alpha,1}$ given in \eqref{eq:def_bgamma_1}, it follows that 
	\begin{align*}
		1 - 2\gamma \mathtt{m} + \gamma^2\br{B_{\mathbf{n},\mathbf{N}} + 3 \omega C_{\mathbf{n},\mathbf{N}}} &\le 1 - \gamma \mathtt{m} \\
		(2/3)\alpha  + (1-\alpha) &\le 1 - \gamma \mathtt{m}\eqsp.
	\end{align*}
	Therefore, we have
	\begin{align*}
		\mathbb{E}^{\mathcal{G}_{sl+r-1}}\br{\psi(\theta_{sl+r},\eta_{sl+r})} &\le \pr{1 - \gamma \mathtt{m}}\psi(\theta_{sl+r-1},\eta_{sl+r-1}) \\
		&\qquad \qquad+8\barM \gamma^2\max_{i \in [b]}\{(\omega_i+1)/p_i\} \sum_{i=1}^b A_{n_i,N_i} \mathtt{M}_i\norm{\theta_{sl}- \theta^{\star}}^2
		+2\gamma d\eqsp.
	\end{align*}

\end{proof}


\begin{lemma}
	\label{lemma:gammacondition_QLSDpp}
	Let $j \in \N^*$ and fix $\gamma > 0$ such that 
	$$
	\gamma \le \frac{\mathtt{m}}{16j\barM \gamma^2\max_{i \in [b]}\{(\omega_i+1)/p_i\} \sum_{i=1}^b A_{n_i,N_i} \mathtt{M}_i} \wedge \frac{1}{\mathtt{m}}\eqsp.
	$$
	Then, $$\pr{1 - \gamma \mathtt{m}}^{j} + 8j\gamma^2\barM \max_{i \in [b]}\{(\omega_i+1)/p_i\} \sum_{i=1}^b A_{n_i,N_i} \mathtt{M}_i \le 1 - \gamma \mathtt{m} /2 \eqsp,$$ where  $A_{n,N}$ is defined in  \eqref{eq:defAnN}.
\end{lemma} 
\begin{proof}
	The proof is straightforward using $(1-\gamma \mtt)^j \le 1 - \gamma \mtt$.
\end{proof}

We have the following corollary regarding the Lyapunov function defined in \eqref{eq:lyapunov}.

Denote for $\alpha >0$,
\begin{equation}
  \label{eq:def_bgamma_2}
  \bgamma_{\alpha,2} = \bgamma_{\alpha,1} \wedge [\mathtt{m}/\{{16 l \barM \max_{i \in [b]}\{(\omega_i+1)/p_i\} \textstyle\sum_{i=1}^b A_{n_i,N_i} \mathtt{M}_i}\}]^{1/3} \eqsp,
\end{equation}
where $\bgamma_{\alpha,1}$ is given in \eqref{eq:def_bgamma_1}. 
\begin{corollary}
	\label{coro:lyapunov}
	Assume \Cref{ass:potential_U}, \Cref{ass:compression}, \Cref{ass:A_k_supp} and \Cref{ass:potential_fij}. 
	Let $\alpha \in (0,1/(1+\max_{i\in[b]}\omega_i)]$ and $\gamma \in\ocint{0,\bgamma_{\alpha,2}}$. 
	Then, for any $s \in \N$ and $r \in \{0,\ldots,l-1\}$ we have 
	\begin{align*}
		\mathbb{E}^{\mathcal{G}_{sl}}\br{\psi(\theta_{(s+1)l - r},\eta_{(s+1)l-r}} &\le \pr{1-\gamma\mathtt{m}/2}\psi(\theta_{sl},\eta_{sl}) +2 \gamma (l-r)d\eqsp,
	\end{align*}
	where $\psi$ is defined in \eqref{eq:lyapunov} and $(\Theta_{\tilde{k}})_{\tilde{k} \in\nset} = (\theta_{\tilde{k}},\zeta_{\tilde{k}},\{\eta^i_{\tilde{k}}\}_{i \in [b]})_{\tilde{k} \in\nset}$ is defined in \eqref{eq:recursion_theta_QLSDpp}, \eqref{eq:QLSDpp_zeta}, \eqref{eq:recursion_eta_QLSDpp}.
\end{corollary}
\begin{proof}
	The proof follows from a straightforward induction of \Cref{lemma:QLSDpp_{i}nnerrecursionLyapunov} combined with \Cref{lemma:gammacondition_QLSDpp}.
\end{proof}

We are now ready to control explicitly the variance of the stochastic gradient defined in \eqref{eq:QLSDpp_gradsto}.

\begin{proposition}
	\label{proposition:varianceQLSDpp_final}
	Assume \Cref{ass:potential_U}, \Cref{ass:compression}, \Cref{ass:A_k_supp} and \Cref{ass:potential_fij}. 
	Let $\alpha \in (0,1/(1+\max_{i \in [b]}\omega_i)]$ and $\gamma \in\ocint{0,\bgamma_{\alpha,2}}$, where $\bgamma_{\alpha,2}$ is defined in
\eqref{eq:def_bgamma_2}.
	Then, for any $k = sl + r$ with $s \in\nset$, $r \in\{0,\ldots,l-1\}$, $\theta_0 \in \Rd$ and $\eta_0 = (\eta_0^{(1)},\ldots,\eta_0^{(b)})^{\top} \in \R^{db}$, we have
	\begin{align*}
		\mathbb{E}\br{\norm{\tilde{G}(\Theta_{sl+r};X_{sl+r+1}) - \nabla U(\theta_{k})}^2} &\le (1-\gamma \mathtt{m}/2)^sD_{\mathbf{n},\mathbf{N}}\psi(\theta_0,\eta_0) + 4ldD_{\mathbf{n},\mathbf{N}}/\mtt\\
		&\qquad \qquad+ \br{2\sum_{i=1}^b(\omega_i+1-p_i)/p_i}(1-\alpha)^k\sum_{i=1}^b \mathbb{E}\br{\norm{\nabla U_{i}(\theta^\star) - \eta^{(i)}_{0}}^2}\eqsp,
	\end{align*}
where
\begin{equation}
	\label{eq:D_n_N}
	D_{\mathbf{n},\mathbf{N}} = \br{2\sum_{i=1}^b \frac{\mathtt{M}_i^2}{p_i}(\omega_i+1-p_i) + \pr{\frac{\omega_i+1}{p_i}}A_{n_i,N_i}\barM \mathtt{M}_i} + 2 \barM \sum_{i=1}^b\br{\pr{\frac{\omega_i+1}{p_i}}A_{n_i,N_i}\mathtt{M}_i} + 4 C_{\mathbf{n},\mathbf{N}} \sum_{i=1}^b(\omega_i+1-p_i)/p_i \eqsp,
\end{equation}
 $A_{n,N}$ and $C_{n,N}$ are defined in \eqref{eq:defAnN} and \eqref{eq:def_C_n_N} respectively, $\psi$ is defined in \eqref{eq:lyapunov}, and $(\Theta_{\tilde{k}})_{\tilde{k} \in\nset} = (\theta_{\tilde{k}},\zeta_{\tilde{k}},\{\eta^i_{\tilde{k}}\}_{i \in [b]})_{\tilde{k} \in\nset}$ is defined in \eqref{eq:recursion_theta_QLSDpp}, \eqref{eq:QLSDpp_zeta}, \eqref{eq:recursion_eta_QLSDpp}.

\end{proposition}
\begin{proof} 
Let $k \in \mathbb{N}$ and write $k= sl + r$ with $s \in\nset$, $r \in\{0,\ldots,l-1\}$
Then, using \Cref{lemma:varianceQLSDpp}, we have
	\begin{align}
		\nonumber
		&\mathbb{E}\br{\norm{\tilde{G}(\Theta_{sl+r};X_{sl+r+1}) - \nabla U(\theta_{k})}^2} \\
		\nonumber
		&\le \br{2\sum_{i=1}^b \frac{\mathtt{M}_i^2}{p_i}(\omega_i+1-p_i) + \pr{\frac{\omega_i+1}{p_i}}A_{n_i,N_i}\barM \mathtt{M}_i}\mathbb{E}\br{\norm{\theta_{k} - \theta^\star}^2} \\ 
		&+ \br{2\sum_{i=1}^b(\omega_i+1-p_i)/p_i} \mathbb{E}\br{\norm{\nabla U_{i}(\theta^\star) - \eta^{(i)}_{k}}^2}
		+ 2 \barM \sum_{i=1}^b\br{\pr{\frac{\omega_i+1}{p_i}}A_{n_i,N_i}\mathtt{M}_i}\mathbb{E}\br{\norm{\theta_{sl} - \theta^\star}^2}\eqsp.\label{eq:variance1}
	\end{align}
	We now use our previous results to upper bound the three expectations at the right-hand side of \eqref{eq:variance1}.
	First, using \Cref{coro:lyapunov} and a straightforward induction gives
	\begin{align}
		\nonumber
		\mathbb{E}\br{\norm{\theta_{sl} - \theta^\star}^2} &\le (1-\gamma\mathtt{m}/2)^s\psi(\theta_{0},\eta_{0}) +2 \gamma l d\sum_{j=0}^{s-1}(1-\gamma\mathtt{m}/2)^j \\
		&\le (1-\gamma\mathtt{m}/2)^s\psi(\theta_{0},\eta_{0}) +4 l d / \mtt\eqsp.\label{eq:propQLSDpp1}
	\end{align}
	Similarly, we have
	\begin{align}
		\nonumber
		\mathbb{E}\br{\norm{\theta_{k} - \theta^\star}^2} &\le (1-\gamma\mathtt{m}/2)^{s+1}\psi(\theta_{0},\eta_{0}) + 2 \gamma ld\sum_{j=0}^{s}(1-\gamma\mathtt{m}/2)^j \\
		&\le (1-\gamma\mathtt{m}/2)^{s}\psi(\theta_{0},\eta_{0}) + 4 ld/\mtt\eqsp.\label{eq:propQLSDpp2}
	\end{align}
	Finally, using \Cref{lemma:QLSDpp_thetastar_eta} combined with \eqref{eq:propQLSDpp1} and \eqref{eq:propQLSDpp2}, we obtain
	\begin{align*}
		\sum_{i=1}^b \mathbb{E}\br{\norm{\nabla U_{i}(\theta^\star) - \eta^{(i)}_{k}}^2} &\le (1-\alpha)\sum_{i=1}^b \mathbb{E}\br{\norm{\nabla U_{i}(\theta^\star) - \eta^{(i)}_{k-1}}^2} \\
		&\qquad \qquad+ 2\alpha C_{\mathbf{n},\mathbf{N}} (1-\gamma \mathtt{m}/2)^s\psi(\theta_0,\eta_0) + 8 ld \alpha C_{\mathbf{n},\mathbf{N}}/\mtt \eqsp.
	\end{align*}
	Then, a straightforward induction leads to
	\begin{align}
		\nonumber
          \sum_{i=1}^b \mathbb{E}\br{\norm{\nabla U_{i}(\theta^\star) - \eta^{(i)}_{k}}^2} &\le (1-\alpha)^k\sum_{i=1}^b \norm{\nabla U_{i}(\theta^\star) - \eta^{(i)}_{0}}^2 \\
          &\qquad \qquad+ 2 C_{\mathbf{n},\mathbf{N}} (1-\gamma \mathtt{m}/2)^s\psi(\theta_0,\eta_0) + 8 ld  C_{\mathbf{n},\mathbf{N}}/\mtt \eqsp.\label{eq:propQLSDpp3}
	\end{align}
Combining \eqref{eq:propQLSDpp1}, \eqref{eq:propQLSDpp2} and \eqref{eq:propQLSDpp3} in \eqref{eq:variance1} concludes the proof.
\end{proof}

\subsection{Proof of \Cref{theorem_QLSDpp}}
\label{sec:quant-bounds-textttq}

Note that $\gamma \in\ocint{0,\bgamma}$, $\alpha \in\ocint{0,\balpha}$ and $l \in\nsets$,          $(\Theta_{\tilde{k}})_{\tilde{k} \in\nset} = (\theta_{\tilde{k}},\zeta_{\tilde{k}},\{\eta^{(i)}_{\tilde{k}}\}_{i \in [b]})_{\tilde{k} \in\nset}$  defined in \eqref{eq:recursion_theta_QLSDpp}, \eqref{eq:QLSDpp_zeta}, \eqref{eq:recursion_eta_QLSDpp}
is a inhomogeneous Markov chain associated with the sequence of Markov kernel $(Q_{\gamma,\alpha,l}^{(k)})_{k\in\nset}$ defined by as follows.
Define for any $(\theta,\zeta,\eta) \in \rset^d \times \rset^d \times \rset^{d}$, and $x^{(1)} \in\wp_{N_i,n_i}$, $x^{(2)} \in \msx_2$ and $x^{(3)} \in \msx_3$,
\begin{align*}
  \mathscr{F}_i ((\theta,\zeta,\eta);(x^{(1)},x^{(2)},x^{(3)})) &=   \mathscr{S}_i\pr{\mathscr{C}_i\left\{G_{i}\pr{\theta,\zeta;x^{(1)}} - \eta;x^{(2)}\right\};x^{(3)}}\\
  \mathscr{G}_i ((\theta,\zeta,\eta);(x^{(1)},x^{(2)},x^{(3)})) &=   \eta + \alpha \mathscr{F}_i ((\theta,\zeta,\eta);(x^{(1)},x^{(2)},x^{(3)}))\eqsp. 
\end{align*}
and for $\ttheta \in\rset^d$, $\{\eta^{(i)}\}_{i=1}^b \in \rset^{db}$, $\{x^{(1,i)}\}_{i=1}^b \in \otimes_{i=1}^b\wp_{N_i,n_i}$, $\{x^{(2,i)}\}_{i=1}^b \in \msx_2^b$, $\{x^{(3,i)}\}_{i=1}^b \in \msx_3^b$,  setting $x^{(1:b)} = \{(x^{(1,i)},x^{(2,i)},x^{(3,i)})\}_{i=1}^b$,
\begin{equation*}
	\varphi_{\gamma}((\ttheta,\theta,\zeta,\{\eta^{(i)}\}_{i=1}^b);x^{(1:b)})
	= (4\uppi\gamma)^{-d/2}\exp\pr{-\normn{\btheta-\theta+\gamma\sum_{i=1}^{b}\mathscr{F}_i ((\theta,\zeta,\eta^{(i)});x^{(i)})}^2/(4\gamma)} \eqsp. 
\end{equation*}
Denote $\tilde{\msx}^{(i)} = \wp_{N_i,n_i} \times \msx_2 \times \msx_3$ and $\tilde{\nu}^{(i)} = \nu_1^{(i)} \times \nu_2 \times \nu_3$. 
Set $Q_{\gamma,\alpha,l}^{(0)} = \Id$ and for $k \geq 0$, $k=ls +r$, $s \in \nset$,  $r\in\{0,\ldots,l-1\}$, $(\theta,\zeta,\eta) \in \rset^d \times \rset^d \times \rset^{db}$ and $\msa \in \mcb(\rset^d \times \rset^d \times \rset^{db})$, 
\begin{equation*}
  \label{eq:def_q_k_l_gamma_alpha_l}
  \begin{aligned}
		& \text{ if $r=0$} \\
		& Q_{\gamma,\alpha,l}^{(k+1)}((\theta,\zeta,\eta) ,\msa) =\\
		& \txts\int_{\otimes_{i=1}^b\tilde{\msx}^{(i)}} \mathbf{1}_{\msa}(\ttheta,\tilde{\zeta},\tilde{\eta})   \varphi_{\gamma}((\ttheta,\theta,\zeta,\{\eta^{(i)}\}_{i=1}^b);x^{(1:b)})\{\prod_{i=1}^b \updelta_{  \mathscr{G}_i ((\theta,\zeta,\eta);x^{(i)})} (\rmd \tilde{\eta}^{(i)}) \}\updelta_{\theta}(\rmd \tilde{\zeta})  \,\rmd \tilde{\theta} \, \otimes_{i=1}^b\tilde{\nu}^{(i)}\parentheseLigne{\rmd x^{\parentheseLigne{i}}} \\
		& \text{otherwise} \\
		& Q_{\gamma,\alpha,l}^{(k+1)}((\theta,\zeta,\eta) ,\msa) =\\
		& \txts \int_{\otimes_{i=1}^b\tilde{\msx}^{(i)} } \mathbf{1}_{\msa}(\ttheta,\tilde{\zeta},\tilde{\eta})   \varphi_{\gamma}((\ttheta,\theta,\zeta,\{\eta^{(i)}\}_{i=1}^b);x^{(1:b)})\{\prod_{i=1}^b \updelta_{  \mathscr{G}_i ((\theta,\zeta,\eta);x^{(i)})} (\rmd \tilde{\eta}^{(i)}) \}\updelta_{\zeta}(\rmd \tilde{\zeta})  \,\rmd \tilde{\theta} \, \otimes_{i=1}^b\tilde{\nu}^{(i)}\parentheseLigne{\rmd x^{\parentheseLigne{i}}}\eqsp.
	\end{aligned}
\end{equation*}
Consider then, the Markov kernel on $\rset^d \times \mcb(\rset^d)$,
\begin{equation}
  \label{eq:def_r_k_l_gamma_alpha_l}
  R_{\gamma,\alpha,l,\eta_0}^{(k)}(\theta_0,\msa) = Q_{\gamma,\alpha,l}^{(k)}((\theta_0,\theta_0,\eta_0),\msa \times \rset^d \times \rset^{db}) \eqsp.
\end{equation}

Define
\begin{equation}
	\label{eq:gamma_alpha_QLSDpp}
\bgamma_{\alpha} = \bgamma_{\alpha,2}  \wedge \bgamma_4 \eqsp, \qquad \bgamma_4 = 1/(10 \mtt)\eqsp,
\end{equation}
where $\bgamma_{\alpha,2}$ is defined in \eqref{eq:def_bgamma_2}.
The following theorem provides a non-asymptotic convergence bound for the \texttt{QLSD}$^{++}$ kernel.
\begin{theorem}\label{theorem_QLSDpp_supp}
  Assume \Cref{ass:potential_U}, \Cref{ass:compression}, \Cref{ass:A_k_supp} and \Cref{ass:potential_fij}. 
and let $l \in \N^*$.
Then, for any probability measure $\mu\in\mathcal{P}_{2}\parentheseLigne{\Rd
}$, $\eta_0 \in \rset^{db}$, $\alpha \in \ocint{0,1/(1+\max_{i\in[b]}\omega_i)}$, $\gamma \in \ocint{0,\bgamma_{\alpha}}$, and $k = sl+r \in \N$ with $s \in\nset$, $r\in\{0,\ldots,l-1\}$, we have
  \begin{align*}
    \wass^{2}\parentheseLigne{\mu   R_{\gamma,\alpha,l,\eta_0}^{(k)},\pi}
    &\le (1-\gamma \mtt/2)^k\wass^{2}\parentheseLigne{\mu,\pi} + (2\gamma/\mathtt{m})(1-\gamma\mathtt{m}/2)^sD_{\mathbf{n},\mathbf{N}}\int_{\rset^d}\psi(\theta_0,\eta_0) \rmd \mu(\theta_0) \\
    &\qquad \qquad+ (4\gamma/\mathtt{m})\br{\sum_{i=1}^b(\omega_i+1-p_i)/p_i}(1-\alpha)^k\sum_{i=1}^b\norm{\nabla U_i(\theta^\star) - \eta^{(i)}_0}^2 + \gamma\mathrm{B}_{\oplus,\bgamma_{\alpha}}\eqsp,
  \end{align*}
  where $ R_{\gamma,\alpha,l,\eta_0}^{(k)}$ is defined in \eqref{eq:def_r_k_l_gamma_alpha_l}, $\psi$ is defined in \eqref{eq:lyapunov}, $D_{n,N}$ in \eqref{eq:D_n_N} 
  and
     \begin{align}
    	\label{eq:def:D_gamma_QLSDpp}
    	\mathrm{B}_{\oplus,\bgamma_{\alpha}} 
    	&= 2d\lip^2\parentheseLigne{1/\mU+5\bgamma_{\alpha}}\br{1+\bgamma_{\alpha}\lip^2/(2\mU)+\bgamma_{\alpha}^{2}\lip^{2}/12}/\mtt + 96ld\pr{\sum_{i=1}^b \Mtt_i(\omega_i+1)(\Mtt_i + \barM A_{n_i,N_i})/p_i}/\mtt^2\eqsp.
    \end{align}
\end{theorem}

\begin{proof}

	Let $k \in\nset$.
	The proof follows from the same lines as \Cref{cor:bound:sec_continuous_vs_cv}.
	By \eqref{eq:def:sde_prop} and \eqref{eq:recursion_theta_QLSDpp}, we have
	\begin{multline*}
	\vartheta_{\gamma(k+1)}-\theta_{k+1}=\vartheta_{\gamma k}-\theta_{k}-\gamma\br{\nabla\U\parentheseLigne{\vartheta_{\gamma k}}-\nabla\U\parentheseLigne{\theta_{k}}}\\
	-\int_{0}^{\gamma}\br{\nabla\U\parentheseLigne{\vartheta_{\gamma k+s}}-\nabla\U\parentheseLigne{\vartheta_{\gamma k}}}\,\rmd s
	+\gamma\br{\tilde{G}(\Theta_{k};X_{k+1})-\nabla\U\parentheseLigne{\theta_{k}}}\eqsp.
	\end{multline*}
	Define the filtration $\parentheseLigne{\mathcal{H}_{\tilde{k}}}_{\tilde{k} \in \mathbb{N}}$ as $\mathcal{H}_{0}=\sigma\parentheseLigne{\vartheta_{0},\Theta_{0}}$ and for $\tilde{k} \in \mathbb{N}^{*}$,
	\begin{align*}
	  \mathcal{H}_{\tilde{k}} &= \sigma\parentheseLigne{\vartheta_{0},\Theta_{0},\parentheseLigne{X_{l}^{(1)} , \ldots, X_{l}^{(b)}}_{1\le l\le \tilde{k}},\parentheseLigne{B_{t}}_{0\le t\le \gamma \tilde{k}}}\eqsp.
	\end{align*}

Note that since $\sequencet{\vartheta}[t][0]$ is a strong solution of \eqref{eq:def:sde_prop}, then is easy to see that $(\vartheta_{\gamma \tilde{k}}, \Theta_{\tilde{k}})_{\tilde{k}\in\nset}$ is $(\mathcal{H}_{\tilde{k}})_{\tilde{k}\in\nset}$-adapted. 
Taking the squared norm and the conditional expectation with respect to $\mathcal{H}_{k}$, we obtain using \Cref{ass:tgrad:sharp}-\ref{ass:tgrad:sharp:2} that
\begin{align}
	&\nonumber\E^{\mathcal{H}_{k}}\br{\norm{\vartheta_{\gamma(k+1)}-\theta_{k+1}}^{2}}
    =\norm{\vartheta_{\gamma k}-\theta_{k}}^{2}
    -2\gamma\ps{\vartheta_{\gamma k}-\theta_{k}}{\nabla\U\parentheseLigne{\vartheta_{\gamma k}}-\nabla\U\parentheseLigne{\theta_{k}}}\\
    &\qquad\qquad\qquad\nonumber+2\gamma\int_{0}^{\gamma}\ps{\nabla\U\parentheseLigne{\vartheta_{\gamma k}}-\nabla\U\parentheseLigne{\theta_{k}}}{\E^{\mathcal{H}_{k}}\br{\nabla\U\parentheseLigne{\vartheta_{\gamma k+u}}-\nabla\U\parentheseLigne{\vartheta_{\gamma k}}}}\,\rmd u\\
    &\qquad\qquad\qquad\nonumber-2\int_{0}^{\gamma}\ps{\vartheta_{\gamma k}-\theta_{k}}{\E^{\mathcal{H}_{k}}\br{\nabla\U\parentheseLigne{\vartheta_{\gamma k+u}}-\nabla\U\parentheseLigne{\vartheta_{\gamma k}}}}\,\rmd u\\
    &\qquad\qquad\qquad\nonumber+\gamma^{2}\norm{\nabla\U\parentheseLigne{\vartheta_{\gamma k}}-\nabla\U\parentheseLigne{\theta_{k}}}^{2}\\
    &\qquad\qquad\qquad\nonumber+\E^{\mathcal{H}_{k}}\br{\norm{\int_{0}^{\gamma}\br{\nabla\U\parentheseLigne{\vartheta_{\gamma k+u}}-\nabla\U\parentheseLigne{\vartheta_{\gamma k}}}\,\rmd u}^{2}}\\
    &\qquad\qquad\qquad+\gamma^{2}\E^{\mathcal{H}_{k}}\br{\norm{\tilde{G}(\Theta_{k};X_{k+1})-\nabla\U\parentheseLigne{\theta_{k}}}^{2}}\eqsp.\label{eq:proof_thm_QLSDpp_1}
\end{align}

Using \Cref{proposition:varianceQLSDpp_final}, we obtain
\begin{align}
	\nonumber
	\mathbb{E}\br{\norm{\tilde{G}(\Theta_{k};X_{k+1}) - \nabla U(\theta_{k})}^2} &\le (1-\gamma \mathtt{m}/2)^{\floor{k/l}}D_{\mathbf{n},\mathbf{N}}\psi(\theta_0,\eta_0) + 4ldD_{\mathbf{n},\mathbf{N}}/\mtt\\
	&\qquad \qquad+ \br{2\sum_{i=1}^b(\omega_i+1-p_i)/p_i}(1-\alpha)^k\sum_{i=1}^b \mathbb{E}\br{\norm{\nabla U_{i}(\theta^\star) - \eta^{(i)}_{0}}^2}\eqsp.
	\label{eq:proof_thm_QLSDpp_2}
	\end{align}
Then, we control the remaining terms in \eqref{eq:proof_thm_QLSDpp_1} using \eqref{eq:bound:ps_un_pp_cv}, \eqref{eq:bound:ps_deux_pp_cv} and \eqref{eq:bound:ps_{t}rois_pp_cv}.
Combining these bounds and \eqref{eq:proof_thm_QLSDpp_2} into \eqref{eq:proof_thm_QLSDpp_1}, for any $\varepsilon >0$, yields

\begin{align*}
	\nonumber
    \E\br{\norm{\vartheta_{\gamma(k+1)}-\theta_{k+1}}^{2}}
    &\le\parentheseLigne{1+2\gamma\varepsilon-5\gamma^2\mU\lip}\E\br{\norm{\vartheta_{\gamma k}-\theta_{k}}^{2}}\\
    \nonumber
	&-\gamma\br{2-5\gamma\parentheseLigne{\mU+\lip}}\E\br{\ps{\vartheta_{\gamma k}-\theta_{k}}{\nabla\U\parentheseLigne{\vartheta_{\gamma k}}-\nabla\U\parentheseLigne{\theta_{k}}}}\\
    \nonumber
	&+\parentheseLigne{5\gamma+(2\varepsilon)^{-1}}\int_{0}^{\gamma}\E\br{\norm{\nabla\U\parentheseLigne{\vartheta_{\gamma k+u}}-\nabla\U\parentheseLigne{\vartheta_{\gamma k}}}^{2}}\,\rmd u\\
    \nonumber
	&+\gamma^2(1-\gamma \mathtt{m}/2)^{\floor{k/l}}D_{\mathbf{n},\mathbf{N}}\E\br{\psi(\theta_0,\eta_0)} + 4ldD_{\mathbf{n},\mathbf{N}}/\mtt\\
		&+ 2\gamma^2\br{\sum_{i=1}^b(\omega_i+1-p_i)/p_i}(1-\alpha)^k\sum_{i=1}^b \norm{\nabla U_{i}(\theta^\star) - \eta^{(i)}_{0}}^2\eqsp.
\end{align*}
Next, we use that under \Cref{ass:potential_U},  $\psLigne{\vartheta_{\gamma k}-\theta_{k}}{\nabla\U\parentheseLigne{\vartheta_{\gamma k}}-\nabla\U\parentheseLigne{\theta_{k}}}\ge \mU \normLigne{\vartheta_{\gamma k}-\theta_{k}}^{2}$ and $\absLigne{\psLigne{\theta_{k}-\theta^{\star}}{\nabla\U(\theta_{k})-\nabla\U(\theta^{\star})}}\le \lip\normLigne{\theta_{k}-\theta^{\star}}^2$,
which implies taking $\varepsilon = \mtt /2$ and since $2 - 5\gamma(\mtt + \Ltt) \geq 0$,
\begin{align}
	\nonumber
	\E\br{\norm{\vartheta_{\gamma(k+1)}-\theta_{k+1}}^{2}}
	&\le\parentheseLigne{1-\gamma\mU(1-5\gamma\mU)}\E\br{\norm{\vartheta_{\gamma k}-\theta_{k}}^{2}}\\
	\nonumber
	&+\parentheseLigne{5\gamma+\mtt^{-1}}\int_{0}^{\gamma}\E\br{\norm{\nabla\U\parentheseLigne{\vartheta_{\gamma k+u}}-\nabla\U\parentheseLigne{\vartheta_{\gamma k}}}^{2}}\,\rmd u\\
	\nonumber
	&+\gamma^2(1-\gamma \mathtt{m}/2)^{\floor{k/l}}D_{\mathbf{n},\mathbf{N}}\E\br{\psi(\theta_0,\eta_0)} + 4ldD_{\mathbf{n},\mathbf{N}}/\mtt\\
	&+ 2\gamma^2\br{\sum_{i=1}^b(\omega_i+1-p_i)/p_i}(1-\alpha)^k\sum_{i=1}^b \norm{\nabla U_{i}(\theta^\star) - \eta^{(i)}_{0}}^2\eqsp.\label{proof_thm_QLSDpp_4}
\end{align}

Further, for any $u\in\R_+$, using \citet[Lemma 21]{durmus2018high} we have 
\begin{equation*}
\lip^{-2}\,\E\br{\norm{\nabla\U\parentheseLigne{\vartheta_{\gamma k+u}}-\nabla\U\parentheseLigne{\vartheta_{\gamma k}}}^{2}}
\le d u\pr{2+u^{2}\lip^{2}/3}+3u^{2}\lip^{2}/2\E\br{\norm{\vartheta_{\gamma k}-\theta^{\star}}^{2}}\eqsp.
\end{equation*}
Integrating the previous inequality on $\br{0,\gamma}$, we obtain
\begin{equation*}
\lip^{-2}\int_{0}^{\gamma}\E\br{\norm{\nabla\U\parentheseLigne{\vartheta_{\gamma k+u}}-\nabla\U\parentheseLigne{\vartheta_{\gamma k}}}^{2}}\,\rmd u
\le d\gamma^2+d\gamma^4\lip^{2}/12+\gamma^{3}\lip^{2}/2\E\br{\norm{\vartheta_{\gamma k}-\theta^{\star}}^{2}}\eqsp.
\end{equation*}
Plugging this bounds in \eqref{proof_thm_QLSDpp_4} and using \citet[Proposition~1]{durmus2018high} complete the proof.

\end{proof}



\section{CONSISTENCY ANALYSIS IN THE BIG DATA REGIME}

In this section, we assume that the number of observations on each client $i \in [b]$ writes $N_i = \floor{c_i N}$ where $\{c_i > 0\}_{i \in [b]}$, $N \in \mathbb{N}^*$, and provide upper bounds on the asymptotic bias associated to each algorithm when $N$ tends towards infinity. 
For simplicity, we assume for any $i \in [b]$, that $n_i = \floor{c_i n}$ with $n \in [N]$, $\Mtt_i = \Mtt$ with $\Mtt > 0$, $p_i = 1$ and $\omega_i = \omega$ with $\omega > 0$ but note that our conclusions also hold for the general setting considered in this paper.

\subsection{Asymptotic analysis for \Cref{algo:QLSD}}

The following corollary is associated with \texttt{QLSD} defined in \Cref{algo:QLSD} in the main paper. 

\begin{corollary}\label{cor:bound:qlsd_bias}
  Assume \Cref{ass:potential_U}, \Cref{ass:compression} \Cref{ass:stochastic_gradient} and \Cref{ass:A_k_supp}.
  In addition, assume that $\liminf_{N \rightarrow \infty} \mathtt{m}/N > 0$ and $\limsup_{N \rightarrow \infty} \mathtt{A}/N <\infty$ for $\mathtt{A} \in \{\lip,\Mtt,\Btt^\star,\sigma_\star\}$.
  Then, we have $\bgamma = \bar{\eta}/N$ where $\bar{\eta} >0$ and $\bar{\gamma}$ is defined in \eqref{eq:bar_gamma}.
  In addition, 
  \begin{equation*}
    B_{\bar{\gamma}} = (\omega+1)\Oh(N)\eqsp,
  \end{equation*}
  where $B_{\bar{\gamma}}$ is defined in \eqref{eq:def:E_gamma}.
\end{corollary}
\begin{proof}
	Since we assume that $\liminf_{N \rightarrow \infty} \mathtt{m}/N > 0$ and $\limsup_{N \rightarrow \infty} \mathtt{A}/N <\infty$ for $\mathtt{A} \in \{\lip,\Mtt,\Btt^\star,\sigma_\star\}$, there exist $C_{\mtt}$, $C_{\lip}$, $C_{\Mtt}$, $C_{\Btt^\star}$ and $C_{\sigma_{\star}} > 0$ such that $\mtt \geq C_{\mtt} N$, $\lip \le C_{\lip} N$, $\Mtt \le C_{\Mtt} N$, $\Btt^\star \le C_{\Btt^{\star}} N$ and $\sigma_\star \le C_{\sigma_{\star}} N$. 
  	Under these assumptions, it is straightforward from \eqref{eq:bar_gamma} to see that there exists $\bar{\eta} > 0$ such that $\bar{\gamma} = \bar{\eta}/N$. In addition, it follows from \eqref{eq:def:E_gamma} that		
  \begin{equation*}
      B_{\bar{\gamma}} \le \frac{2dC_{\lip}^2}{C_{\mtt}}\pr{\frac{1}{C_{\mtt}}+5\bar{\eta}}\br{1+\frac{\bar{\eta}C_{\lip}^2}{2C_{\mtt}}+\frac{\bar{\eta}^{2} C_{\lip}^2}{12}}
      +\frac{4}{C_{\mtt}}\parentheseLigne{\omega C_{\Bs}+C_{\sigmas}^{2}N}
      + \frac{8\parentheseLigne{\omega+1}C_{\lip}C_{\MH}}{C_{\mU}^2} \br{d+\bar{\eta}\pr{\omega C_{\Bs}+C_{\sigmas}^{2}N}}\eqsp.
  \end{equation*}
  The proof is concluded by letting $N$ tend towards infinity.
\end{proof}

Regarding the specific instance \texttt{QLSD}$^{\#}$ of \Cref{algo:QLSD} in the main paper, a similar result holds.
Indeed, by using \Cref{lem:var_grad_sto}, we can notice that \Cref{ass:stochastic_gradient}-\ref{ass:stochastic_gradient:3} is verified with $\sigmas = C_{\sigmas} N$ for some $C_{\sigmas} > 0$ and we can apply \Cref{cor:bound:qlsd_bias}.

\subsection{Asymptotic analysis for \Cref{algo:QLSD-VR}}

The following corollary is associated with \texttt{QLSD}$^\star$ defined in \Cref{algo:QLSD-VR} in the main paper. 

\begin{corollary}\label{cor:bound:qlsds_bias}
  Assume \Cref{ass:potential_U}, \Cref{ass:compression}, \Cref{ass:A_k_supp} and  \Cref{ass:potential_fij}.
  In addition, assume that $\liminf_{N \rightarrow \infty} \mathtt{m}/N > 0$ and $\limsup_{N \rightarrow \infty} \mathtt{A}/N <\infty$ for $\mathtt{A} \in \{\lip,\Mtt\}$.
  Then, we have $\bgamma = \bar{\eta}/N$ where $\bar{\eta} >0$ and $\bar{\gamma}$ is defined in \eqref{eq:bar_gamma}.
  In addition, 
  \begin{equation*}
    B_{\ostar,\bar{\gamma}} = d(\omega+1)\Oh(1)\eqsp,
  \end{equation*}
  where $B_{\ostar,\bgamma}$ is defined in \eqref{B_gamma_QLSDs}.
\end{corollary}
\begin{proof}
  Since we assume that $\liminf_{N \rightarrow \infty} \mathtt{m}/N > 0$ and $\limsup_{N \rightarrow \infty} \mathtt{A}/N <\infty$ for $\mathtt{A} \in \{\lip,\Mtt\}$, there exist $C_{\mtt}$, $C_{\lip}$ and $C_{\Mtt} > 0$ such that $\mtt \geq C_{\mtt} N$, $\lip \le C_{\lip} N$ and $\Mtt \le C_{\Mtt} N$. 
  Under these assumptions, it is straightforward from \eqref{eq:bar_gamma} to see that there exists $\bar{\eta} > 0$ such that $\bar{\gamma}_{\alpha} = \bar{\eta}/N$. In addition, it follows from \eqref{lemma:gammacondition_QLSDpp} that
  \begin{equation*}
      B_{\ostar,\bar{\gamma}}
      \le \frac{2dC_{\lip}^2}{C_{\mtt}}\pr{\frac{1}{C_{\mtt}}+5\bar{\eta}}\br{1+\frac{\bar{\eta}C_{\lip}^2}{2C_{\mtt}}+\frac{\bar{\eta}^{2} C_{\lip}^2}{12}}
      + \frac{4 d \bMH C_{\lip}}{C_{\mtt}^2} \max_{i\in [b]}\ac{c_i\omega+(\omega+1)\cdot \frac{N-n}{n(\floor{c_i N}-1)}}\eqsp.
  \end{equation*}
  The proof is concluded by letting $N$ tend towards infinity.
\end{proof}

Lastly, we have the following asymptotic convergence result regarding \texttt{QLSD}$^{++}$ defined in \Cref{algo:QLSD-VR} in the main paper. 

\begin{corollary}\label{cor:bound:qlsds_bias}
  Assume \Cref{ass:potential_U}, \Cref{ass:compression}, \Cref{ass:A_k_supp} and  \Cref{ass:potential_fij}.
  In addition, assume that $\liminf_{N \rightarrow \infty} \mathtt{m}/N > 0$ and $\limsup_{N \rightarrow \infty} \mathtt{A}/N <\infty$ for $\mathtt{A} \in \{\lip,\Mtt\}$.
  Then, we have $\bgamma_{\alpha} = \bar{\eta}/N$ where $\bar{\eta} >0$ and $\bar{\gamma}_{\alpha}$ is defined in \eqref{eq:gamma_alpha_QLSDpp}.
  In addition, 
  \begin{equation*}
    B_{\oplus,\bar{\gamma}_{\alpha}} = d(\omega+1)\Oh(1)\eqsp,
  \end{equation*}
  where $B_{\oplus,\bgamma_{\alpha}}$ is defined in \eqref{eq:def:D_gamma_QLSDpp}.
\end{corollary}
\begin{proof}
  Since we assume that $\liminf_{N \rightarrow \infty} \mathtt{m}/N > 0$ and $\limsup_{N \rightarrow \infty} \mathtt{A}/N <\infty$ for $\mathtt{A} \in \{\lip,\Mtt\}$, there exist $C_{\mtt}$, $C_{\lip}$ and $C_{\Mtt} > 0$ such that $\mtt \geq C_{\mtt} N$, $\lip \le C_{\lip} N$ and $\Mtt \le C_{\Mtt} N$. 
  Under these assumptions, it is straightforward from \eqref{eq:gamma_alpha_QLSDpp} to see that there exists $\bar{\eta} > 0$ such that $\bar{\gamma}_{\alpha} = \bar{\eta}/N$. In addition, it follows from \eqref{eq:def:D_gamma_QLSDpp} that
  \begin{equation*}
    	\mathrm{B}_{\oplus,\bgamma_{\alpha}} 
    	\le \frac{2d C_{\lip}^2}{C_{\mtt}}\pr{\frac{1}{C_{\mtt}}+5\bar{\eta}}\br{1+\frac{\bar{\eta}C_{\lip}^2}{2C_{\mtt}}+\frac{\bar{\eta}^{2} C_{\lip}^2}{12}}
    	+ \frac{96(\omega+1)ldbC_{\mathtt{M}}}{C_{\mtt}^2} \pr{\frac{(N-n)\barM}{n(\min_{i\in[b]}\{\floor{c_iN}\}-1)} +C_{\mathtt{M}}}\eqsp.
    \end{equation*}
  The proof is concluded by letting $N$ tend towards infinity.
\end{proof}

\section{EXPERIMENTAL DETAILS}
\label{sec:expresults}

In this section, we provide additional details regarding our numerical experiments.
The code, data and instructions to reproduce our experimental results can be found in the supplementary material.

\subsection{Toy Gaussian example}

\noindent\textbf{Pseudocode of \texttt{LSD}$^\star$.} For completeness, we provide in \Cref{algo:LSD-star} the pseudocode of the non-compressed counterpart of \texttt{QLSD}$^\star$, namely \texttt{LSD}$^\star$.

\begin{algorithm}[h]
   \caption{Variance-reduced Langevin Stochastic Dynamics (\texttt{LSD}$^{\star}$)}
   \label{algo:LSD-star}
  \begin{algorithmic}
     \State {\bfseries Input:} minibatch sizes $\{n_i\}_{i \in [b]}$, number of iterations $K$, step-size $\gamma \in (0,\bar{\gamma}]$ with $\bar{\gamma} > 0$ and initial point $\theta_0$.
     \For{$k=0$ {\bfseries to} $K-1$}
     \For{$i \in \mathcal{A}_{k+1}$ \Comment{On active clients}}
        \State Draw $\mathcal{S}_{k+1}^{(i)} \sim \mathrm{Uniform}\pr{\wp_{N_i,n_i}}$.
        \State {\small Set $H_{k+1}^{(i)}(\theta_k) = (N_i/n_i)\sum_{j \in \mathcal{S}_{k+1}^{(i)}} [\nabla U_{i,j}(\theta_k) - \nabla U_{i,j}(\thetas)]$.}
        \State Compute $\textsl{g}_{i,k+1} = H_{k+1}^{(i)}(\theta_k)$.
        \State Send $\textsl{g}_{i,k+1}$ to the central server.
     \EndFor
     \State \Comment{On the central server}
     \State Compute $\textsl{g}_{k+1} = \frac{b}{|\mathcal{A}_{k+1}|}\sum_{i \in \mathcal{A}_{k+1}}\textsl{g}_{i,k+1}$.
     \State Draw $Z_{k+1} \sim \mathrm{N}(0_d,\mathrm{I}_d)$.
     \State Compute $\theta_{k+1} = \theta_k - \gamma \textsl{g}_{k+1} + \sqrt{2\gamma}Z_{k+1}$.
     \State Send $\theta_{k+1}$ to the $b$ clients.
     \EndFor
     \State {\bfseries Output:} samples $\{\theta_k\}_{k=0}^{K}$.
  \end{algorithmic}
\end{algorithm}

\noindent\textbf{Additional experimental details.}
As highlighted in Section 4 (\emph{Toy Gaussian example} paragraph) in the main paper, the synthetic dataset has been generated so that each client owns a heterogeneous and unbalanced dataset. An illustration of the unbalancedness is given in \Cref{fig:toy_example_supp}. The precise procedure to generate such a dataset can be found in the aforementioned notebook.

\begin{figure}[h]
  \begin{center}
    \mbox{{\includegraphics[scale=0.55]{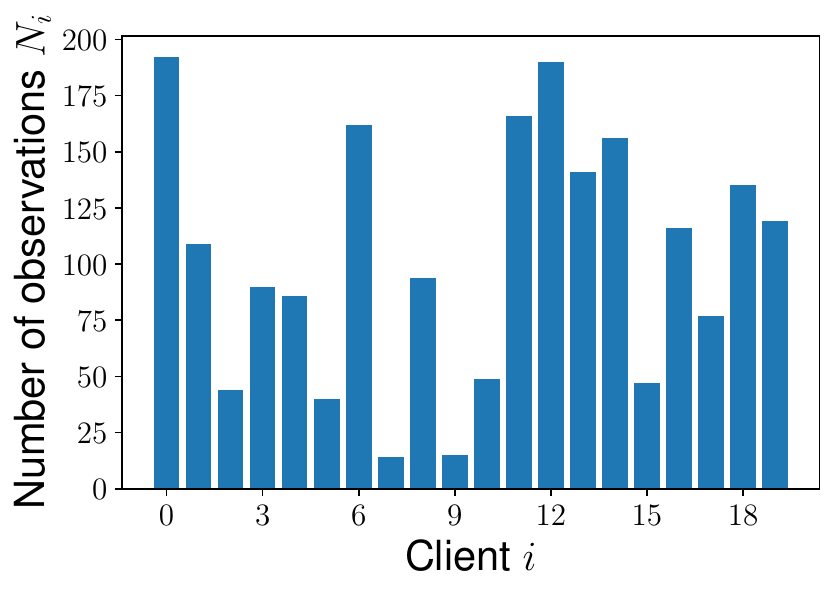}}}
  \end{center}
  \caption{Illustration of the unbalancedness of the synthetic dataset used in the Toy Gaussian experiment.\label{fig:toy_example_supp}}
\end{figure}

To obtain the figure at the bottom row of Figure 1 in the main paper, we launched all the MCMC algorithms with $K = 500,000$ outer iterations and considered a burn-in period of $450,000$ iterations. 
Hence, only the last $50,000$ samples have been used to compute the MSE associated to the test function $f:\theta \mapsto \norm{\theta}$.
In order to compute the expected number of bits transmitted during each upload period, we considered the \texttt{Elias} encoding scheme and used the upper-bounds given in \citet[Theorem 3.2 and Lemma A.2]{alistarh2017qsgd}.

\begin{itemize}
	\item \textbf{License of the assets:} No existing asset has been used for this experiment.
	\item \textbf{Total amount of compute and type of resources used:} This experiment has been run on a laptop running Windows 10 and equipped with Intel(R) Core(TM) i7\_8565U CPU 1.80GHz with 16Go of RAM.
	The total amount of compute is roughly 33 hours. 
	\item \textbf{Training details:} All training details (here hyperparameters) are detailed in Section 4 in the main paper. 
\end{itemize}

\noindent\textbf{Discretisation step-size and compression trade-off.} We complement the analysis made in the main paper by showing on \Cref{fig:toy_example_supp2} that the saving in terms of number of transmitted bits can be further improved by decreasing the value of $\gamma$.
This numerical finding illustrates our theory which in particular shows that the asymptotic bias associated to \qlsds~is of the order $\omega\Oh(\gamma)$, see \Cref{table:overview} in the main paper.

\begin{figure}
  \begin{center}
    \mbox{{\includegraphics[scale=0.6]{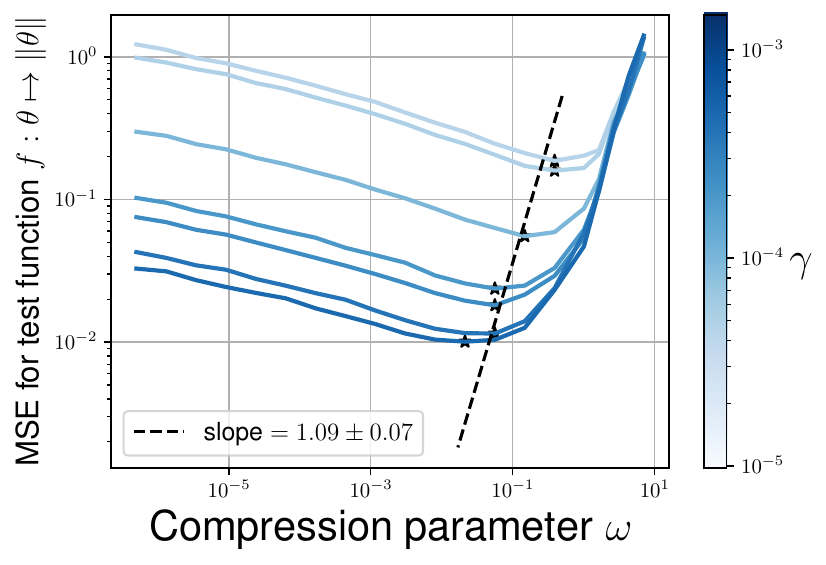}}}
  \end{center}
  \caption{Toy Gaussian example. Trade-off between step-size and compression parameter values.\label{fig:toy_example_supp2}}
\end{figure}

\subsection{Bayesian logistic regression}

\noindent\textbf{Pseudo-code of \texttt{LSD}$^{++}$.} For completeness, we provide in \Cref{algo:LSD-plus} the pseudo-code of the non-compressed counterpart of \texttt{QLSD}$^{++}$, namely \texttt{LSD}$^{++}$.

\begin{algorithm}[h]
   \caption{Variance-reduced Langevin Stochastic Dynamics (\texttt{LSD}$^{++}$)}
   \label{algo:LSD-plus}
  \begin{algorithmic}
     \State {\bfseries Input:} minibatch sizes $\{n_i\}_{i \in [b]}$, number of iterations $K$, step-size $\gamma \in (0,\bar{\gamma}]$ with $\bar{\gamma} > 0$, initial point $\theta_0$ and $\alpha \in (0,\bar{\alpha}]$ with $\bar{\alpha} > 0$.
     \State \Comment{Memory mechanism initialisation}
     \State Initialise $\{\eta^{(1)}_0,\ldots,\eta^{(b)}_0\}$ and $\eta_0 = \sum_{i=1}^b \eta^{(i)}_0$.
     \For{$k=0$ {\bfseries to} $K-1$}
      \State \Comment{Update of the control variates}
      \If{$k \equiv 0$ ($\mathrm{mod} \ l$)}
      \State Set $\zeta_k = \theta_k$.
      \Else{}
      \State Set $\zeta_k = \zeta_{k-1}$
      \EndIf
     \For{$i \in \mathcal{A}_{k+1}$ \Comment{On active clients}}
        \State Draw $\mathcal{S}_{k+1}^{(i)} \sim \mathrm{Uniform}\pr{\wp_{N_i,n_i}}$.
        \State {\small Set $H_{k+1}^{(i)}(\theta_k) = (N_i/n_i)\sum_{j \in \mathcal{S}_{k+1}^{(i)}} [\nabla U_{i,j}(\theta_k) - \nabla U_{i,j}(\zeta_k)] + \nabla U_i(\zeta_k)$.}
        \State Compute $\textsl{g}_{i,k+1} = H_{k+1}^{(i)}(\theta_k) - \eta^{(i)}_{k}$.
        \State Send $\textsl{g}_{i,k+1}$ to the central server.
        \State Set $\eta^{(i)}_{k+1} = \eta^{(i)}_{k} + \alpha \textsl{g}_{i,k+1}$.
     \EndFor
     \State \Comment{On the central server}
     \State Compute $\textsl{g}_{k+1} = \eta_k + \frac{b}{|\mathcal{A}_{k+1}|}\sum_{i \in \mathcal{A}_{k+1}}\textsl{g}_{i,k+1}$.
     \State Set $\eta_{k+1} = \eta_k + \alpha\sum_{i \in \mathcal{A}_{k+1}}^b\textsl{g}_{i,k+1}$.
     \State Draw $Z_{k+1} \sim \mathrm{N}(0_d,\mathrm{I}_d)$.
     \State Compute $\theta_{k+1} = \theta_k - \gamma \textsl{g}_{k+1} + \sqrt{2\gamma}Z_{k+1}$.
     \State Send $\theta_{k+1}$ to the $b$ clients.
     \EndFor
     \State {\bfseries Output:} samples $\{\theta_k\}_{k=0}^{K}$.
  \end{algorithmic}
\end{algorithm}

\noindent\textbf{Additional experimental details.} The code associated to this experiment can be found in the supplementary material (see ./code/notebook\_logistic\_regression.ipynb).
For the Bayesian logistic regression experiment detailed in the main paper, we ran the MCMC algorithms with $K=500,000$ outer iterations and considered a burn-in period of length $50,000$.

\noindent\textbf{Benefits of the memory mechanism.} We also run an additional experiment on a low-dimensional synthetic dataset to highlight the benefits brought by the memory mechanism involved in \texttt{QLSD}$^{++}$ when the dataset is highly heterogeneous.
To this end, we consider the \textsc{Synthetic}$(\alpha,\beta)$ dataset \citep{fedprox20} with $\alpha=\beta=1$, $d=2$ and $b=50$.
We run \texttt{QLSD}$^{++}$ with and without memory terms using $l=100$, $\alpha = 1/(\omega+1)$, $\gamma = 10^{-5}$ and for huge compression parameters, namely $s \in \{2^1,2^2\}$.
We use $K=100,000$ outer iterations without considering a burn-in period.
In order to have access to some ground truth, we also implement the Metropolis-adjusted Langevin algorithm (\texttt{MALA}) \citep{Robert2004}.

\begin{figure}
  \begin{center}
    \mbox{{\includegraphics[scale=0.55]{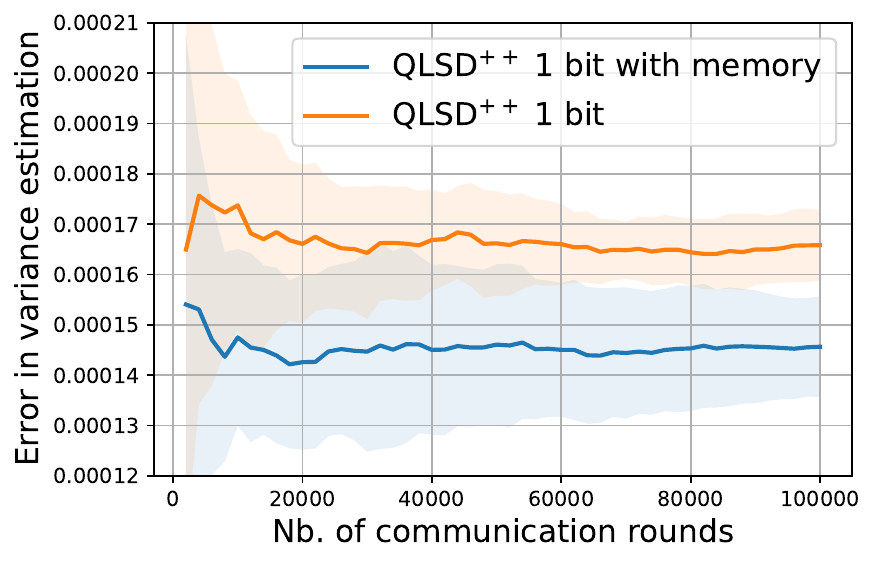}}}
    \mbox{{\includegraphics[scale=0.55]{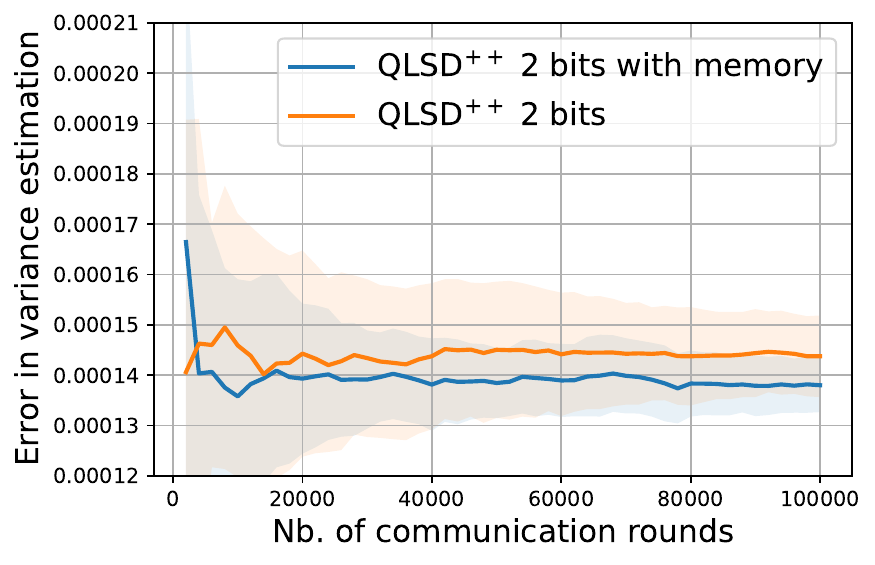}}}
  \end{center}
  \caption{Bayesian logistic regression on synthetic data.\label{fig:toy_example_log}}
\end{figure}

\Cref{fig:toy_example_log} shows the Euclidean norm of the error between the true variance under $\pi$ estimated with \texttt{MALA} and the empirical variance computed using samples generated by \texttt{QLSD}$^{++}$.
As expected, we can notice that the memory mechanism reduces the impact of the compression on the asymptotic bias of \texttt{QLSD}$^{++}$ when $\omega$ is large.

\noindent\textbf{Results on a non-image dataset.} In order to complement our results on an image dataset (FEMNIST), we also implement our methodology and one competitor (\texttt{DG-SGLD}) on the \emph{covtype}\footnote{https://archive.ics.uci.edu/ml/datasets/covertype} dataset. Again, the ground truth has been obtained by implementing a long-run Metropolis-adjusted Langevin algorithm.
The results we obtained are gathered in \Cref{table:log_reg_covtype}.

\begin{table*}
\centering
\caption{Bayesian Logistic Regression on \emph{covtype} dataset.}\label{table:log_reg_covtype}
\begin{tabular}{cc}\\\toprule
Algorithm & 99\% HPD error \\\midrule
\texttt{DG-SGLD} &1.8e-2\\
\texttt{QLSD}$^{++}$ 4 bits &2.2e-3\\
\texttt{QLSD}$^{++}$ 8 bits &2.0e-2 \\
\texttt{QLSD}$^{++}$ 16 bits &1.9e-2\\
\bottomrule
\end{tabular}
\end{table*}

\begin{itemize}
	\item \textbf{License of the assets:} We use the Synthetic dataset whose associated code is under the MIT license, and the FEMNIST dataset whose data are publicy available and associated code is under MIT license.

	\item \textbf{Total amount of compute and type of resources used:} This experiment has been run on a laptop running Windows 10 and equipped with Intel(R) Core(TM) i7\_8565U CPU 1.80GHz with 16Go of RAM.
	The total amount of compute is roughly 30 hours. 
	
	\item \textbf{Training details:} Hyperparameter values are detailed in Section 4 in the main paper.
	Regarding our experiment on real data, we use a random subset of the initial training data (for computational reasons).

\end{itemize}

\subsection{Bayesian neural networks}
\label{subsec:expresults:bnn}



\begin{itemize}
	\item \textbf{License of the assets:} We use the MNIST, FMNIST, CIFAR10 and SVHN datasets which are publicly downloadable with the torchvision.datasets package.
	\item \textbf{Total amount of compute and type of resources used:} The total computational cost depends on the dataset, but is roughly 40 hours in the worst case. 
	\item \textbf{Training details:} We consider the same hyperparameter values detailed in \Cref{table:mnist-bbn_comparison} for both training on MNIST and CIFAR10 except for the initialisation and the sampling period. For the MNIST dataset, we use the default random weights given by pytorch whereas for CIFAR-10 we use the warm-start provided by the pytorchcv library and consider a burn-in period of half the sampling period ($K=10^4$ iterations) with a thinning of 10.
\end{itemize}

In the following, we denote $\mathrm{D}_{\text{test}}$ the test dataset and for any data $(x,y) \in \mathrm{D}_{\text{test}}$, we define the preditive density by
\begin{equation}
  \label{eq:def_predictive_distribution}
  p(y \mid x)= \int p(y \mid x,\theta) \ \pi(\theta \mid \mathrm{D}) \ \rmd \theta \eqsp,
\end{equation}
where $p(y\mid x,\theta)$ is the conditional likelihood. 
For any input $x$, the predicted label is denoted by $y_{\mathrm{pred}}(x)=\argmax_{y} p(y \mid x)$.

\noindent\textbf{Metrics used for the Bayesian neural network experiment in the main paper.} In the main paper, we consider three metrics to compare the different Bayesian FL algorithms, namely \emph{Accuracy}, \emph{Agreement} and \emph{TV}. They are defined in the following. 

\begin{itemize}
  \item \textbf{Accuracy:} Based on samples from the approximate posterior distribution, we compute the minimum mean-square estimator (\emph{i.e.} corresponding to the posterior mean) and use it to make predictions on the test dataset. The \emph{Accuracy} metric corresponds to the percentage of well-predicted labels.
  \item \textbf{Agreement:} Let denote $p_{\mathrm{ref}}$ and $p$ the predictive densities associated to \texttt{HMC} and an approximate simulation-based algorithm, respectively. 
  Similar to \citet{izmailov2021bayesian}, we define the agreement between $p_{\mathrm{ref}}$ and $p$ as the fraction of the test datapoints for which the top-1 predictions of $p_{\mathrm{ref}}$ and $p$, \emph{i.e.}
  $$
  \mathrm{agreement}(p_{\mathrm{ref}},p) = \frac{1}{|\mathrm{D}_{\mathrm{test}}|} \sum_{x \in \mathrm{D}_{\mathrm{test}}} \mathbf{1}\bbr{\argmax_{y'} p_{\mathrm{ref}}(y'\mid x) = \argmax_{y'} p(y'\mid x)}\eqsp.
  $$

  \item \textbf{Total variation (TV):} By denoting $\mathcal{Y}$ the set of possible labels, we consider the total variation metric between $p_{\mathrm{ref}}$ and $p$, \emph{i.e.}
  $$
  \mathrm{TV}(p_{\mathrm{ref}},p) = \frac{1}{2|\mathrm{D}_{\mathrm{test}}|} \sum_{x \in \mathrm{D}_{\mathrm{test}}} \sum_{y'\in\mathcal{Y}} \left|p_{\mathrm{ref}}(y'\mid x) - p(y'\mid x)\right|\eqsp.
  $$
\end{itemize}

\noindent\textbf{Performance results on a highly heterogeneous dataset.} 
We train LeNet5 \citep{lecun1998gradient} architecture on the MNIST dataset \citep{deng2012mnist} and we consider the FMNIST \citep{xiao2017fashion} as the out-of-distribution dataset.
To obtain a highly heterogeneous setting, we split the data among $b=20$ clients so that each client has a dominant label representing $40\%$ of the total amount in the training set and $1\%$ of the other labels as described in \Cref{fig:num_data}.

\begin{figure}[!h]
  \begin{center}
    \mbox{{\includegraphics[scale=0.65]{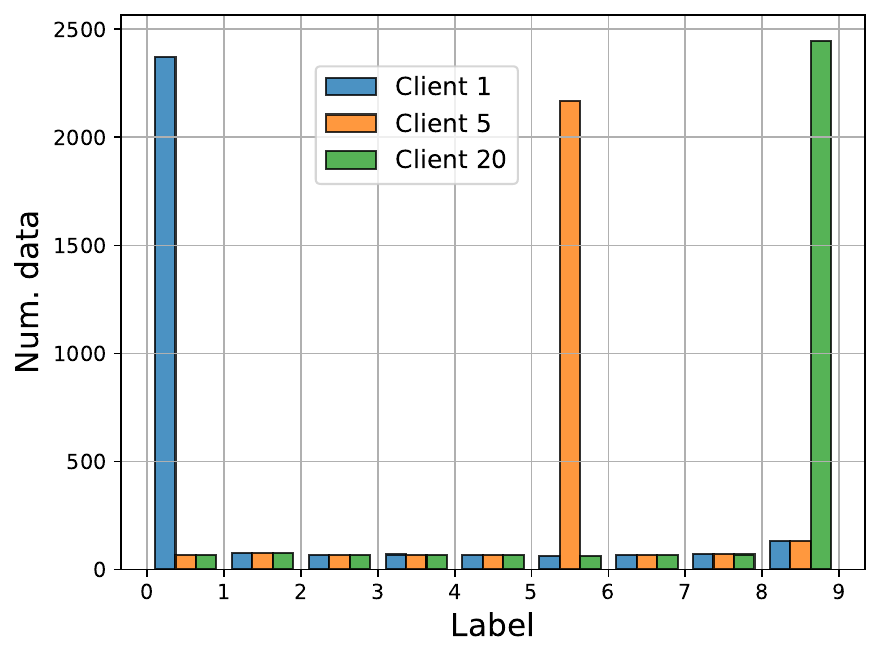}}}
  \end{center}
   \caption{Number of labels owned by different clients. \label{fig:num_data}}
\end{figure}

Inspired by the scores defined in \citet{guo2017calibration}, we measure the performance of the different algorithms and report those results in \Cref{table:mnist-bbn_comparison}. These statistics aim to better understand the predictions in order to calibrate the models \citep{rahaman2020uncertainty}.

\begin{table*}[!h]
\centering 
\begin{scriptsize} 
    \begin{tabular}{ccccccccccc}
    \\ \toprule
    Method & \texttt{SGLD} & \texttt{pSGLD} & \texttt{QLSD} & \texttt{QLSD} \texttt{PP} & \texttt{QLSD}$^{++}$ & \texttt{QLSD}$^{++}$ \texttt{PP} & \texttt{FedBe-Gauss.} & \texttt{FedBe-Dirich.} &  \texttt{FSGLD} \\ 
    \midrule
    Accuracy & 99.1 & 99.2 & 98.8 & 98.3 & 98.8 & 98.7 & 43.5 & 79.3 & 98.5 \\ 
    $10^2\times$ ECE & 0.577 & 1.25 & 0.916 & 1.57 & 0.692 & 0.930 & 7.51 & 21.3 & 2.65 \\ 
    $10^2\times$ BS & 1.38 & 1.39 & 1.98 & 2.23 & 1.91 & 2.18 & 66.6 & 36.1 & 2.64 \\ 
    $10^2\times$ nNLL & 2.86 & 3.16 & 4.15 & 4.82 & 4.11 & 4.65 & 139 & 78.0 & 6.19 \\ 
    Weight Decay & 5 & 5 & 5 & 5 & 5 & 5 & 0 & 0 & 5 \\ 
    Batch Size & 64 & 64 & 64 & 64 & 64 & 64 & 64 & 64 & 64 \\ 
    Learning rate & 1e-07 & 1e-08 & 1e-07 & 1e-07 & 1e-07 & 1e-07 & 1e-02 & 1e-02 & 1e-07 \\ 
    Local steps & N/A & N/A & 1 & 1 & 1 & 1 & 250 & 250 & 16 \\ 
    Burn-in & 100epch. & 100epch. & 1e04 & 1e04 & 1e04 & 1e04 & N/A & N/A & 1e04 \\ 
    Thinning & 1 & 1 & 500 & 500 & 500 & 500 & N/A & N/A & 500 \\
    Training & 1e03epch. & 1e03epch. & 1e05it. & 1e05it. & 1e05it. & 1e05it. & N/A & N/A & 1e05it. \\ 
    \bottomrule
    \end{tabular}
\end{scriptsize}
  \caption{Performance of Bayesian FL algorithms trained on the highly-heterogeneous dataset. \label{table:mnist-bbn_comparison}}
\end{table*}

\paragraph{Expected Calibration Error (ECE).}

To measure the difference between the accuracy and confidence of the predictions, we group the data into $M\ge 1$ buckets defined for any $m \in [M]$ by $\mathrm{B_m}=\{(x,y)\in\mathrm{D}_{\mathrm{test}}: p(y_{\mathrm{pred}}(x)| x)\in\left]\nofrac{(m-1)}{M}, \nofrac{m}{M}\right]\}$. As in the previous work of \citet{ovadia2019can}, we denote the model accuracy on $\mathrm{B_{m}}$ by
\[
\mathrm{acc}\left(\mathrm{B_{m}}\right)=\frac{1}{\left|\mathrm{B_{m}}\right|} \sum_{(x,y)\in \mathrm{B_{m}}} \mathbf{1}_{y_{\mathrm{pred}}(x)=y}
\]
and define the confidence on $\mathrm{B_m}$ by
\[
\mathrm{conf}\left(\mathrm{B_m}\right)=\frac{1}{\left|\mathrm{B_m}\right|} \sum_{(x,y) \in \mathrm{B_m}} p(y_{\mathrm{pred}}(x)| x)\eqsp.
\]
As stressed in \citet{guo2017calibration}, for any $m\in[M]$ the accurcay $\mathrm{acc}\left(\mathrm{B_m}\right)$ is an unbiased and consistent estimator of $\PP\left(y_{\mathrm{pred}}(x)=y \mid (m-1)/M<p(y_{\mathrm{pred}}(x)| x)\le \nofrac{m}{M}\right)$.
Therefore, the ECE defined by
\[
  \mathrm{ECE}=\sum_{m=1}^{M} \frac{\left|\mathrm{B_m}\right|}{\left|\mathrm{D}_{\mathrm{test}}\right|}\left|\mathrm{acc}\left(\mathrm{B_m}\right)-\mathrm{conf}\left(\mathrm{B_m}\right)\right|
\]
is an estimator of
\[
  \E_{(x,y)}\Big[\big|\PP\pr{y_{\mathrm{pred}}(x)=y \mid p(y_{\mathrm{pred}}(x)| x)}-p(y_{\mathrm{pred}}(x) | x)\big|\Big].
\]
Thus, ECE measures the absolute difference between the confidence level of a prediction and its accuracy.

\paragraph{Brier Score (BS).}

The BS is a proper scoring rule (see for example \citet{dawid2014theory}) that can only evaluate random variables taking a finite number of values. Denote by $\mathcal{Y}$ the finite set of possible labels, the BS measures the model's confidence in its predictions and is defined by
\[
  \mathrm{BS} = \frac{1}{|\mathrm{D}_{\mathrm{test}}|}\sum_{(x,y)\in\mathrm{D}_{\mathrm{test}}}\sum_{c\in\mathcal{Y}}(p(y=c| x) - \mathbf{1}_{y=c})^2 \eqsp.
\]

\paragraph{Normalised negative log-likelihood (nNLL).}

This classical score defined by
\[
  \mathrm{nNLL} = -\frac{1}{|\mathrm{D}_{\mathrm{test}}|}\sum_{(x,y)\in\mathrm{D}_{\mathrm{test}}}\log p(y| x)
\]
measures the model ability to predict good labels with high probability.

\paragraph{Out of distribution detection.}

Here we  study the behavior of our proposed algorithms in the out-of-distribution (OOD) framework, we consider the pairs MNIST/FMNIST and CIFAR10/SVHN, comparing the densities of the predictive entropies on the ID vs OOD data. These densities denoted by $p_{\text{in}}$ and $p_{\text{out}}$ respectively, are approximated using a kernel estimator based on of the histogram associated with $\{\mathrm{Ent}(x) : x \in \mathrm{D}_{\mathrm{test}}^x\}$ for $\mathrm{D}_{\mathrm{test}}\in\{\text{MNIST}, \text{FMNIST}\}$ or $\{\text{CIFAR10}, \text{SVHN}\}$, where $\mathrm{Ent}(x)$ is the predictive entropy defined by:
\[
  \mathrm{Ent}(x) = \sum_{y\in\mathcal{Y}}p(y| x)\log p(y| x) \eqsp,
\]
and $p(y|x)$ is defined by \eqref{eq:def_predictive_distribution} and estimated by the different methods that we consider. 
The resulting densities from the different methods that we consider are  displayed in \Cref{fig:suppl:entropies}.
\begin{figure}[!h]
  \begin{center}
    \mbox{{\includegraphics[scale=0.47]{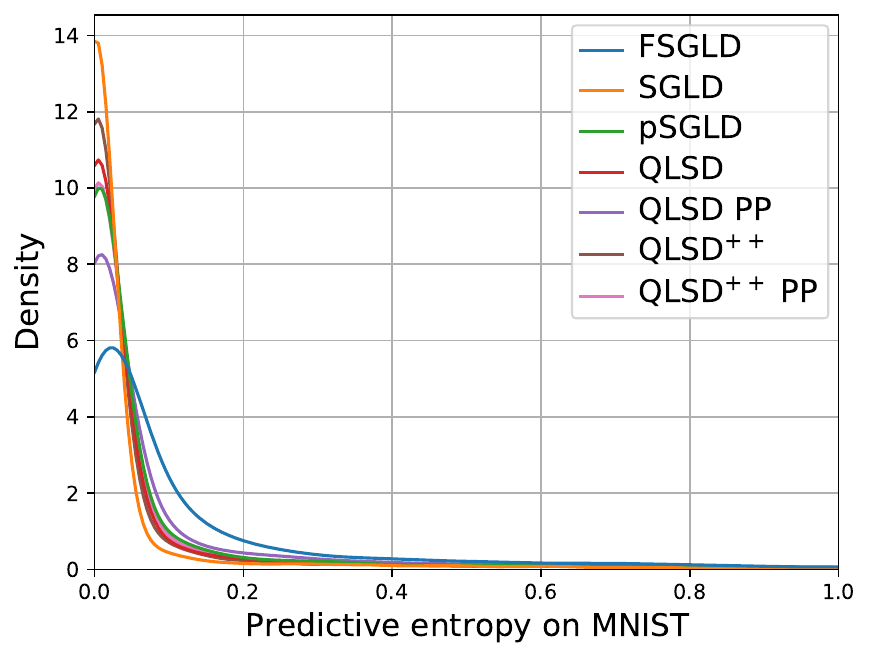}}}
    \mbox{{\includegraphics[scale=0.47]{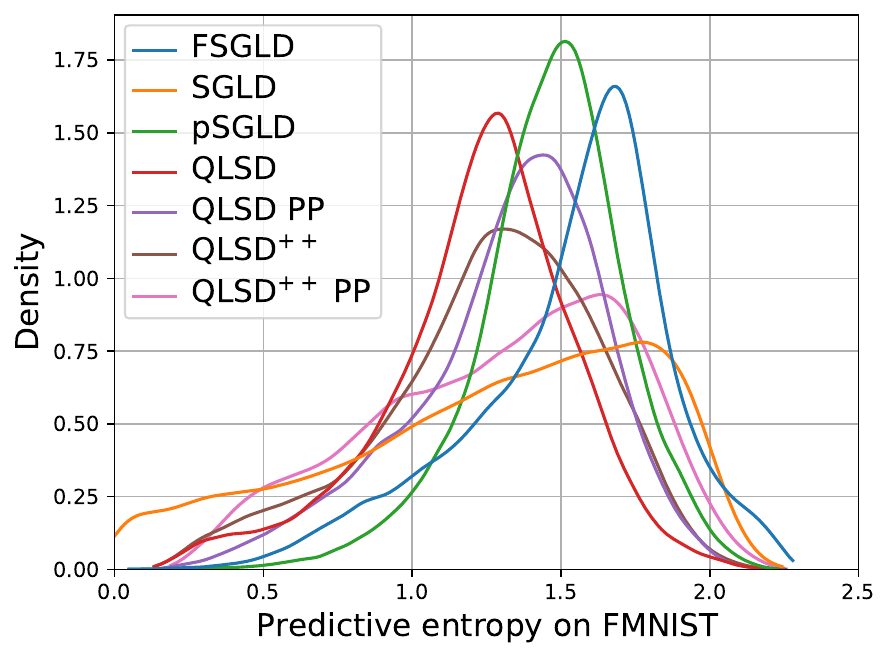}}}
  \end{center}
   \caption{Predictive entropies comparison between MNIST and FMNIST. \label{fig:suppl:entropies}}
\end{figure}

A new data point $x$ is then labeled in the original dataset (MNIST or CIFAR10) if $p_{\text{in}}(\mathrm{Ent}(x))>p_{\text{out}}(\mathrm{Ent}(x))$ and out-of-distribution otherwise.

\paragraph{Calibration results.}

Interpreting the predicted outputs as probabilities is only correct for well a calibrated model. Indeed, when a model is calibrated, the confidence is closed to the accuracy of the predictions. In order to evaluate the calibration of the models, we display the reliability diagram on the left-hand side of \Cref{fig:suppl:conf}. It represents the evolution of $\mathrm{acc}(\mathrm{B_m}) - \mathrm{conf}(\mathrm{B_m})$ in function of $\mathrm{conf}(\mathrm{B_m})$, closer the values are to zero better the model is calibrated.

For the second sub-experiment, we consider for any $\tau\in\brn{0,1}$, the set $\mathrm{D}_{\mathrm{pred}}^{(\tau)}=\{x\in\mathrm{D}_{\mathrm{test}}^x: p(y|x) \ge \tau\}$ of classified data with credibility greater than $\tau$.
We define the test accuracy on $\mathrm{D}_{\mathrm{pred}}^{(\tau)}$ by $$\mathrm{Card}(\{x\in\mathrm{D}_{\mathrm{pred}}^{(\tau)}: y_{\mathrm{true}}(x)=y_{\mathrm{pred}}(x)\})/\mathrm{Card}(\mathrm{D}_{\mathrm{pred}}^{(\tau)})\eqsp.$$
The right-hand side of \Cref{fig:suppl:conf} shows the evolution of the test accuracy on $\mathrm{D}_{\mathrm{pred}}^{(\tau)}$ with respect to the credibility threshold $\tau$. It can be noted that in both plots of \Cref{fig:suppl:conf}, the accuracy tends to $100\%$ for confident predictions.

\begin{figure}[!h]
  \begin{center}
    \mbox{{\includegraphics[scale=0.47]{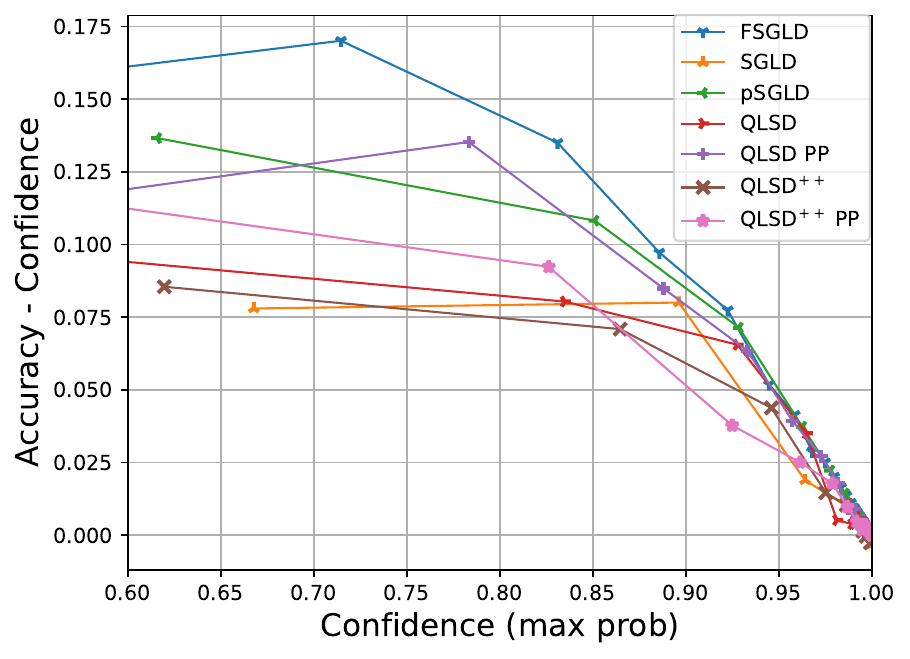}}}
    \mbox{{\includegraphics[scale=0.47]{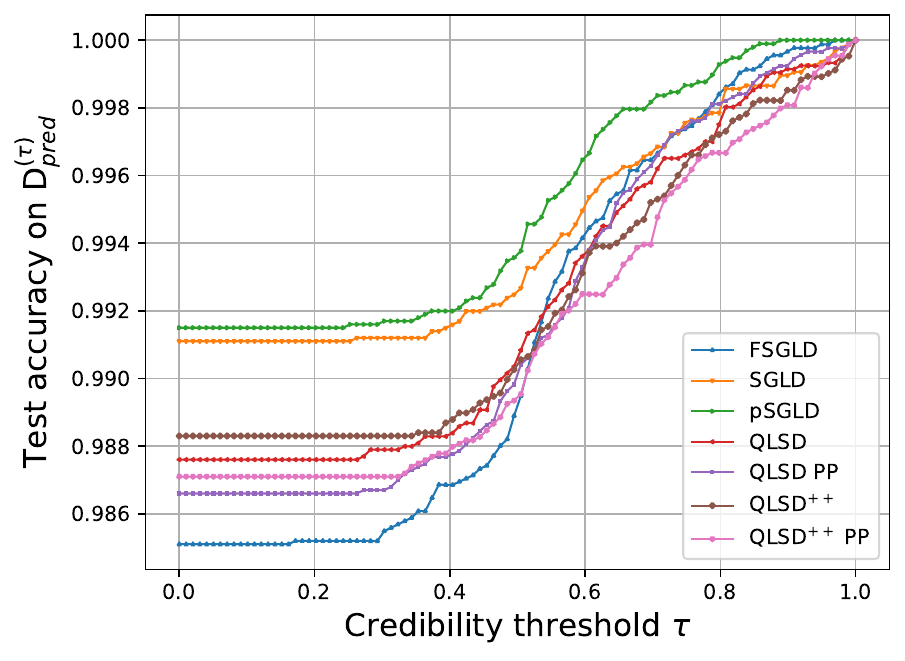}}}
  \end{center}
  \caption{Left: Calibration test from reliability diagrams -- Right: Test accuracy on $\mathrm{D}_{\mathrm{pred}}^{(\tau)}$ with respect to the threshold $\tau$. \label{fig:suppl:conf}}
\end{figure}


\end{document}